%% file: main.tex
\documentclass[runningheads, envcountsame, a4paper]{llncs}

\usepackage{microtype}
\usepackage{amsmath,amsfonts,bm}
\usepackage[pdftex]{graphicx}
\usepackage{epstopdf}

\usepackage[hyphens]{url}

\def\eqref#1{equation~\ref{#1}}

\def\1{\bm{1}}

\DeclareMathAlphabet{\mathsfit}{\encodingdefault}{\sfdefault}{m}{sl}
\SetMathAlphabet{\mathsfit}{bold}{\encodingdefault}{\sfdefault}{bx}{n}

\DeclareMathOperator*{\argmax}{arg\,max}
\DeclareMathOperator*{\argmin}{arg\,min}

\usepackage{balance}
\usepackage{color}
\usepackage{enumerate}
\usepackage{enumitem}
\usepackage{etoolbox}
\usepackage{float}
\usepackage{multirow}
\usepackage{subcaption}

\usepackage{hyperref}
\definecolor{mydarkblue}{rgb}{0,0.08,0.45}
\hypersetup{
  colorlinks,
  linkcolor={mydarkblue},
  citecolor={mydarkblue},
  urlcolor={magenta}
}

\usepackage[ruled,vlined,linesnumbered]{algorithm2e}

\DeclareMathOperator*{\minimize}{minimize\quad}

\usepackage{newtxmath}
\usepackage{mathtools}
\usepackage{enumitem}
\usepackage{bm}
\usepackage{multirow}
\usepackage{booktabs} 
\usepackage{adjustbox}
\usepackage{longtable}
\usepackage{bm,nicefrac}
\usepackage{pifont}
\usepackage{thmtools,thm-restate}
\usepackage[misc]{ifsym}

\def\*#1{\mathbf{#1}}

\newcommand{\para}[1]{\noindent \textbf{#1}}

\newcommand{\PX}{P_{\mathbf{X}}}
\newcommand{\UX}{U_{\mathbf{X}}}
\newcommand{\QX}{Q_{\mathbf{X}}} 
\newcommand{\Qfamily}{\mathcal{Q}}

\newcommand{\FPRr}{\mathrm{FPR}}
\newcommand{\FNRr}{\mathrm{FNR}}

\newcommand{\UXb}{U_{\mathrm{mix}}}

\newcommand{\Sin}{S_{\mathrm{in}}}
\newcommand{\Sau}{S_{\mathrm{au}}}

\newcommand{\uom}{u_{\mathrm{om}}}

\begin{document}

\title{ATOM: Robustifying Out-of-distribution Detection Using Outlier Mining}
\toctitle{ATOM: Robustifying Out-of-distribution Detection Using Outlier Mining}

\author{Jiefeng Chen(\Letter)\inst{1} \and
Yixuan Li \inst{1} \and
Xi Wu\inst{2} \and
Yingyu Liang\inst{1} \and
Somesh Jha\inst{1}
}

\authorrunning{J. Chen et al.}

\tocauthor{Jiefeng~Chen,Yixuan~Li,Xi~Wu,Yingyu~Liang,Somesh~Jha}

\institute{
Department of Computer Sciences \\
University of Wisconsin-Madison \\
1210 W. Dayton Street Madison, WI, US \\
\email{$\{$jiefeng,sharonli,yliang,jha$\}$@cs.wisc.edu}
\and
Google \\
\email{wuxi@google.com}}

\maketitle              

\setcounter{footnote}{0}

\begin{abstract}

Detecting out-of-distribution (OOD) inputs is critical for safely deploying deep learning models in an open-world setting. However, existing OOD detection solutions can be brittle in the open world, facing various types of adversarial OOD inputs. While methods leveraging auxiliary OOD data have emerged, our analysis on illuminative examples reveals a key insight that the majority of auxiliary OOD examples may not meaningfully improve or even hurt the decision boundary of the OOD detector, which is also observed in empirical results on real data. In this paper, we provide a theoretically motivated method, {\em Adversarial Training with informative Outlier Mining} (ATOM), which improves the robustness of OOD detection. We show that, by mining informative auxiliary OOD data, one can significantly improve OOD detection performance, and somewhat surprisingly, generalize to unseen adversarial attacks. ATOM achieves \textbf{state-of-the-art} performance under a broad family of classic and adversarial OOD evaluation tasks. For example, on the CIFAR-10 in-distribution dataset, ATOM reduces the FPR (at TPR 95\%) by up to 57.99\% under adversarial OOD inputs, surpassing the previous best baseline by a large margin. 


\end{abstract}

\section{Introduction}
\label{sec:intro}
Out-of-distribution (OOD) detection has become an indispensable part of building reliable open-world machine learning models~\cite{bendale2015towards}. An OOD detector determines whether an input is from the same distribution as the training data, or different distribution. As of recently a plethora of exciting literature has emerged to combat the problem of OOD detection~\cite{hein2019relu,GODIN,huang2021mos,lakshminarayanan2017simple,lee2018simple,liang2018enhancing,lin2021mood,liu2020energy,mohseni2020self}.

Despite the promise, previous methods primarily focused on clean OOD data, while largely underlooking the robustness aspect of OOD detection. Concerningly, recent works have shown the brittleness of OOD detection methods under adversarial perturbations~\cite{bitterwolf2020provable,hein2019relu,sehwag2019analyzing}. As illustrated in Figure~\ref{fig:adversarial-ood-example}, an OOD image (\emph{e.g.}, {mailbox}) can be perturbed to be misclassified by the OOD detector as in-distribution (traffic sign data). Failing to detect  such an \emph{adversarial OOD example\footnote{Adversarial OOD examples are constructed w.r.t the OOD detector, which is different from the standard notion of adversarial examples (constructed w.r.t the classification model).}} can be consequential in safety-critical applications such as autonomous driving~\cite{filos2020can}. Empirically on CIFAR-10, our analysis reveals that the false positive rate (FPR) of a competitive method Outlier Exposure~\cite{hendrycks2018deep} can increase from 3.66\% to 99.94\% under adversarial attack.

Motivated by this, we make an important step towards the {robust OOD detection} problem, and propose a novel training framework, {\em  \textbf{A}dversarial \textbf{T}raining with informative \textbf{O}utlier \textbf{M}ining} (ATOM). Our key idea is to \emph{selectively} utilize auxiliary outlier data for estimating a tight decision boundary between ID and OOD data, which leads to robust OOD detection performance. 
While recent methods~\cite{hein2019relu,hendrycks2018deep,meinke2019towards,mohseni2020self} have leveraged auxiliary OOD data, we show that \emph{randomly} selecting outlier samples for training yields a large portion of uninformative samples, which do not meaningfully improve the decision boundary between ID and OOD data (see Figure~\ref{fig:toy-illustration-example}). Our work demonstrates that by  mining low OOD score data for training, one can significantly improve the robustness of an OOD detector, and somewhat surprisingly, generalize to unseen adversarial attacks. 

\begin{figure*}[t!]
	\centering
		\includegraphics[width=0.88\linewidth]{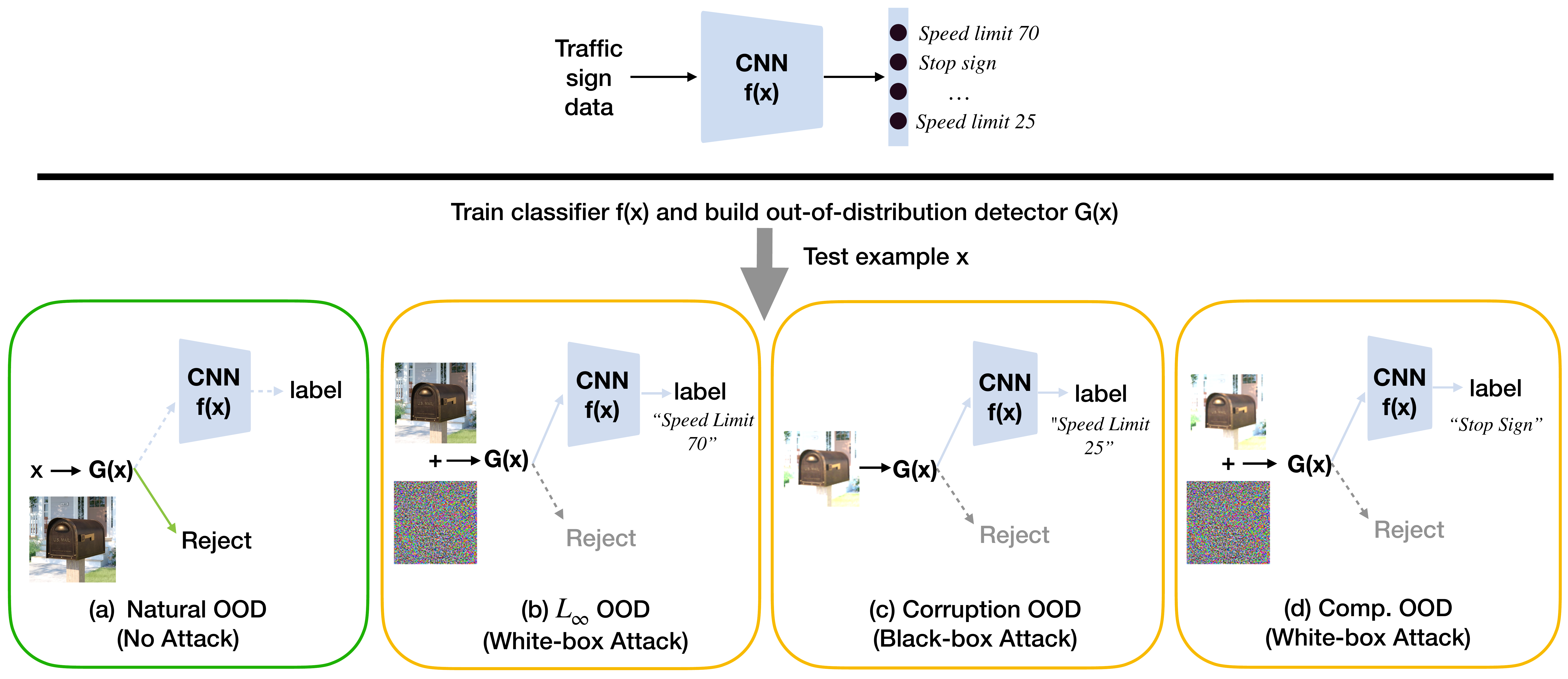}
	\caption{\small \textbf{Robust out-of-distribution detection}. When deploying an image classification system (OOD detector $G(\*x)$ + image classifier $f(\*x)$) in an open world, there can be multiple types of OOD examples. We consider a broad family of  OOD inputs, including (a) Natural OOD, (b) $L_\infty$ OOD, (c) corruption OOD, and (d) Compositional OOD. A detailed description of these OOD inputs can be found in Section~\ref{sec:setup}. In (b-d), a perturbed OOD input (e.g., a perturbed {mailbox} image) can mislead the OOD detector to classify it as an in-distribution sample. This can trigger the downstream image classifier $f(\*x)$ to predict it as one of the in-distribution classes (e.g.,  {speed limit 70}). Through {\em adversarial training with informative outlier mining} (ATOM), our method can robustify the decision boundary of OOD detector $G(\*x)$, which leads to improved performance across all types of OOD inputs. Solid lines are actual computation flow.
	}
	\label{fig:adversarial-ood-example}
\end{figure*}

We extensively evaluate ATOM on common OOD detection benchmarks, as well as a suite of {adversarial OOD tasks}, as illustrated in Figure~\ref{fig:adversarial-ood-example}. ATOM achieves state-of-the-art performance, significantly outperforming competitive methods using standard training on random outliers~\cite{hendrycks2018deep,meinke2019towards,mohseni2020self}, or using adversarial training on random outlier data~\cite{hein2019relu}. On the classic OOD evaluation task (clean OOD data), ATOM achieves comparable and often better performance than current state-of-the-art methods. On $L_\infty$ OOD evaluation task, ATOM outperforms the best baseline ACET~\cite{hein2019relu} by a large margin (e.g.  {\bf 53.9\%} false positive rate deduction on CIFAR-10). Moreover, our ablation study underlines the importance of having both adversarial training and outlier mining (ATOM) for achieving robust OOD detection. 

Lastly, we provide theoretical analysis for ATOM, characterizing how outlier mining can better shape the decision boundary of the OOD detector.
While hard negative mining has been explored
in different domains of learning, e.g., object detection, deep metric learning~\cite{felzenszwalb2009object,gidaris2015object,shrivastava2016training}, the vast literature
of OOD detection has not explored this idea. Moreover, most uses
of hard negative mining are on a heuristic basis, but in this paper, we derive precise formal guarantees with insights. Our \textbf{key contributions} are summarized as follows:
\begin{itemize}

    \item We propose a novel training framework, adversarial training with outlier mining (ATOM), which facilitates efficient use of auxiliary outlier data to regularize the model for robust OOD detection. 
    \item We perform extensive analysis and comparison with a diverse collection of OOD detection methods using: (1) pre-trained models, (2) models trained on randomly sampled outliers, (3) adversarial training. ATOM establishes \textbf{state-of-the-art} performance under a broad family of clean and adversarial OOD evaluation tasks. 
      \item  We contribute theoretical analysis formalizing the intuition of mining informative outliers for improving the robustness of OOD detection.
    \item Lastly, we provide a unified evaluation framework that allows future research examining the robustness of OOD detection algorithms under a broad family of OOD inputs. Our code and data are released to facilitate future research on robust OOD detection: \url{https://github.com/jfc43/informative-outlier-mining}.
\end{itemize}

\section{Preliminaries}
\label{sec:problem-statement}
We consider the setting of multi-class classification. We consider a training dataset $\mathcal{D}_{\text{in}}^{\text{train}}$ drawn i.i.d.\ from a data distribution $P_{\bm{X},Y}$, where $\bm{X}$ is the sample space and ${Y} = \{1,2,\cdots,K \}$ is the set of labels. In addition, we have an auxiliary outlier data $\mathcal{D}_{\text{out}}^{\text{auxiliary}}$ from distribution $U_\*X$. The use of auxiliary outliers helps regularize the model for OOD detection, as shown in several recent works \cite{hein2019relu,lee2017training,liu2020energy,meinke2019towards,mohseni2020self}.

\para{Robust out-of-distribution detection.} 
The goal is to learn a detector $G:\*x \to \{-1, 1\}$, which outputs $1$ for an in-distribution example $\*x$ and output $-1$ for a clean or perturbed OOD example $\*x$. Formally, let $\Omega(\*x)$ be a set of small perturbations on an OOD example $\*x$. The detector is evaluated on $\*x$ from $\PX$ and on the worst-case input inside $\Omega(\*x)$ for an OOD example $\*x$ from $\QX$. The {false negative rate} (FNR) and {false positive rate} (FPR) are defined as:
\begin{align}
\textrm{FNR}(G) = \mathbb{E}_{\*x\sim \PX} \mathbb{I}[G(\*x) = -1], \quad
\textrm{FPR}(G; \QX, \Omega)  = \mathbb{E}_{\*x\sim \QX} \max_{\delta \in \Omega(\*x)}  \mathbb{I}[G(\*x+\delta)=1]. \nonumber
\end{align}

\para{Remark.} Note that test-time OOD distribution $Q_{\*{X}}$ is unknown, which can be different from $U_\*{X}$. The difference between the auxiliary data $\UX$ and test OOD data $\QX$ raises the fundamental question of how to effectively leverage $\mathcal{D}_{\text{out}}^{\text{auxiliary}}$ for improving learning the decision boundary between in- vs. OOD data. For terminology clarity, we refer to training OOD examples as \emph{outliers}, and exclusively use \emph{OOD data} to refer to test-time anomalous inputs.

\section{Method}
\label{sec:method}
In this section, we introduce {\em Adversarial Training with informative Outlier Mining} (ATOM). We first present our method overview, and then describe details of the training objective with informative outlier mining.

\para{Method overview: a conceptual example.} We use the terminology \emph{outlier mining} to denote the process of selecting informative outlier training samples from the pool of auxiliary outlier data. We illustrate our idea with a toy example in Figure~\ref{fig:toy-illustration-example}, where in-distribution data consists of class-conditional Gaussians. Outlier training data is sampled from a uniform distribution from outside the support of in-distribution. Without outlier mining (\emph{left}), we will almost sample those ``easy" outliers and the decision boundary of the OOD detector learned can be loose. In contrast, with outlier mining (\emph{right}), selective outliers close to the decision boundary between ID and OOD data, which improves OOD detection.  This is particularly important for robust OOD detection where the boundary needs to have a margin from the OOD data so that even adversarial perturbation (red color) cannot move the OOD data points across the boundary. We proceed with describing the training mechanism that achieves our novel conceptual idea and will provide formal theoretical guarantees in Section~\ref{sec:analysis}.

\begin{figure}[t]
	\centering
		\includegraphics[width=0.95\linewidth]{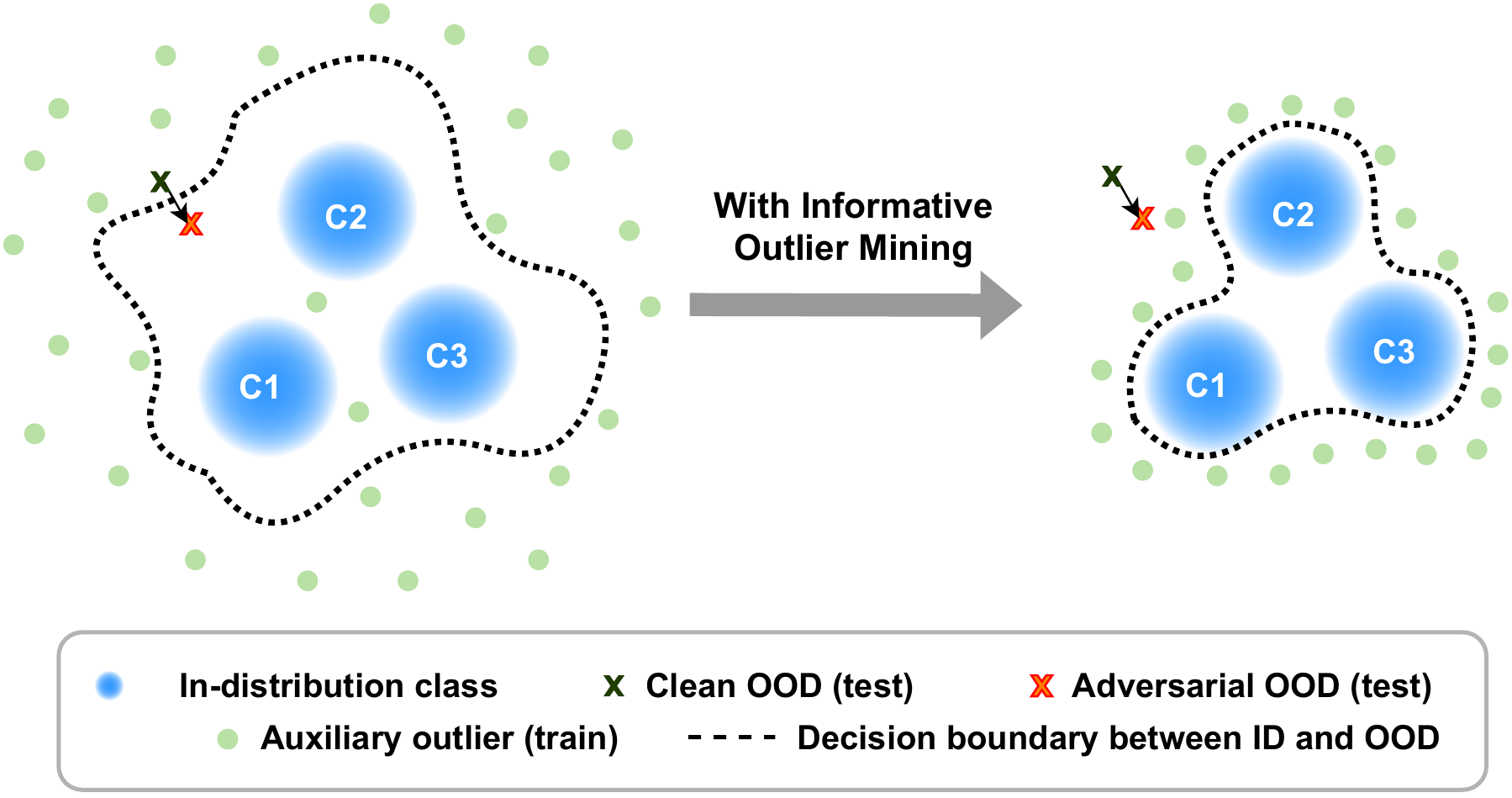}
	\caption{\small A toy example in 2D space for illustration of informative outlier mining. With informative outlier mining, we can tighten the decision boundary and build a robust OOD detector.  
	}
	\label{fig:toy-illustration-example}

\end{figure}

\subsection{ATOM: Adversarial Training with Informative Outlier Mining}

\para{Training objective.} The classification  involves using a mixture of ID data and outlier samples. Specifically,  we consider a $(K+1)$-way classifier network $f$, where the $(K+1)$-th class label indicates out-of-distribution class. Denote by $F_\theta(\*x)$  the softmax output of $f$ on $\*x$.
The robust training objective is given by
\begin{align}
\label{obj:adv}
\minimize_\theta \mathbb{E}_{(\*x,y)\sim \mathcal{D}_{\text{in}}^{\text{train}}} [\ell(\*x, y; F_\theta)] 
 + \lambda \cdot \mathbb{E}_{\*x \sim \mathcal{D}_{\text{out}}^{\text{train}}} \max_{\*x' \in \Omega_{\infty, \epsilon}(\*x)} [\ell(\*x', K+1; F_\theta)]
\end{align}
where $\ell$ is the cross entropy loss, and $\mathcal{D}_{\text{out}}^{\text{train}}$ is the OOD training dataset. We use Projected Gradient Descent (PGD) \cite{madry2017towards} to solve the inner max of the objective, and apply it to half of a minibatch while keeping the other half clean to ensure performance on both clean and perturbed data.

Once trained, the OOD detector $G(\*x)$ can be constructed by:
\begin{align}
\label{obj:detector}
    G(\*x) = 
    \begin{cases} 
        -1 & \quad \text{if } F(\*x)_{K+1} \ge \gamma, \\
        1 & \quad \text{if } F(\*x)_{K+1} < \gamma,
    \end{cases}
\end{align}
where $\gamma$ is the threshold, and in practice can be chosen on the in-distribution data so that a high fraction of the test examples are correctly classified by $G$. We call $F(\*x)_{K+1}$  the {\em OOD score} of $\*x$. 
For an input labeled as in-distribution by $G$, one can obtain its semantic label using $\hat{F}(\*x)$: 
\begin{align}
\label{obj:classifier}
    \hat{F}(\*x) = \argmax_{y \in \{1,2,\cdots, K\}} F(\*x)_y
\end{align}

\para{Informative outlier mining.} We propose to adaptively choose OOD training examples where the detector is uncertain about. Specifically,
during each training epoch, we randomly sample $N$ data points from the auxiliary OOD dataset $\mathcal{D}_{\text{out}}^{\text{auxiliary}}$, and  use the current model to infer the OOD scores\footnote{Since the inference stage can be fully parallel, outlier mining can be applied with relatively low overhead.}. Next, we sort the data points according to the OOD scores and select a subset of $n<N$ data points, starting with the $qN^\text{th}$ data in the sorted list. We then use the selected samples as  OOD training data $\mathcal{D}_{\text{out}}^{\text{train}}$ for the next epoch of training.  Intuitively, $q$ determines the  {\em informativeness} of the sampled points w.r.t the OOD detector. The larger $q$ is, the less informative those sampled examples become. Note that informative outlier mining is performed on (non-adversarial) auxiliary OOD data. Selected examples are then used in the robust training objective (\ref{obj:adv}).

\begin{algorithm}[h] 
    \SetAlgoLined
		\KwIn{$\mathcal{D}_{\text{in}}^{\text{train}}$,
		$\mathcal{D}_{\text{out}}^{\text{auxiliary}}$, $F_\theta$, $m$, $N$, $n$, $q$}
		\KwOut{$\hat{F}$, $G$} 
		\For{$t=1, 2, \cdots, m$}{
		     Randomly sample $N$ data points from $\mathcal{D}_{\text{out}}^{\text{auxiliary}}$ to get a candidate set $\mathcal{S}$\;
             Compute OOD scores on $\mathcal{S}$ using current model  $F_\theta$ to get set $V=\{ F(\*x)_{K+1} \mid \*x \in \mathcal{S} \}$. Sort scores in $V$ from the lowest to the highest\;
             $\mathcal{D}_{\text{out}}^{\text{train}} \leftarrow V[qN:qN + n]$  \tcc*{$q\in [0,1-n/N]$}  
             Train $F_\theta$ for one epoch using the training objective of (\ref{obj:adv})\;
		}
        Build $G$ and $\hat{F}$ using (\ref{obj:detector}) and (\ref{obj:classifier}) respectively\; 
        \caption{ATOM: Adv. Training with informative Outlier Mining}
        \label{alg:training-with-informative-outlier-mining}
\end{algorithm}


We provide the complete training algorithm using informative outlier mining in Algorithm~\ref{alg:training-with-informative-outlier-mining}. Importantly, the use of informative outlier mining highlights the key difference between ATOM and previous work using \emph{\textbf{randomly}} sampled outliers~\cite{hein2019relu,hendrycks2018deep,meinke2019towards,mohseni2020self}. 

\section{Experiments}
\label{sec:experiment}
In this section, we describe our experimental setup and  show  that ATOM can substantially improve OOD detection performance on both clean OOD data and adversarially perturbed OOD inputs. We also conducted extensive ablation analysis to explore different aspects of our algorithm.

\subsection{Setup}
\label{sec:setup}

\para{In-distribution datasets.} We use CIFAR-10, and CIFAR-100~\cite{krizhevsky2009learning} datasets as in-distribution datasets. We also show results on SVHN in Appendix~\ref{sec:svhn-results}. 

\para{Auxiliary OOD datasets.} By default, we use 80 Million Tiny Images (TinyImages) \cite{torralba200880} as $\mathcal{D}_{\text{out}}^{\text{auxiliary}}$, which is a common setting in prior works. We also use ImageNet-RC, a variant of ImageNet~\cite{chrabaszcz2017downsampled} as an alternative auxiliary OOD dataset.

\para{Out-of-distribution datasets.} For OOD test dataset, we follow common setup in literature and use six diverse datasets: \texttt{SVHN}, \texttt{Textures}~\cite{cimpoi14describing}, \texttt{Places365}~\cite{zhou2017places}, \texttt{LSUN (crop)}, \texttt{LSUN (resize)}~\cite{yu2015lsun}, and \texttt{iSUN}~\cite{xu2015turkergaze}. 

\para{Hyperparameters.}  The hyperparameter $q$ is chosen on a separate validation set from TinyImages, which is different from test-time OOD data (see Appendix~\ref{appendix:choose-best-q}). Based on the validation, we set $q=0.125$ for CIFAR-10 and $q=0.5$ for CIFAR-100. For all experiments, we set $\lambda=1$. For CIFAR-10 and CIFAR-100, we set $N=400,000$, and $n=100,000$. More  details about experimental set up are in  Appendix~\ref{sec:more-experiment-settings}.

\para{Robust OOD evaluation tasks. } We consider the following family of OOD inputs, for which we provide details and visualizations in Appendix~\ref{sec:example-of-four-types-of-OOD}:

\begin{itemize}
    \item {\bf Natural OOD}: This is equivalent to the classic OOD evaluation with clean OOD input $\*x$, and $\Omega=\O$. 
    \item {\bf $L_\infty$ attacked OOD (white-box)}: We consider small $L_\infty$-norm bounded perturbations on an OOD input  $\*x$~\cite{athalye2018obfuscated,madry2017towards}, which induce the model to produce a high confidence score (or a low OOD score) for $\*x$. We denote the adversarial perturbations by $\Omega_{\infty, \epsilon}(\*x)$, where $\epsilon$ is the adversarial budget. We provide attack algorithms for all eight OOD detection methods in Appendix~\ref{sec:attack-algorithms}.
    \item {\bf Corruption attacked OOD (black-box)}: We consider a more realistic type of attack based on common corruptions~\cite{hendrycks2019benchmarking}, which could appear naturally in the  physical world. For each OOD image, we generate 75 corrupted images (15 corruption types $\times$ 5 severity levels), and then select the one with the lowest OOD score. 
    \item {\bf Compositionally attacked OOD (white-box)}: Lastly, we consider applying $L_\infty$-norm bounded attack and corruption attack jointly to an OOD input $\*x$, 
    as considered in \cite{laidlaw2019functional}.
\end{itemize}

\para{Evaluation metrics.} We measure the following metrics: the false positive rate (FPR) at 5\% false negative rate (FNR), and the area under the receiver operating characteristic curve (AUROC).

\subsection{Results}
\label{sec:results}

\begin{table*}[t!]
    \caption[]{\footnotesize Comparison with competitive OOD detection methods. We use DenseNet as network architecture for all methods. We evaluate on four types of OOD inputs: (1) natural OOD, (2) corruption attacked OOD, (3) $L_\infty$ attacked OOD, and (4) compositionally attacked OOD inputs. The description of these OOD inputs can be found in Section~\ref{sec:setup}. $\uparrow$ indicates larger value is better, and $\downarrow$ indicates lower value is better. All values are percentages and are averaged over six different OOD test datasets described in Section \ref{sec:setup}. 
	\textbf{Bold} numbers are superior results. Results on additional in-distribution dataset SVHN are provided in Appendix~\ref{sec:svhn-results}. Results on a different architecture, WideResNet, are provided in Appendix~\ref{sec:ablation-network-arch}.}
	\begin{adjustbox}{width=\columnwidth,center}
		\begin{tabular}{l|l|cc|cc|cc|cc}
			\toprule
			\multirow{4}{0.12\linewidth}{$\mathcal{D}_{\text{in}}^{\text{test}}$} &  \multirow{4}{0.06\linewidth}{\textbf{Method}}  &\bf{FPR}  & {\bf AUROC}  & {\bf FPR} & {\bf AUROC}  & {\bf FPR} & {\bf AUROC} & {\bf FPR} & {\bf AUROC}  \\
			& & $\textbf{(5\% FNR)}$ & $\textbf{}$ & {\bf (5\% FNR)} & $\textbf{}$ & $\textbf{(5\% FNR)}$ & {\bf } & $\textbf{(5\% FNR)}$ & {\bf } \\
			& & $\downarrow$ & $\uparrow$ & $\downarrow$ & $\uparrow$ & $\downarrow$ & $\uparrow$ & $\downarrow$ & $\uparrow$ \\ \cline{3-10}
			& & \multicolumn{2}{c|}{\textbf{Natural OOD}} & \multicolumn{2}{c|}{\textbf{Corruption OOD}} & \multicolumn{2}{c|}{\textbf{$L_\infty$ OOD }} & \multicolumn{2}{c}{\textbf{Comp. OOD}} \\  \hline 
			\multirow{10}{0.12\linewidth}{{{\bf CIFAR-10}}}  
			& MSP~\cite{hendrycks2016baseline}  & 50.54 & 91.79 & 100.00 & 58.35 & 100.00 & 13.82 & 100.00 & 13.67 \\
			& ODIN~\cite{liang2018enhancing}  & 21.65 & 94.66 & 99.37 & 51.44 & 99.99 & 0.18 & 100.00 & 0.01 \\
			& Mahalanobis~\cite{lee2018simple} & 26.95 & 90.30 & 91.92 & 43.94 & 95.07 & 12.47 & 99.88 & 1.58 \\
			& SOFL~\cite{mohseni2020self}  & 2.78 & 99.04 & 62.07 & 88.65 & 99.98 & 1.01  & 100.00 & 0.76 \\
			& OE~\cite{hendrycks2018deep}  & 3.66 & 98.82 & 56.25 & 90.66 & 99.94 & 0.34 & 99.99 & 0.16 \\ 
			& ACET~\cite{hein2019relu} & 12.28 & 97.67 & 66.93 & 88.43 & 74.45 & 78.05 & 96.88 & 53.71 \\ 
			& CCU~\cite{meinke2019towards} & 3.39 & 98.92 & 56.76 & 89.38 & 99.91 & 0.35 & 99.97 & 0.21 \\ 
			& ROWL~\cite{sehwag2019analyzing}  & 25.03 & 86.96 & 94.34 & 52.31 & 99.98 & 49.49 & 100.00 & 49.48 \\ 
			& \textbf{ATOM} (ours) & {\bf 1.69} & {\bf 99.20} & {\bf 25.26} & {\bf 95.29} & {\bf 20.55} & {\bf 88.94} & {\bf 38.89} & {\bf 86.71} \\  \hline 
			\multirow{10}{0.12\linewidth}{{\bf CIFAR-100}} 
			& MSP~\cite{hendrycks2016baseline} & 78.05 & 76.11 & 100.00 & 30.04 & 100.00 & 2.25 & 100.00 & 2.06 \\
			& ODIN~\cite{liang2018enhancing}  & 56.77 & 83.62 & 100.00 & 36.95 & 100.00 & 0.14 & 100.00 & 0.00 \\
			& Mahalanobis~\cite{lee2018simple}  & 42.63 & 87.86 & 95.92 & 42.96 & 95.44 & 15.87 & 99.86 & 2.08 \\
			& SOFL~\cite{mohseni2020self} & 43.36 & 91.21 & 99.93 & 45.23 & 100.00 & 0.35  & 100.00 & 0.27 \\
			& OE~\cite{hendrycks2018deep}  & 49.21 & 88.05 & 99.96 & 45.01 & 100.00 & 0.94 & 100.00 & 0.59 \\ 
			& ACET~\cite{hein2019relu}  & 50.93 & 89.29 & 99.53 & 54.19 & 76.27 & 59.45 & 99.71 & 38.63 \\ 
			& CCU~\cite{meinke2019towards}  & 43.04 & 90.95 & 99.90 & 48.34 & 100.00 & 0.75 & 100.00 & 0.48 \\ 
			& ROWL~\cite{sehwag2019analyzing}  & 93.35 & 53.02 & 100.00 & 49.69 & 100.00 & 49.69 & 100.00 & 49.69 \\ 
			& \textbf{ATOM} (ours)  & {\bf 32.30} & {\bf 93.06} & {\bf 93.15} & {\bf 71.96} & {\bf 38.72} & {\bf 88.03} & {\bf 93.44} & {\bf 69.15} \\ 
			\bottomrule
		\end{tabular}
	\end{adjustbox}
	\label{tab:main-results}
\end{table*}

\para{ATOM vs. existing methods.} We show in Table~\ref{tab:main-results} that ATOM outperforms competitive OOD detection methods on both classic and adversarial OOD evaluation tasks. There are several salient observations. \textbf{First}, on classic OOD evaluation task (clean OOD data), ATOM achieves comparable or often even better performance than the current state-of-the-art methods. \textbf{Second}, on the existing adversarial OOD evaluation task, $L_\infty$ OOD, ATOM outperforms current state-of-the-art method ACET~\cite{hein2019relu} by a large margin (e.g. on CIFAR-10, our method outperforms ACET by {\bf 53.9\%} measured by FPR). \textbf{Third}, while ACET is somewhat brittle under the new Corruption OOD evaluation task, our method can generalize surprisingly well to the unknown corruption attacked OOD inputs, outperforming the best baseline by a large margin (e.g. on CIFAR-10, by up to {\bf 30.99\%} measured by FPR). \textbf{Finally}, while almost every method fails under the hardest compositional OOD evaluation task, our method still achieves impressive results (e.g. on CIFAR-10, reduces the FPR  by {\bf 57.99\%}). The performance is noteworthy since our method is not trained explicitly on corrupted OOD inputs. Our training method leads to improved OOD detection while preserving  classification performance on in-distribution data (see Appendix~\ref{app:in-distribution-acc}). Consistent performance improvement is observed on \emph{alternative in-distribution datasets} (SVHN and CIFAR-100), \emph{alternative network architecture} (WideResNet, Appendix~\ref{sec:ablation-network-arch}), and with \emph{alternative auxiliary dataset} (ImageNet-RC, see Appendix~\ref{sec:ablation-auxiliary-datasets}).

\noindent \para{Adversarial training alone is not able to achieve strong OOD robustness.} We perform an ablation study that isolates the effect of outlier mining. 
In particular, we use the same training objective as in Equation~(\ref{obj:adv}), but with randomly sampled outliers. The results in Table~\ref{tab:atom-variant-results} show AT (no outlier mining) is in general less robust. For example,  under $L_\infty$ OOD, AT displays 23.76\% and 31.61\% reduction in FPR on CIFAR-10 and CIFAR-100 respectively. This validates the importance of outlier mining for robust OOD detection, which provably improves the decision boundary as we will show in Section~\ref{sec:analysis}.

\para{Effect of adversarial training.}  We perform an ablation study that isolates the effect of adversarial training. In particular, we consider the following objective without adversarial training: 
\begin{align}
\label{obj:nat}
\minimize_\theta \mathbb{E}_{(\*x,y)\sim \mathcal{D}_{\text{in}}^{\text{train}}} [\ell(\*x, y; \hat{F}_\theta)] + \lambda \cdot \mathbb{E}_{\*x \sim \mathcal{D}_{\text{out}}^{\text{train}}} [\ell(\*x, K+1; \hat{F}_\theta)],
\end{align}
which we name \emph{Natural Training with informative Outlier Mining} (NTOM).
In Table~\ref{tab:atom-variant-results}, we show that NTOM achieves comparable performance as ATOM on natural OOD and corruption OOD. However, NTOM is less robust under $L_\infty$ OOD (with 79.35\% reduction in FPR on CIFAR-10) and compositional OOD inputs. This underlies the importance of having both adversarial training and outlier mining (ATOM) for overall good performance, particularly for robust OOD evaluation tasks.

\begin{table*}[t!]
    \caption[]{\small \textbf{Ablation} on ATOM training objective. We use DenseNet as network architecture. $\uparrow$ indicates larger value is better, and $\downarrow$ indicates lower value is better. All values are percentages and are averaged over six different OOD test datasets described in Section \ref{sec:setup}. }
	\begin{adjustbox}{width=\columnwidth,center}
		\begin{tabular}{l|l|cc|cc|cc|cc}
			\toprule
			\multirow{4}{0.12\linewidth}{$\mathcal{D}_{\text{in}}^{\text{test}}$} &  \multirow{4}{0.06\linewidth}{\textbf{Method}}  &{FPR}  & {AUROC}  & { FPR} & {AUROC}  & { FPR} & { AUROC} & { FPR} & {AUROC}  \\
			& & ${(5\% FNR)}$ & $\textbf{}$ & { (5\% FNR)} & $\textbf{}$ & ${(5\% FNR)}$ & { } & ${(5\% FNR)}$ & { } \\
			& & $\downarrow$ & $\uparrow$ & $\downarrow$ & $\uparrow$ & $\downarrow$ & $\uparrow$ & $\downarrow$ & $\uparrow$ \\ \cline{3-10}
			& & \multicolumn{2}{c|}{\textbf{Natural OOD}} & \multicolumn{2}{c|}{\textbf{Corruption OOD}} & \multicolumn{2}{c|}{\textbf{$L_\infty$ OOD }} & \multicolumn{2}{c}{\textbf{Comp. OOD}} \\  \hline 
			\multirow{4}{0.12\linewidth}{{{\bf CIFAR-10}}}  
			& AT (no outlier mining) & 2.65 & 99.11 & 42.28 & 91.94 & 44.31 & 68.64 & 65.17 & 72.62 \\
 			& NTOM (no adversarial training) & 1.87 & {\bf 99.28} & 30.58 & 94.67 & 99.90 & 1.22 & 99.99 & 0.45 \\ 
			& ATOM (ours) & {\bf 1.69} & 99.20 & {\bf 25.26} & {\bf 95.29} & {\bf 20.55} & {\bf 88.94} & {\bf 38.89} & {\bf 86.71} \\  \hline 
			\multirow{4}{0.12\linewidth}{{\bf CIFAR-100}} 
			& AT (no outlier mining)& 51.50 & 89.62 & 99.70 & 58.61 & 70.33 & 58.84 & 99.80 & 34.98 \\ 
 			& NTOM (no adversarial training) & 36.94 & 92.61 & 98.17 & 65.70 & 99.97 & 0.76 & 100.00 & 0.16 \\ 
			& ATOM (ours) & {\bf 32.30} & {\bf 93.06} & {\bf 93.15} & {\bf 71.96} & {\bf 38.72} & {\bf 88.03} & {\bf 93.44} & {\bf 69.15} \\ 
			\bottomrule
		\end{tabular}
	\end{adjustbox}
	\label{tab:atom-variant-results}
\end{table*}

\para{Effect of sampling parameter $q$.} Table~\ref{tab:ablation-study} 
shows the performance with different sampling parameter $q$. For all three datasets, training on auxiliary outliers with large OOD scores (\emph{i.e.}, too easy examples with $q=0.75$) worsens the performance, which suggests the necessity  to include examples on which the OOD detector is uncertain. Interestingly, in the setting where the in-distribution data and auxiliary OOD data are disjoint (\emph{e.g.}, SVHN/TinyImages), $q=0$ is optimal, which suggests that the hardest outliers are mostly useful for training. However, in a more realistic setting, the auxiliary OOD data can almost always contain data similar to in-distribution data (\emph{e.g.}, CIFAR/TinyImages). Even without removing near-duplicates exhaustively, ATOM can adaptively avoid training on those near-duplicates of in-distribution data (e.g. using $q=0.125$ for CIFAR-10 and $q=0.5$ for CIFAR-100). 

\para{Ablation on a different auxiliary dataset.}  To see the effect of the auxiliary dataset, we additionally experiment with ImageNet-RC as an alternative. We observe a consistent improvement of ATOM, and in many cases with performance better than using TinyImages. For example, on CIFAR-100, the FPR under natural OOD inputs is reduced from 32.30\% (w/ TinyImages) to 15.49\% (w/ ImageNet-RC). Interestingly, in all three datasets, using $q=0$ (hardest outliers) yields the optimal performance since there are substantially fewer near-duplicates between ImageNet-RC and in-distribution data. This ablation suggests that ATOM's success does not depend on a particular auxiliary dataset. Full results are provided in Appendix~\ref{sec:ablation-auxiliary-datasets}.

\begin{table*}[t!]
    \caption[]{\small Ablation study on $q$. We use DenseNet as network architecture. $\uparrow$ indicates larger value is better, and $\downarrow$ indicates lower value is better. All values are percentages and are averaged over six natural OOD test datasets mentioned in Section \ref{sec:setup}. Note: the hyperparameter $q$ is chosen on a separate validation set, which is different from test-time OOD data. See Appendix~\ref{appendix:choose-best-q} for details.}
	\begin{adjustbox}{width=\columnwidth,center}
		\begin{tabular}{l|l|cc|cc|cc|cc}
			\toprule
			 \multirow{4}{0.08\linewidth}{$\mathcal{D}_{\text{in}}^{\text{test}}$} & \multirow{4}{0.06\linewidth}{\textbf{Model}}  &\bf{FPR}  & {\bf AUROC}  & {\bf FPR} & {\bf AUROC}  & {\bf FPR} & {\bf AUROC} & {\bf FPR} & {\bf AUROC}  \\
			 & & $\textbf{(5\% FNR)}$ & $\textbf{}$ & {\bf (5\% FNR)} & $\textbf{}$ & $\textbf{(5\% FNR)}$ & {\bf } & $\textbf{(5\% FNR)}$ & {\bf } \\
			 & & $\downarrow$ & $\uparrow$ & $\downarrow$ & $\uparrow$ & $\downarrow$ & $\uparrow$ & $\downarrow$ & $\uparrow$ \\ \cline{3-10}
			 & & \multicolumn{2}{c|}{\textbf{Natural OOD}} & \multicolumn{2}{c|}{\textbf{Corruption OOD}} & \multicolumn{2}{c|}{\textbf{$L_\infty$ OOD }} & \multicolumn{2}{c}{\textbf{Comp. OOD}} \\  \hline 
			\multirow{6}{0.1\linewidth}{{{\bf SVHN}}}  
            & ATOM (q=0.0) & 0.07 & 99.97 & 5.47 & 98.52 & 7.02 & 98.00 & 96.33 & 49.52 \\ 
			& ATOM (q=0.125) & 1.30 & 99.63 & 34.97 & 94.97 & 39.61 & 82.92 & 99.92 & 6.30 \\ 
			& ATOM (q=0.25) & 1.36 & 99.60 & 41.98 & 94.30 & 52.39 & 71.34 & 99.97 & 1.35 \\ 
			& ATOM (q=0.5) & 2.11 & 99.46 & 44.85 & 93.84 & 59.72 & 65.59 & 99.97 & 3.15 \\ 
			& ATOM (q=0.75) & 2.91 & 99.26 & 51.33 & 93.07 & 66.20 & 57.16 & 99.96 & 2.04 \\ \hline 
			\multirow{6}{0.1\linewidth}{{{\bf CIFAR-10}}}  
            & ATOM (q=0.0) & 2.24 & 99.20 & 40.46 & 92.86 & 36.80 & 73.11 & 66.15 & 73.93 \\ 
			& ATOM (q=0.125) & 1.69 & 99.20 & 25.26 & 95.29 & 20.55 & 88.94 & 38.89 & 86.71 \\ 
			& ATOM (q=0.25) & 2.34 & 99.12 & 22.71 & 95.29 & 24.93 & 94.83 & 41.58 & 91.56 \\ 
			& ATOM (q=0.5) & 4.03 & 98.97 & 33.93 & 93.51 & 22.39 & 95.16 & 45.11 & 90.56 \\ 
			& ATOM (q=0.75) & 5.35 & 98.77 & 41.02 & 92.78 & 21.87 & 93.37 & 43.64 & 91.98 \\ \hline 
			\multirow{5}{0.1\linewidth}{{{\bf CIFAR-100}}}  
            & ATOM (q=0.0) & 44.38 & 91.92 & 99.76 & 60.12 & 68.32 & 65.75 & 99.80 & 49.85 \\ 
			& ATOM (q=0.125) & 26.91 & 94.97 & 98.35 & 71.53 & 34.66 & 87.54 & 98.42 & 68.52 \\ 
			& ATOM (q=0.25) & 32.43 & 93.93 & 97.71 & 72.61 & 40.37 & 82.68 & 97.87 & 65.19 \\ 
			& ATOM (q=0.5) & 32.30 & 93.06 & 93.15 & 71.96 & 38.72 & 88.03 & 93.44 & 69.15 \\ 
			& ATOM (q=0.75) & 38.56 & 91.20 & 97.59 & 58.53 & 62.66 & 78.70 & 97.97 & 54.89 \\ 
			\bottomrule
		\end{tabular}
	\end{adjustbox}
	\label{tab:ablation-study}
\end{table*}

\section{Theoretical Analysis}
\label{sec:analysis}
In this section, we provide theoretical  
insight on  mining informative outliers for robust OOD detection. We proceed with a brief summary of our key results. 
 
\para{Results overview.} 
At a high level, our analysis provides two important insights.
{\bf First}, we show that with informative auxiliary OOD data, \emph{less} in-distribution data is needed to build a robust OOD detector.
{\bf Second}, we show using outlier mining achieves a robust OOD detector in a more \emph{realistic} case when the auxiliary OOD data contains many outliers that are far from the decision boundary (and thus non-informative), and may contain some in-distribution data. The above two insights are important for building a robust OOD detector in practice, particularly because labeled in-distribution data is expensive to obtain while auxiliary outlier data is relatively cheap to collect. \emph{By performing outlier mining, one can effectively reduce the sample complexity while achieving strong robustness.} We provide the main results and intuition here and refer readers to Appendix~\ref{app:theory} for the details and the proofs. 

\subsection{Setup}
\para{Data model.}
To establish formal guarantees, we use a Gaussian $\mathcal{N}(\mu, \sigma^2 I)$ to model the in-distribution $\PX$ and the test OOD distribution can be any distribution largely supported outside a ball around $\mu$. We consider robust OOD detection under adversarial perturbation with bounded $\ell_\infty$ norm, i.e., the perturbation $\|\delta\|_\infty \le \epsilon$. Given $\mu \in \mathbb{R}^d, \sigma > 0, \gamma \in (0, \sqrt{d}), \epsilon_\tau > 0$, we consider the following data model: 
\begin{itemize}[topsep=0pt,itemsep=0pt,partopsep=0pt, parsep=0pt]
    \item {\bf $\PX$ (in-distribution data)} is
    $\mathcal{N}(\mu, \sigma^2 I)$.
    The in-distribution data $\{\*x_i\}_{i=1}^n$ is drawn from $\PX$.
    \item {\bf $\QX$ (out-of-distribution data)} can be any distribution from the family $\Qfamily = \{\QX: \Pr_{\*x \sim \QX}[ \|\*x - \mu\|_2 \le \tau ] \le  \epsilon_\tau \}$, where $\tau = \sigma \sqrt{d} + \sigma \gamma + \epsilon \sqrt{d}$. 
    \item {\bf Hypothesis class of OOD detector}: $\mathcal{G} = \{G_{u,r}(\*x): G_{u,r}(\*x) = 2\cdot \mathbb{I}[\|\*x - u\|_2 \le r]-1, u\in \mathbb{R}^d, r \in \mathbb{R}_+ \}$.
\end{itemize}
Here, $\gamma$ is a parameter indicating the margin between the in-distribution and OOD data, and $\epsilon_\tau$ is a small number bounding the probability mass the OOD distribution can have close to the in-distribution.

\para{Metrics.} For a detector $G$, we are interested in the False Negative Rate $\FNRr(G)$ and the worst False Positive Rate $\sup_{\QX\in \Qfamily}\FPRr(G;\QX, \Omega_{\infty,\epsilon}(\*x))$ over all the test OOD distributions $\Qfamily$ under $\ell_{\infty}$ perturbations of magnitude $\epsilon$. For simplicity, we denote them as $\FNRr(G)$ and $\FPRr(G; \Qfamily)$. 

While the Gaussian data model may be simpler than the practical data, its simplicity is desirable for our purpose of demonstrating our insights. Finally, the analysis can be generalized to mixtures of Gaussians which better models real-world data.

\subsection{Learning with Informative Auxiliary Data} 
 
We show that informative auxiliary outliers can reduce the sample complexity for in-distribution data. Note that learning a robust detector requires to estimate $\mu$ to distance $\gamma \sigma$, which needs $\Tilde\Theta(d/\gamma^2)$ in-distribution data, for example, one can compute a robust detector by:
\begin{align}
   u  = \bar{\*x} = \frac{1}{n} \sum_{i=1}^n \*x_i, \quad
   r  = (1 + \gamma/4\sqrt{d}) \hat\sigma, \label{eq:in_ur}  
\end{align} 
where $\hat\sigma^2  = \frac{1}{n} \sum_{i=1}^n \| {\*x}_i - \bar{\*x} \|_2^2.$
Then we show that with informative auxiliary data, we need much less in-distribution data for learning. We model the auxiliary data $\UX$ as a distribution over the sphere $\{\*x: \|\*x - \mu\|_2^2 = \sigma^2_o d\}$ for $\sigma_o > \sigma$, and assume its density is at least $\eta$ times that of the uniform distribution on the sphere for some constant $\eta > 0$, i.e., it's surrounding the boundary of $\PX$. Given $\{\*x_i\}_{i=1}^n$ from $\PX$ and $\{\Tilde{\*x}_i\}_{i=1}^{n'}$ from $\UX$, a natural idea is to compute $\bar{\*x}$ and $r$ as above as an intermediate solution, and refine it to have small errors on the auxiliary data under perturbation, i.e., find $u$ by minimizing a natural ``margin loss'': 
\begin{align} 
    \hspace{-4mm} u = \argmin_{p: \| p - \bar{\*x} \|_2 \le s}   
    & \frac{1}{n'} \sum_{i=1}^{n'} \max_{\|\delta\|_\infty \le \epsilon} \mathbb{I}\left[  \| \Tilde{\*x}_i + \delta - p\|_2 < t \right]  \label{eq:margin_opt}
\end{align}  
where $s, t$ are hyper-parameters to be chosen. We show that with $\Tilde{O}(d/\gamma^4)$ in-distribution data and sufficient auxiliary data can give a robust detector. See proof in Appendix~\ref{sec:learning-with-ideal-auxiliary-ood}.

\subsection{Learning with Informative Outlier Mining} 
In this subsection, we consider a more realistic data distribution where the auxiliary data can contain non-informative outliers (far away from the boundary), and in some cases mixed with in-distribution data. The non-informative outliers may not provide useful information to distinguish a good OOD detector statistically, which motivates the need for outlier mining.

\para{Uninformative outliers can lead to bad detectors.} To formalize, we model the non-informative (``easy" outlier) data as $Q_q = \mathcal{N}(0, \sigma_q^2 I )$, where $\sigma_q$ is large to ensure they are obvious outliers. The auxiliary data distribution $\UXb$ is then a mixture of $\UX$, $Q_q$ and $\PX$, where $Q_q$ has a large weight. Formally, $\UXb = \nu \UX + (1-2\nu) Q_q + \nu \PX$ for a small $\nu \in (0,1)$. Then we see that the previous learning rule cannot work: those robust detectors (with $u$ of distance $O(\sigma\gamma)$ to $\mu$) and those bad ones (with $u$ far away from $\mu$) cannot be distinguished. There is only a small fraction of auxiliary data from $\UX$ for distinguishing the good and bad detectors, while the majority (those from $Q_q$) do not differentiate them and some (those from $\PX$) can even penalize the good ones and favor the bad ones. 

\para{Informative outlier mining improves the detector with reduced sample complexity.} The above failure case suggests that a more sophisticated method is needed. Below we show that outlier mining can help to identify informative data and improve the learning performance. It can remove most data outside $\UX$, and keep the data from $\UX$, and the previous method can work after outlier mining. We first use in-distribution data to get an intermediate solution $\bar{\*x}$ and $r$ by equations~(\ref{eq:in_ur}).
Then, we use a simple thresholding mechanism to only pick points close to the decision boundary of the intermediate solution, which removes {\em non-informative outliers}. Specifically, we only select outliers with mild ``confidence scores'' w.r.t.\ the intermediate solution, i.e., the distances to $\bar{\*x}$ fall in some interval $[a, b]$:
\begin{align}
    S  := \{ i :  \|\Tilde{\*x}_i - \bar{\*x} \|_2 \in [a, b] , 1 \le i \le n'\}
\end{align}
The final solution $\uom$ is obtained by solving~\eqref{eq:margin_opt} on only $S$ instead of all auxiliary data. 
We can prove: 
\begin{proposition}\label{prop:om} {\bf (Error bound with outlier mining.)}
    Suppose $ \sigma^2\gamma^2 \ge C \epsilon\sigma_o d$ and $\sigma \sqrt{d} + C \sigma \gamma^2 <  \sigma_o \sqrt{d} < C \sigma\sqrt{d}$ for a sufficiently large constant $C$, and $\sigma_q \sqrt{d} > 2(\sigma_o\sqrt{d} + \|\mu\|_2)$. For some absolute constant $c$ and any $\alpha \in (0,1)$, if the number of in-distribution data $n \ge \frac{C d}{\gamma^4} \log \frac{1}{\alpha} $ and the number of auxiliary data $n' \ge \frac{\exp(C \gamma^4)}{\nu^2\eta^2}\log \frac{d\sigma }{\alpha}$, then there exist parameter values $s, t, a, b$ such that with probability $\ge 1 - \alpha$, the detector $G_{\uom, r}$ computed above satisfies:  
\begin{align*}
    \FNRr(G_{\uom,r}) \le \exp(-c \gamma^2), \quad
    \FPRr(G_{\uom,r}; \Qfamily)  \le \epsilon_\tau.
\end{align*}
\end{proposition}
This means that even in the presence of a large amount of uninformative or even harmful auxiliary data, we can successfully learn a good detector. Furthermore, this can reduce the sample size $n$ by a factor of $\gamma^2$. For example, when $\gamma = \Theta(d^{1/8})$, we only need $n = \Tilde\Theta(\sqrt{d})$, while in the case without auxiliary data, we need $n = \Tilde\Theta(d^{3/4})$. 

\para{Remark.} We note that when $\UX$ is as ideal as the uniform distribution over the sphere (i.e., $\eta=1$), then we can let $u$ be the average of points in $S$ after mining, which will require $n' = \Tilde{\Theta}(d/(\nu^2\gamma^2))$ auxiliary data, much less than that for more general $\eta$. 
We also note that our analysis and the result also hold for many other auxiliary data distributions $\UXb$, and the particular $\UXb$ used here is for the ease of explanation; see Appendix~\ref{app:theory} for more discussions.

\section{Related Work}
\label{sec:related}
 \para{OOD detection.}
~\cite{hendrycks2016baseline} introduced a baseline for OOD detection using the maximum softmax probability from a pre-trained network. 
Subsequent works improve the OOD uncertainty estimation by using deep ensembles~\cite{lakshminarayanan2017simple}, the calibrated softmax score~\cite{liang2018enhancing}, the Mahalanobis
distance-based confidence score~\cite{lee2018simple}, as well as the energy score~\cite{liu2020energy}.\@
Some methods regularize the model with auxiliary anomalous data that were either realistic~\cite{hendrycks2018deep,mohseni2020self,papadopoulos2019outlier} or artificially generated by GANs~\cite{lee2017training}. Several other works \cite{bevandic2018discriminative,malinin2018predictive,subramanya2017confidence} also explored regularizing the model to produce lower confidence for anomalous examples. Recent works have also studied the computational efficiency aspect of OOD detection~\cite{lin2021mood} and large-scale OOD detection on ImageNet~\cite{huang2021mos}.

\para{Robustness of OOD detection.} Worst-case aspects of OOD detection have been studied in \cite{hein2019relu,sehwag2019analyzing}. However, these papers are primarily concerned with $L_\infty$ norm bounded adversarial attacks, while our evaluation also includes common image corruption attacks. Besides, \cite{hein2019relu,meinke2019towards} only evaluate adversarial robustness of OOD detection on random noise images, while we also evaluate it on natural OOD images.~\cite{meinke2019towards} has  shown the first provable guarantees for worst-case OOD detection on some balls around uniform noise, and ~\cite{bitterwolf2020provable} studied the provable guarantees for worst-case OOD detection not only for noise but also for images from related but different image classification tasks. Our paper proposes ATOM which achieves state-of-the-art performance on a broader family of clean and perturbed OOD inputs. The key difference compared to prior work is introducing the informative outlier mining technique, which can significantly improve the generalization and robustness of OOD detection.

 \para{Adversarial robustness.} 
 Adversarial examples \cite{biggio2013evasion,goodfellow2014explaining,papernot2016limitations,szegedy2013intriguing} have received considerable attention in recent years. Many defense methods have been proposed to mitigate this problem. One of the most effective methods is adversarial training \cite{madry2017towards}, which uses robust optimization techniques to render deep learning models resistant to adversarial attacks.
\cite{carmon2019unlabeled,najafi2019robustness,uesato2019labels,zhai2019adversarially} showed that unlabeled data could improve adversarial robustness for classification. 

 \para{Hard example mining. } 
 Hard example mining was introduced in~\cite{sung1996learning} for training face detection models, where they gradually grew the set of background examples by selecting those examples for which the detector triggered a false alarm. The idea has been used extensively for object detection literature~\cite{felzenszwalb2009object,gidaris2015object,shrivastava2016training}. It also has been used extensively in deep metric learning~\cite{cui2016fine,harwood2017smart,simo2015discriminative,suh2019stochastic,wang2015unsupervised} and deep embedding learning~\cite{duan2019deep,smirnov2018hard,wu2017sampling,yuan2017hard}. Although hard example mining has been used in various learning domains, to the best of our knowledge, we are the first to explore it to improve the robustness of out-of-distribution detection.

\section{Conclusion}
\label{sec:conclusion}
In this paper, we propose {Adversarial Training with informative Outlier Mining} (ATOM), a method that enhances the robustness of the OOD detector. We show the merit of adaptively selecting the OOD training examples which the OOD detector is uncertain about. Extensive experiments show ATOM can significantly improve the decision boundary of the OOD detector, achieving state-of-the-art performance under a broad family of \emph{clean and perturbed} OOD evaluation tasks. We also provide a theoretical analysis that justifies the benefits of outlier mining. Further, our unified evaluation framework allows future research to examine the robustness of the OOD detector. We hope our research can raise more attention to a broader view of robustness in out-of-distribution detection. 

\section*{Acknowledgments}
The work is partially supported by Air Force Grant FA9550-18-1-0166, the National Science Foundation (NSF) Grants CCF-FMitF-1836978, IIS-2008559, SaTC-Frontiers-1804648 and CCF-1652140, and ARO grant number W911NF-17-1-0405. Jiefeng Chen and Somesh Jha are partially supported by the DARPA-GARD problem under agreement number 885000.

\bibliographystyle{splncs04}

\newpage
\appendix
\input{appendix}

\end{document}

%% file: appendix.tex
\begin{center}
	\textbf{\LARGE Supplementary Material}
\end{center}

 \begin{center}
	\textbf{\large ATOM: Robustifying Out-of-distribution Detection Using Outlier Mining}
\end{center}

The details about the theory and the experiment are provided in Section~\ref{app:theory} and Section~\ref{sec:details-of-experiments} respectively. 

\section{Theoretical Analysis}
\label{app:theory}

We consider the following data model to demonstrate how selecting informative outliers can help. In the proof below, we will use $C, c$ to denote some absolute constants; their values can change from line to line.

\begin{enumerate}
    \item The in-distribution $\PX$ is a Gaussian $\mathcal{N}(\mu, \sigma^2 I)$ with mean $\mu \in \mathbb{R}^d$ and variance $\sigma^2$. 
    \item The test OOD distribution can be any distribution largely supported outside a ball around $\mu$. More precisely, it can be any distribution from the family:
    $$
    \Qfamily = \left\{\QX: \Pr_{\*x \sim \QX}[ \|\*x - \mu\|_2 \le \tau ] \le  \epsilon_\tau \right\}
    $$
    where $\tau = \sigma \sqrt{d} + \sigma \gamma + \epsilon \sqrt{d}$,  $\gamma \in (0, \sqrt{d})$ is a parameter indicating some margin between the in-distribution and OOD distributions, and $\epsilon_\tau$ is a small number bounding the probability mass the OOD distribution can have close to the in-distribution. 
    \item The hypothesis class for detectors is 
    $$
    \mathcal{G} = \left\{G_{u,r}(\*x): G_{u,r}(\*x) = 2\cdot \mathbb{I}[\|\*x - u\|_2 \le r]-1, u\in \mathbb{R}^d, r \in \mathbb{R}_+ \right \}.
    $$  
\end{enumerate}

Recall that we consider $\ell_\infty$ attack with adversarial budget $\epsilon > 0$. For a detector $G$ and test OOD distribution family $\Qfamily$, we are interested in the False Negative Rate $\FNRr(G)$ and worst False Positive Rate $\sup_{\QX\in \Qfamily}\FPRr(G;\QX, \Omega_{\infty,\epsilon}(\*x))$ over $\QX\in\Qfamily$ under $\ell_{\infty}$ perturbations of magnitude $\epsilon$. For simplicity, we denote them as $\FNRr(G)$ and $\FPRr(G; \Qfamily)$ in our proofs. 

We note that the data model is set up such that there exists a robust OOD detector with good FPR and FNR (Proposition~\ref{prop:exist}), while using only in-distrbution data to learn a good robust OOD detector one needs sufficiently amount of them, i.e., $\Tilde\Theta(d/\gamma^2)$ (Propostion~\ref{prop:without}). (Therefore, we assume $\gamma < \sqrt{d}$ to avoid the trivial case.)

\begin{proposition} \label{prop:exist}
The detector $G_{u,r}(\*x)$ with $u = \mu$ and $r = \sigma \sqrt{d} + \sigma \gamma$ satisfies:
\begin{align}
    \FNRr(G_{u,r}) & \le \exp(-c\gamma^2),
    \\
    \FPRr(G_{u,r}; \Qfamily) & \le \epsilon_\tau,
\end{align}
for some absolute constant $c>0$.
\end{proposition}
\begin{proof}
The first statement follows from the concentration of the norm of $\*x - \mu$:
\begin{align}
    \FNRr(G_{u,r}(\*x)) &= \mathbb{E}_{\*x\sim \PX} \mathbb{I}[\|\*x-u\|_2 > r] \\
    &= \Pr_{\*x\sim \mathcal{N}(\*0, I)}[\| \*x \|_2 > \sqrt{d} + \gamma] \\
    &\le \exp(-c\gamma^2)
\end{align}
the second statement follows from the definition of $\mathcal{Q}$:
\begin{align}
    \FPRr(G_{u,r}; \Qfamily) & =  \sup_{\QX \in \Qfamily} \mathbb{E}_{\*x\sim \QX} \max_{\|\delta\|_\infty \leq \epsilon }  \mathbb{I}[\|\*x+\delta-u\|_2\le r] \\
    & \le \sup_{\QX \in \Qfamily} \mathbb{E}_{\*x\sim \QX}  \mathbb{I}[\|\*x-\mu\|_2 - \sqrt{d}\epsilon \le \sigma \sqrt{d} + \sigma \gamma]  \\
    & = \sup_{\QX \in \Qfamily} \Pr_{\*x\sim \QX}[\|\*x-\mu\|_2 \le \tau] \le \epsilon_\tau.
\end{align} 
\end{proof}

\subsection{Learning Without Auxiliary OOD Data}
Given in-distribution data $\*x_1, \*x_2, \ldots, \*x_n$, we consider the detector $G_{u,r}(\*x)$ with 
\begin{align}
    u & = \bar{\*x} = \frac{1}{n} \sum_{i=1}^n \*x_i \\
    r & = (1 + \gamma/4\sqrt{d}) \hat\sigma, \textrm{~where~} \hat\sigma^2  = \frac{1}{n} \sum_{i=1}^n \| {\*x}_i - \bar{\*x} \|_2^2.
\end{align}
By concentration bounds, we can show that this leads to a good solution, if the number of data points $n$ is sufficiently large.

\begin{proposition}\label{prop:without}
If the number of in-distribution data points  $n \ge  \frac{C d }{\gamma^2}\log\frac{1}{\alpha} + \frac{C d }{\gamma^4 }\log\frac{1}{\alpha}$
for $\alpha \in (0,1)$ and some sufficiently large constant $C$, then with probability at least $1 - \alpha$,
the detector $G_{u,r}(\*x)$ with $u = \frac{1}{n} \sum_{i=1}^n \*x_i$ and $r = (1 + \gamma/4\sqrt{d}) \sqrt{\frac{1}{n} \sum_{i=1}^n \|\*x_i - u\|_2^2}$ satisfies:
\begin{align}
    \FNRr(G_{u,r}) & \le  \exp(-c\gamma^2),
    \\
    \FPRr(G_{u,r}; \Qfamily) & \le \epsilon_\tau,
\end{align}
for some absolute constant $c$.
\end{proposition}
\begin{proof}
We will prove that $u$ is close to $\mu$ and $r$ is close to $(1 + c\gamma/\sqrt{d}) \sigma \sqrt{d}$. 

First, consider $u - \mu$. Let $\*x_i = \mu + \sigma g_i$ where $g_i \sim \mathcal{N}(0, I)$, and $\bar{g} = \frac{1}{n}\sum_{j=1}^n g_j$. Then $u - \mu = \sigma \cdot \frac{1}{n} \sum_{i=1}^n g_i = \sigma \bar{g}$. By the concentration of sub-gaussian variables, with probability $\ge 1 - \alpha/2$,
\begin{align}
    \left\| \bar{g}\right\|_2 = \left\| \frac{1}{n} \sum_{i=1}^n g_i \right\|_2 \le c \sqrt{\frac{d}{n} \log \frac{1}{\alpha}}
\end{align} 
and thus
\begin{align}
     \|u - \mu\|_2 \le c \sigma \sqrt{\frac{d}{n} \log \frac{1}{\alpha}}\le c \sigma \gamma.
\end{align} 
Next, let $\hat\sigma^2 := \frac{1}{n} \sum_{i=1}^n \|\*x_i - u\|_2^2$. Then
\begin{align} 
    \hat\sigma^2 = \frac{\sigma^2}{n} \sum_{i=1}^n \left\| g_i - \bar{g}  \right\|_2^2 = \sigma^2 \left( \frac{1}{n} \sum_{i=1}^n \| g_i\|_2^2 - \sigma^2\left\| \bar{g} \right\|_2^2 \right).
\end{align}
By sub-exponential concentration, we have with probability $\ge 1-\alpha/2$,
\begin{align} 
    \left|\frac{1}{n} \sum_{i=1}^n \|g_i\|_2^2  - d \right| \le c\sqrt{ \frac{d }{n}\log \frac{1}{\alpha} }.
\end{align} 
Then
\begin{align}
    |\hat\sigma^2 - \sigma^2 d | \le  c\sigma^2 \sqrt{ \frac{d }{n}\log \frac{1}{\alpha} } \le c \sigma^2 \gamma^2.
\end{align}
Given
\begin{align}
     \|u - \mu\|_2 \le  c \sigma \gamma, |\hat\sigma - \sigma \sqrt{d} | \le \sqrt{|\hat\sigma^2 - \sigma^2 d |} \le c \sigma\gamma,
\end{align} 
for sufficiently small $c < 1/16$, we have $(1 + \gamma/8\sqrt{d}) \sigma \sqrt{d} \le r  \le (1 + \gamma/2\sqrt{d}) \sigma \sqrt{d}$. Then the statement follows. 
\end{proof}

\subsection{Learning With Informative Auxiliary OOD Data}
\label{sec:learning-with-ideal-auxiliary-ood}
The informative auxiliary data will be a distribution around the boundary of the in-distribution and the outlier distributions.
Assume we have access to  auxiliary OOD data $\*x$ from an ideal distribution $\UX$ where:
\begin{itemize} 
    \item $\UX$ is a distribution over the sphere $\{\*x: \|\*x - \mu\|_2^2 = \sigma^2_o d\}$ for some $\sigma_o > \sigma$, and its density is at least $\eta$ times that of the uniform distribution on the sphere for some constant $\eta > 0$.
\end{itemize}

Given in-distribution data $\*x_1, \*x_2, \ldots, \*x_n$ from $\PX$ and  auxiliary OOD data $\Tilde{\*x}_1, \Tilde{\*x}_2, \ldots, \Tilde{\*x}_{n'}$ from $\UX$, we now design a good detector which only requires a small number $n$ of in-distribution data. The radius $r$ can be estimated using a small number of in-distribution data as above. The challenge is to learn a good $u$, ideally close to $\mu$. Given an outlier data point $\Tilde{\*x}$, a natural idea frequently used in practice is then to find a $u$ so that not so many outliers (potentially with adversarial perturbations) can be close to $u$, so that $G_{u,r}$ will be able to detect the outliers. 
Furthermore, we know that the $\mu$ is not far from $\bar{\*x}$, so we only need to search $u$ near $\bar{\*x}$.
We thus come to the following learning rule:
\begin{align}
   \hat\sigma^2 & = \frac{1}{n} \sum_{i=1}^n \| {\*x}_i - \bar{\*x} \|_2^2,  \textrm{~where~} \bar{\*x} = \frac{1}{n} \sum_{i=1}^n \*x_i, 
   \\
   r & = (1 + \gamma/4\sqrt{d}) \hat\sigma, \label{eqn:ideal_r}
   \\
    u & \leftarrow \argmin_{p: \| p - \bar{\*x} \|_2 \le s} L_{t}(p) := \frac{1}{n'} \sum_{i=1}^{n'} \max_{\|\delta\|_\infty \le \epsilon} \mathbb{I}\left[  \| \Tilde{\*x}_i + \delta - p  \|_2 < t \right]. \label{eqn:ideal_u}
\end{align} 
where $s$ and $t$ are some hyper-parameters to be determined. 

\begin{lemma} \label{lem:with}
Suppose $ \sigma^2\gamma^2 > C \epsilon\sigma_o d$, and the number of in-distribution data points 
$n \ge \frac{C d}{\gamma^4} \log \frac{1}{\alpha} + \frac{C \sigma^2 d}{s^2} \log\frac{1}{\alpha}$ and the number of auxiliary data $n' \ge \frac{ \exp(C s^2 /\sigma_o^2)}{\eta}\log \frac{d\sigma }{\alpha}$ for $\alpha \in (0,1)$ and for some sufficiently large constant $C$, then there exists proper parameter values $t$ such that with probability at least $1 - \alpha$,
the detector $G_{u,r}(\*x)$ with $r$ from \eqref{eqn:ideal_r} and $u$ from \eqref{eqn:ideal_u} satisfies:
\begin{align}
    \FNRr(G_{u,r}(\*x)) & \le \exp(-c \gamma^2),
    \\
    \FPRr(G_{u,r}; \Qfamily) & \le \epsilon_\tau,
\end{align}
for some absolute constant $c$.
\end{lemma}

\begin{proof}
Set $B = \sigma \gamma/4$, and set $t$ to be any value such that $ \sigma^2_o d - B^2 \le t^2$ and $ (t + \epsilon\sqrt{d})^2 < \sigma^2_o d$.  (Note that $\sigma_o^2 d > B^2$, since $\sigma_o^2 d > \sigma^2 d$ and $\gamma<\sqrt{d}$. Furthermore, when $ \sigma^2\gamma^2 < C \epsilon\sigma_o d$ for some sufficiently constant $C$, we have $\sqrt{\sigma_o^2 d - B^2}   <  \sigma_o \sqrt{d} - \epsilon \sqrt{d}$, so we can select $t$ within this range which will satisfy our requirements for $t$.)

First, consider $r$ and $\bar{\*x}$. Similar to the proof of Proposition~\ref{prop:without}, when $n \ge \frac{cd}{\gamma^4} \log \frac{1}{\alpha}$, we have with probability $1-\alpha/8$, 
\begin{align}
    (1 + \gamma/8\sqrt{d}) \sigma \sqrt{d} \le r  \le (1 + \gamma/2\sqrt{d}) \sigma \sqrt{d}.
\end{align}
Also for $\bar{\*x}$, when $n \ge \frac{c\sigma^2 d}{s^2} \log\frac{1}{\alpha}$, we have with probability $1-\alpha/8$, 
\begin{align}
    \|\mu - \bar{\*x}\|_2 \le c\sigma \sqrt{\frac{d}{n} \log \frac{1}{\alpha} } \le s.
\end{align}
This ensures $\mu$ will be included in the feasible set of \eqref{eqn:ideal_u}.

Now, consider $u$. Let $P$ be the set of $p$ with $\|p - \bar{\*x}\|_2 \le s$ but $\| p - \mu \|_2 > B$:
\begin{align}
    P := \{ p \in \mathbb{R}^d: \|p - \bar{\*x}\|_2 \le s, \| p - \mu \|_2 > B\}.
\end{align}
We will show that for any $p \in P$, with probability $1 - \alpha/8$, $L_t(p) > L_t(\mu) = 0$, and thus \eqref{eqn:ideal_u} finds an $u$ with $\| u - \mu\|_2 \le  \sigma \gamma/4$, which then leads to the theorem statements.

We begin by noting that 
\begin{align}
    \max_{\|\delta\|_\infty \le \epsilon} \mathbb{I}\left[  \| \Tilde{\*x}_i + \delta - p  \|_2 < t \right] \le \mathbb{I}\left[   \| \Tilde{\*x}_i - p  \|_2 < t + \epsilon \sqrt{d} \right].
\end{align}
Since $t + \epsilon \sqrt{d} < \sigma_o\sqrt{d}$, we have $L_t(\mu) = 0$.

We now consider a fixed $p \in P$. 
\begin{align}
   \mathbb{E}\left\{ \max_{\|\delta\|_\infty \le \epsilon} \mathbb{I}\left[  \| \Tilde{\*x}_i + \delta - p  \|_2 < t \right] \right \} 
   \ge  \mathbb{E} \left\{ \mathbb{I}\left[   \| \Tilde{\*x}_i - p  \|_2 < t  \right] \right\} = \Pr\left[ \|p - \Tilde{\*x}_i \|_2 < t   \right].
\end{align}
Let $\Delta = p - \mu$.   
Then since $ \sigma^2_o d - B^2 \le t^2 $ and $\| \Delta\|_2 \ge B$, we have $ \sigma^2_o d - \|\Delta\|_2^2 \le t^2 $, and thus any $\Tilde{\*x}_i$ from $\UX$ whose projection onto the direction $p - \mu$ has distance larger than $\|\Delta\|_2$ from $\mu$ will satisfy $\|p - \Tilde{\*x}_i \|_2 < t $. Then 
\begin{align}
   \mathbb{E}\left\{ \max_{\|\delta\|_\infty \le \epsilon} \mathbb{I}\left[  \| \Tilde{\*x}_i + \delta - p  \|_2 < t \right] \right \} 
   \ge   \Pr\left[ \|p - \Tilde{\*x}_i \|_2 < t - \epsilon\sqrt{d} \right] 
   \ge \frac{\eta A_d\left(1 - \|\Delta\|_2 / (\sigma_o \sqrt{d}) \right) }{A_d}
\end{align}
where $A_d(v)$ is the area of the spherical cap of the unit hypersphere in dimension $d$ with height $v$, and $A_d$ is the area of the whole unit hypersphere.
 
By the bound in~\cite{becker2016new}, we have
\begin{align}
   A_d(v)/A_d = d^{\Theta(1)} [1 - (1-v)^2 ]^{d/2}.  
\end{align}
Since $\|\Delta\|_2 \le \|p - \bar{\*x}\|_2 + \| \bar{\*x} - \mu\|_2 \le 2s$, we have  
\begin{align}
   \mathbb{E}\left\{ \max_{\|\delta\|_\infty \le \epsilon} \mathbb{I}\left[  \| \Tilde{\*x}_i + \delta - p  \|_2 < t \right] \right \}  \ge \frac{\eta A_d\left(1 - \|\Delta\|_2 / (\sigma_o \sqrt{d}) \right) }{A_d} = \eta d^{\Theta(1)} \left [1 - \frac{4s^2}{\sigma_o^2 d} \right]^{d/2} := q. 
\end{align}
Then when $n' \ge \frac{c}{q}\log \frac{1}{\alpha'}$,
\begin{align}
    \Pr[L_t(p) = 0] \le (1  - q)^{n'} \le e^{-q n'} \le \alpha'/8. 
\end{align}
Then a net argument on $P$ proves that for the given $n'$,  we have with probability $\ge 1 - \alpha/8$, any $p \in P$ has $L_t(p) > 0$. This completes the proof. 
\end{proof}

By setting $s = \sigma \gamma^2$ and assuming $\sigma_o^2 = O(\sigma^2)$, we have the following corollary. 

\begin{proposition} \label{prop:with}
Suppose $ \sigma^2\gamma^2 \ge C \epsilon\sigma_o d$ for some sufficiently constant $C$ and  $ \sigma^2 < \sigma_o^2 < C'\sigma^2$ for some absolute constant $C'$. 
For $\alpha \in (0,1)$, if the number of in-distribution data points 
$n \ge \frac{C d}{\gamma^4} \log \frac{1}{\alpha} $ and the number of auxiliary data $n' \ge \frac{\exp(C \gamma^4)}{\eta} \log \frac{d\sigma }{\alpha}$, then there exists proper parameter values $s, t$ such that with probability at least $1 - \alpha$,
the detector $G_{u,r}(\*x)$ with $r$ from \eqref{eqn:ideal_r} and $u$ from \eqref{eqn:ideal_u} satisfies:
\begin{align}
    \FNRr(G_{u,r}(\*x)) & \le \exp(-c \gamma^2),
    \\
    \FPRr(G_{u,r}; \Qfamily) & \le \epsilon_\tau,
\end{align}
for some absolute constant $c$.
\end{proposition}

Then we can see that this can reduce the sample size $n$ by a factor of $\gamma^2$. For example, when $\gamma = \Theta(d^{1/8})$, we only need $n = O(\sqrt{d} \log(1/\alpha))$, while in the case without auxiliary data, we need $n = O(d^{3/4} \log(1/\alpha))$.

\subsection{Learning With Informative Outlier Mining} \label{app:om}
The above example shows the benefit of having auxiliary OOD data for training. All the auxiliary OOD data given in the example are implicitly related to the ideal parameter for the detector $\mu$ and thus are informative for learning the detector. However, this may not be the case in practice: typically only part of the auxiliary OOD data are informative, while the remaining are not very useful or even can be harmful for the learning. Here we study such an example, and shows that how outlier mining can help to identify informative data and improve the learning performance.

The practical auxiliary data can have a lot of obvious outliers not on the boundary and also quite a few in-distribution data mixed. We model the former data as $Q_q = \mathcal{N}(0, \sigma_q^2 I )$ where $\sigma_q$ is large compared to $\|\mu\|_2$, $ \sigma$, and $\sigma_o$. The auxiliary data distribution $\UXb$ is then a mixture of $\UX$ and $Q_q$ where $Q_q$ has a large weight. 
\begin{itemize}
    \item $\UXb = \nu \UX + (1-2\nu) Q_q + \nu \PX$ for a small $\nu \in (0,1)$, where $Q_q = \mathcal{N}(0, \sigma_q^2 I )$ for some large $\sigma_q$.
\end{itemize}  
That is, the distribution is defined by the following process: with probability $\nu$ sample the data from the informative part $\UX$, and with probability $1-2\nu$ sample from the uninformative part $Q_q$, and with probability $\nu$ sample from the in-distribution. 

Then the previous simple method will not work, and a more sophisticated method is needed. Below we show that outlier mining can remove most data outside $\UX$, and keep the data from $\UX$, and the previous method can work after outlier mining.

Suppose the algorithm gets $n$ in-distribution data $\Sin = \{\*x_1,\*x_2, \ldots, \*x_n\}$ i.i.d.\ from $\PX$ and $n'$ auxiliary data $\Sau = \{\Tilde{\*x}_1,\Tilde{\*x}_2, \ldots, \Tilde{\*x}_{n'}\}$ from $\UXb$ for training. 
Specifically, we first use in-distribution data to get an intermediate solution $\bar{\*x}$ and $r$:
\begin{align}
   \hat\sigma^2 & = \frac{1}{n} \sum_{i=1}^n \| {\*x}_i - \bar{\*x} \|_2^2,  \textrm{~where~} \bar{\*x} = \frac{1}{n} \sum_{i=1}^n \*x_i, 
   \\
   r & = (1 + \gamma/4\sqrt{d}) \hat\sigma.
\end{align} 
Then, we use a simple thresholding mechanism to only pick points close to the decision boundary of the intermediate solution, which removes {\em non-informative outliers}.
Specifically, we only select outliers $\Tilde{\*x}$ with mild ``confidence scores'' w.r.t.\ the intermediate solution, i.e., the distances to $\bar{\*x}$ fall in an interval $[a, b]$:
\begin{align}
    S  := \{ i :  \|\Tilde{\*x}_i - \bar{\*x} \|_2 \in [a, b] , 1 \le i \le n'\}
\end{align}
The final solution $u$ is then by:
\begin{align} 
    \uom & \leftarrow \argmin_{p: \| p - \bar{\*x} \|_2 \le s} L_{S, t}(p) := \frac{1}{|S|} \sum_{i \in S} \max_{\|\delta\|_\infty \le \epsilon} \mathbb{I}\left[ \ \| \Tilde{\*x}_i + \delta - p\|_2 < t \right]. \label{eqn:om_u}
\end{align} 
We can prove the following:

\begin{lemma}   \label{lem:om}  
Suppose $ \sigma^2\gamma^2 \ge C \epsilon\sigma_o d$ and $\sigma_o \sqrt{d} > \sigma \sqrt{d} + C s$ for a sufficiently large constant $C$, and $\sigma_q \sqrt{d} > 2(\sigma_o\sqrt{d} + \|\mu\|_2)$. For some absolute constant $c$ and any $\alpha \in (0,1)$, if the number of in-distribution data $n \ge \frac{Cd}{\gamma^4} \log \frac{1}{\alpha} + \frac{C \sigma^2 d}{s^2} \log\frac{1}{\alpha}$ and the number of auxiliary data $n' \ge \frac{\exp(C s^2 /\sigma_o^2)}{\nu^2 \eta^2}\log \frac{d\sigma }{\alpha}$, then there exists parameter values $t, a, b$ such that with probability $\ge 1 - \alpha$, the detector $G_{\uom, r}$ computed above satisfies:  
\begin{align}
    \FNRr(G_{\uom,r}) & \le \exp(-c \gamma^2),
    \\
    \FPRr(G_{u,r}; \Qfamily) & \le \epsilon_\tau.
\end{align} 
\end{lemma}

\begin{proof}
Following the proof as in Proposition~\ref{prop:with}, we have the same guarantees for $r$ and $\bar{\*x}$, i.e., with probability $1-\alpha/4$, 
\begin{align}
    (1 + \gamma/8\sqrt{d}) \sigma \sqrt{d} \le r  \le (1 + \gamma/2\sqrt{d}) \sigma \sqrt{d},
    \\
    \|\mu - \bar{\*x}\|_2 \le c\sigma \sqrt{\frac{d}{n} \log \frac{1}{\alpha} } \le s.
\end{align} 
This ensures $\mu$ will be included in the feasible set of \eqref{eqn:om_u}.

Set $B = \sigma \gamma/4$, and set $t$ to be any value such that $ \sigma^2_o d - B^2 \le t^2$ and $ (t + \epsilon\sqrt{d})^2 < \sigma^2_o d$.  (As before, the assumptions make sure we can select such a $t$. We have $\sigma_o^2 d > B^2$, since $\sigma_o^2 d > \sigma^2 d$ and $\gamma<\sqrt{d}$. Furthermore, when $ \sigma^2\gamma^2 < C \epsilon\sigma_o d$ for some sufficiently constant $C$, we have $\sqrt{\sigma_o^2 d - B^2}   <  \sigma_o \sqrt{d} - \epsilon \sqrt{d}$, so we can select $t$ within this range which will satisfy our requirements for $t$.)

Set $a = \sigma_o \sqrt{d} - s$, and $b = \sigma_o \sqrt{d} + s$. Let $P$ be the set of $p$ with $\|p - \bar{\*x}\|_2 \le s$ but $\| p - \mu \|_2 > B$:
\begin{align}
    P := \{ p \in \mathbb{R}^d: \|p - \bar{\*x}\|_2 \le s, \| p - \mu \|_2 > B\}.
\end{align}
We will show that for any $p \in P$, with probability $1 - \alpha/4$, $L_t(p) > L_t(\mu) = 0$, and thus \eqref{eqn:om_u} finds a $\uom$ with $\| \uom - \mu\|_2 \le B =  \sigma \gamma/4$, which then leads to the theorem statements.

First, we show that $L_{S,t}(p)$ is large for any $p\in P$. The same proof as in Proposition~\ref{prop:with} shows that for any $p$ as described above, for an $\Tilde{\*x}_i$ from $\UX$, we have
\begin{align}
   \mathbb{E}\left\{ \max_{\|\delta\|_\infty \le \epsilon} \mathbb{I}\left[  \| \Tilde{\*x}_i + \delta - p  \|_2 < t \right] \right \}  \ge
   \Pr\left[ \|p - \Tilde{\*x}_i \| < t  \right] \ge   \eta d^{\Theta(1)} \left [1 - \frac{4s^2}{\sigma_o^2 d} \right]^{d/2} := q. 
\end{align}
We note that all points from $U_X$ will be in $S$ for the given $a$ and $b$.
Then by the Chernoff's inequality (for multiplicative factors), when $n' \ge \frac{c}{q\nu}\log \frac{1}{\alpha'}$, with probability $1 - \alpha'$, we have
\begin{align}
    L_{S,t}(p) > \frac{q \nu}{2}.
\end{align}
Then a net argument on $P$ proves that with the given $n'$, we have with probability $\ge 1 - \alpha/8$, any $p \in P$ has $L_{S,t}(p) > q \nu/2$. 

Next, we will show that $L_{S,t}(\mu)$ is small. 
We first note that since $t + \epsilon \sqrt{d} < \sigma_o\sqrt{d}$, all points from $\UX$ will contribute 0 to $L_{S,t}(\mu)$. Furthermore, most points from $P_X$ and $Q_q$ are filtered outside $S$.
\begin{align}
    \Pr_{\*x \sim \PX}\left[ \|\*x - \bar{\*x} \|_2 \in [a, b] \right]  
    & \le \Pr_{\*x \sim \PX}\left[ \|\*x - \bar{\*x} \|_2 \ge a \right]
    \\
    & \le \Pr_{\*x \sim \PX}\left[ \|\*x - \mu \|_2 \ge a - s = \sigma_o\sqrt{d} - 2 s  \right]
    \\
    & \le e^{-c (\sigma_o\sqrt{d} - 2 s)^2/\sigma^2} := q_1,
\end{align}
and 
\begin{align}
    \Pr_{\*x \sim Q_q}\left[ \|\*x - \bar{\*x} \|_2 \in [a, b] \right]  
    & \le \Pr_{\*x \sim Q_q}\left[ \|\*x - \bar{\*x} \|_2 \le b \right]
    \\
    & \le \Pr_{\*x \sim Q_q}\left[ \|\*x\|_2 \le b + s  + \|\mu\|_2 \right]
    \\
    & \le e^{-c (\sigma_q \sqrt{d} - \sigma_o \sqrt{d} -2s - \|\mu\|_2)^2} := q_2. 
\end{align}
Then by Hoeffding's inequality, when $n' \ge \frac{c }{q^2 \nu^2} \log\frac{1}{\alpha}$, we have with probability $\ge 1 - \alpha/8$,
\begin{align}
    L_{S,t}(\mu) < qv/4 + \nu q_1 + (1-\nu)q_2.
\end{align}
Note that $q \ge \eta d^{\Theta(1)} e^{-c s^2/\sigma_o^2}$. 
Then when $\sigma_o \sqrt{d} > \sigma \sqrt{d} + C s$, we have $q_1 \le e^{-cd}$. Since $d$ is sufficiently large compared to $s^2 / \sigma_o^2$, $q_1 \le q/8$.
Similarly, when $\sigma_q \sqrt{d} > 2(\sigma_o\sqrt{d} + \|\mu\|_2)$, we have $\sigma_q \sqrt{d} - \sigma_o \sqrt{d} -2s - \|\mu\|_2 > \sigma_q \sqrt{d}/4$, and thus $q_2 \le \nu q/8$. 
Therefore, $L_{S,t}(\mu) < q\nu/2$.

In summary, with probability $\ge 1-\alpha$, $L_{S,t}(\mu) < L_{S,t}(p)$ for any $p \in P$. This then completes the proof.
\end{proof}

By setting $s = \sigma \gamma^2$ and assuming $\sigma_o^2 = O(\sigma^2)$, we have the following corollary.

\begin{proposition}[Restatement of Proposition~\ref{prop:om}] {\bf (Error bound with outlier mining.)}
    Suppose $ \sigma^2\gamma^2 \ge C \epsilon\sigma_o d$ and $\sigma \sqrt{d} + C \sigma \gamma^2 <  \sigma_o \sqrt{d} < C \sigma\sqrt{d}$ for a sufficiently large constant $C$, and $\sigma_q \sqrt{d} > 2(\sigma_o\sqrt{d} + \|\mu\|_2)$. For any $\alpha \in (0,1)$ and some absolute constant $c$, if the number of in-distribution data $n \ge \frac{C d}{\gamma^4} \log \frac{1}{\alpha} $ and the number of auxiliary data $n' \ge \frac{\exp(C \gamma^4)}{\nu^2\eta^2}\log \frac{d\sigma }{\alpha}$, then there exist parameter values $s, t, a, b$ such that with probability $\ge 1 - \alpha$, the detector $G_{\uom, r}$ computed above satisfies:  
\begin{align*}
    \FNRr(G_{\uom,r}) & \le \exp(-c \gamma^2),
    \\
    \FPRr(G_{\uom,r}; \Qfamily) & \le \epsilon_\tau.
\end{align*}
\end{proposition}

This means that even in the presence of a large amount of uninformative or even harmful auxiliary data, we can successfully learn a good detector. Furthermore, this can reduce the sample size $n$ by a factor of $\gamma^2$. For example, when $\gamma = \Theta(d^{1/8})$, we only need $n = O(\sqrt{d} \log(1/\alpha)$, while in the case without auxiliary data, we need $n = O(d^{3/4} \log(1/\alpha))$. 

\subsection{Learning with Ideal $\UX$} 
Consider the same setting in the above Section~\ref{app:om} but assume that the $\UX$ in the mixture $\UXb$ is as ideal as a uniform distribution:
\begin{itemize} 
    \item $\UX$ is the uniform distribution over the sphere $\{\*x: \|\*x - \mu\|_2^2 = \sigma^2_o d\}$ where $\sigma_o^2 > \sigma^2$. 
\end{itemize}

In this case we can do the outlier mining as above, but finally compute $\uom$ by averaging the auxiliary data points selected in $S$. That is:
\begin{align}
   \hat\sigma^2 & = \frac{1}{n} \sum_{i=1}^n \| {\*x}_i - \bar{\*x} \|_2^2,  \textrm{~where~} \bar{\*x} = \frac{1}{n} \sum_{i=1}^n \*x_i, 
   \\
   r & = (1 + \gamma/4\sqrt{d}) \hat\sigma,
   \\
    S & = \{ i :  \|\Tilde{\*x}_i - \bar{\*x} \|_2 \in [a, b] , 1 \le i \le n'\},
    \\
    \uom & = \frac{1}{|S|} \sum_{\bar{\*x} \in S} \bar{\*x}.
\end{align} 
Then we can prove essentially the same guarantees but with much fewer auxiliary data: 

\begin{proposition} \label{prop:om_idealUX} 
    Suppose $ \sigma^2\gamma^2 \ge C \epsilon\sigma_o d$ and $\sigma \sqrt{d} + C \sigma \gamma^2 <  \sigma_o \sqrt{d} < C \sigma\sqrt{d}$ for a sufficiently large constant $C$, and $\sigma_q \sqrt{d} > 2(\sigma_o\sqrt{d} + \|\mu\|_2)$. For any $\alpha \in (0,1)$ and some absolute constant $c$, if the number of in-distribution data $n \ge \frac{C d}{\gamma^4} \log \frac{1}{\alpha} $ and the number of auxiliary data $n' \ge \frac{C d}{\gamma^2\nu^2}\log \frac{d}{\alpha}$ and $d\ge C\log\frac{n'}{\alpha}$, then there exist parameter values $a, b$ such that with probability $\ge 1 - \alpha$, the detector $G_{\uom, r}$ computed above satisfies:  
\begin{align*}
    \FNRr(G_{\uom,r}) & \le \exp(-c \gamma^2),
    \\
    \FPRr(G_{\uom,r}; \Qfamily) & \le \epsilon_\tau.
\end{align*}
\end{proposition}
\begin{proof}
Following the same setting of $a,b$ and the same proof as in Lemma~\ref{lem:om}, we can show that all points from $\UX$ will be included in $S$, and most points from $P_X$ and $Q_q$ are filtered outside $S$:
\begin{align}
    \Pr_{\*x \sim \PX}\left[ \|\*x - \bar{\*x} \|_2 \in [a, b] \right]   
    & \le e^{-c (\sigma_o\sqrt{d} - 2 s)^2/\sigma^2} < e^{-c'd},
    \\
    \Pr_{\*x \sim Q_q}\left[ \|\*x - \bar{\*x} \|_2 \in [a, b] \right]   
    & \le e^{-c (\sigma_q \sqrt{d} - \sigma_o \sqrt{d} -2s - \|\mu\|_2)^2} < e^{-c'd},
\end{align}
for some absolute constant $c'$. Then with probability $\ge 1 - n' e^{-2c'd}$, $S$ is exactly all the auxiliary points from $\UX$. By the sub-gaussian concentration of the uniform distribution over sphere, we have with probability $1 - \alpha/2$,
\begin{align}
    \|\uom - \mu\|_2 \le c\sigma_o \sqrt{\frac{d}{ |S|} \log \frac{1}{\alpha} }.
\end{align}
For the given $n'$, we have the theorem.
\end{proof}

\section{Details of Experiments}
\label{sec:details-of-experiments}

\subsection{Experimental Settings}
\label{sec:more-experiment-settings}
\para{Software and Hardware.} We run all experiments with PyTorch and NVIDIA GeForce RTX 2080Ti GPUs. 

\para{Number of Evaluation Runs.} We run all experiments once with fixed random seeds.  

\para{In-distribution Datasets. } We use SVHN~\cite{netzer2011reading}, CIFAR-10 and CIFAR-100~\cite{krizhevsky2009learning} as in-distribution datasets. SVHN has 10 classes and contains 73,257 training images. CIFAR-10 and CIFAR-100 have 10 and 100 classes, respectively. Both datasets consist of 50,000 training images and 10,000 test images. 

\para{Auxiliary OOD Datasets.} We provide the details of auxiliary OOD datasets below. For each auxiliary OOD dataset, we use random cropping with padding of 4 pixels to generate $32\times 32$ images, and further augment the data by random horizontal flipping. We don't use any image corruptions to augment the data. 
\begin{enumerate}
    \item \textbf{TinyImages.} 80 Million Tiny Images (TinyImages) \cite{torralba200880} is a dataset that contains 79,302,017 images collected from the Web. The images in the dataset are stored as $32 \times 32$ color images. Since CIFAR-10 and CIFAR-100 are labeled subsets of the TinyImages dataset, we need to remove those images in the dataset that belong to CIFAR-10 or CIFAR-100. We follow the same  deduplication procedure as in \cite{hendrycks2018deep} and remove all examples in this dataset that appear in CIFAR-10 or CIFAR-100. Even after deduplication, the auxiliary OOD dataset may still contain some in-distribution data if we use CIFAR-10 or CIFAR-100 as in-distribution datasets, but the fraction of them is low. 
    \item \textbf{ImageNet-RC.} We use the downsampled ImageNet dataset (ImageNet$64\times 64$)~\cite{chrabaszcz2017downsampled}, which is a downsampled variant of the original ImageNet dataset. It contains 1,281,167 images with image size of $64 \times 64$ and 1,000 classes. Some of the classes overlap with CIFAR-10 or CIFAR-100 classes. Since we don't use any label information from the dataset, we can say that the auxiliary OOD dataset is unlabeled. Since we randomly crop the $64\times 64$ images into $32\times 32$ images with padding of 4 pixels, with high probability, the resulting images won't contain objects belonging to the in-distribution classes even if the original images contain objects belonging to those classes. Therefore, we still can have a lot of OOD data for training and the fraction of in-distribution data in the auxiliary OOD dataset is low. We call this auxiliary OOD dataset ImageNet-RC. 
\end{enumerate}

\para{Out-of-distribution Datasets.} For OOD test dataset, we follow the procedure in \cite{hendrycks2018deep,liang2018enhancing} and use six different natural image datasets. For CIFAR-10 and CIFAR-100, we use \texttt{SVHN}, \texttt{Textures}~\cite{cimpoi14describing}, \texttt{Places365}~\cite{zhou2017places}, \texttt{LSUN (crop)}, \texttt{LSUN (resize)}~\cite{yu2015lsun}, and \texttt{iSUN}~\cite{xu2015turkergaze}. For SVHN, we use \texttt{CIFAR-10}, \texttt{Textures}, \texttt{Places365}, \texttt{LSUN (crop)}, \texttt{LSUN (resize)}, and \texttt{iSUN}. 
We provide the details of OOD test datasets below. All images are of size  $32\times 32$.
\begin{enumerate}
	\item \textbf{SVHN.} The SVHN dataset \cite{netzer2011reading} contains color images of house numbers. There are ten classes of digits \texttt{0-9}. The original test set has 26,032 images. We randomly select 1,000 test images for each class and form a new test dataset of 10,000 images for  evaluation. 
	\item \textbf{Textures.} The Describable Textures Dataset (DTD) \cite{cimpoi14describing} contains textural images in the wild. We include the entire collection of 5640 images for evaluation. 
	\item \textbf{Places365.} The Places365 dataset \cite{zhou2017places} contains large-scale photographs of scenes with 365 scene categories. There are 900 images per category in the test set. We randomly sample 10,000 images from the test set for evaluation.   
	\item \textbf{LSUN (crop) and LSUN (resize).} The Large-scale Scene UNderstanding dataset (LSUN) has a testing set of 10,000 images of 10 different scenes \cite{yu2015lsun}. We construct two datasets, \texttt{LSUN-C} and \texttt{LSUN-R}, by randomly cropping image patches of size $32 \times 32$ and downsampling each image to size $32 \times 32$, respectively. 
	\item \textbf{iSUN.} The iSUN \cite{xu2015turkergaze} consists of a subset of SUN images. We include the entire collection of 8925 images in iSUN. 
	\item \textbf{CIFAR-10.} We use the test set of CIFAR-10, which contains 10,000 images. 
\end{enumerate}

\para{Hyperparameters.}  The hyperparameter $q$ is chosen on a separate validation set from TinyImages, which is different from test-time OOD data. Based on the validation, we set $q=0$ for SVHN, $q=0.125$ for CIFAR-10, and $q=0.5$ for CIFAR-100. To ensure fair comparison, in each epoch, ATOM uses the same amount of outlier  data as OE, where $n$ is twice larger than the in-distribution data size (i.e., 50,000). For all experiments, we set $\lambda=1$. For CIFAR-10 and CIFAR-100, we set $N=400,000$, and $n=100,000$; For SVHN, we set $N=586,056$, and $n=146,514$. 

\para{Architectures and Training Configurations.}  We use the state-of-the-art neural network architecture DenseNet~\cite{huang2017densely} and WideResNet~\cite{zagoruyko2016wide}. For DenseNet, we follow the same setup as in \cite{huang2017densely}, with depth $L=100$, growth rate $k=12$ (Dense-BC) and dropout rate $0$. For WideResNet, we also follow the same setup as in \cite{zagoruyko2016wide}, with depth of 40 and widening parameter $k=4$ (WRN-40-4). All neural networks are trained with stochastic gradient descent with Nesterov momentum \cite{duchi2011adaptive,kingma2014adam}. We set momentum $0.9$ and $\ell_2$ weight decay with a coefficient of $10^{-4}$ for all model training. Specifically, for SVHN, we train the networks for 20 epochs and the initial learning rate of $0.1$ decays by $0.1$ at 10, 15, 18 epoch; for CIFAR-10 and CIFAR-100, we train the networks for 100 epochs and the initial learning rate of $0.1$ decays by $0.1$ at 50, 75, 90 epoch. In ATOM and NTOM, we use batch size 64 for in-distribution data and 128 for out-of-distribution data. To solve the inner max of the robust training objective in ATOM, we use PGD with $\epsilon=8/255$, the number of iterations of 5, the step size of $2/255$, and random start.

\subsection{Average Runtime}
\label{sec:average-runtime}
We run our experiments using a single GPU on a machine with 4 GPUs and 32 cores. The estimated average runtime for each method is summarized in Table~\ref{tab:average-runtime}.

\begin{table}[t]
        \center
        \caption[]{\small  The estimated average runtime for each result. We use DenseNet as network architecture. $h$ means hour. For MSP, ODIN, and Mahalanobis, we use standard training. The evaluation includes four OOD detection tasks listed in Section~\ref{sec:problem-statement}. }
		\begin{tabular}{l|cc}
			\toprule
			\bf{Method}   & {\bf Training} & {\bf Evaluation} \\ \hline
			MSP  & 2.5 h &  4 h \\
			ODIN & 2.5 h & 4 h  \\
			Mahalanobis  & 2.5 h & 20 h \\
			SOFL  & 14 h & 4 h \\
			OE  & 5 h & 4 h \\ 
			ACET  & 17 h & 4 h \\ 
			CCU  & 6.7 h & 4 h \\ 
			ROWL  & 24 h & 4 h \\ 	
			ATOM (ours) & 21 h & 4 h \\  
			\bottomrule
		\end{tabular}
	\label{tab:average-runtime}
\end{table}

\subsection{OOD Detection Methods}
\label{subsec:ood-techs}
We consider eight common OOD detection methods and describe each method in detail below. 

\textbf{Maximum Softmax Probability (MSP).} ~\cite{hendrycks2016baseline} propose to use $\max_i F_i(x)$ as confidence scores to detect OOD examples, where $F(x)$ is the softmax output of the neural network. 

\textbf{ODIN.}  ~\cite{liang2018enhancing} computes calibrated confidence scores using temperature scaling and input perturbation techniques. In all of our experiments, we set temperature scaling parameter $T=1000$. We choose perturbation magnitude $\eta$ by validating on 1000 images randomly sampled from in-distribution test set $\mathcal{D}_{\text{in}}^{\text{test}}$ and 1000 images randomly sampled from auxiliary OOD dataset $\mathcal{D}_{\text{out}}^{\text{auxiliary}}$, which does not depend on prior knowledge of test OOD datasets. For DenseNet, we set $\eta=0.0006$ for SVHN, $\eta=0.0016$ for CIFAR-10, and $\eta=0.0012$ for CIFAR-100. For WideResNet, we set $\eta=0.0002$ for SVHN, $\eta=0.0006$ for CIFAR-10, and $\eta=0.0012$ for CIFAR-100. 

\textbf{Mahalanobis.}  ~\cite{lee2018simple} propose to use Mahalanobis distance-based confidence scores to detect OOD samples. Following \cite{lee2018simple}, we use 1000 examples randomly selected from in-distribution test set $\mathcal{D}_{\text{in}}^{\text{test}}$ and adversarial examples generated by FGSM \cite{goodfellow2014explaining} on them with perturbation size of $0.05$ to train the Logistic Regression model and tune the noise perturbation magnitude $\eta$. $\eta$ is chosen from $\{0.0, 0.01, 0.005, 0.002, 0.0014, 0.001, 0.0005\}$, and the optimal parameters are chosen to minimize the FPR at FNR 5\%.

\textbf{Outlier Exposure (OE).}  Outlier Exposure \cite{hendrycks2018deep} makes use of a large, auxiliary OOD dataset $\mathcal{D}_{\text{out}}^{\text{auxiliary}}$ to enhance the performance of existing OOD detection. We train from scratch with $\lambda=0.5$, and use in-distribution batch size of 64 and out-distribution batch size of 128 in our experiments. Other training parameters are specified in Section~\ref{sec:more-experiment-settings}.

\textbf{Self-Supervised OOD Feature Learning (SOFL).} ~\cite{mohseni2020self} add an auxiliary head to the network and train in for the OOD detection task. They first use a full-supervised training to learn in-distribution training data for the main classification head and then a self-supervised training with OOD training set for the auxiliary head. Following the original setting, we set $\lambda=5$ and use an in-distribution batch size of 64 and an out-distribution batch size of 320 in all of our experiments. In SVHN and CIFAR-10, we use 5 reject classes, while in CIFAR-100, we use 10 reject classes. We first train the model with the full-supervised learning using the training parameters specified in Section~\ref{sec:more-experiment-settings} and then continue to train with the self-supervised OOD feature learning using the same training parameters. We use the large, auxiliary OOD dataset $\mathcal{D}_{\text{out}}^{\text{auxiliary}}$ as out-of-distribution training dataset.

\textbf{Adversarial Confidence Enhancing Training (ACET).} ~\cite{hein2019relu} propose Adversarial Confidence Enhancing Training to enforce low model confidence for the OOD data point, as well as worst-case adversarial example in the neighborhood of an OOD example. We use the large, auxiliary OOD dataset $\mathcal{D}_{\text{out}}^{\text{auxiliary}}$ as an OOD training dataset instead of using random noise data for a fair comparison. In all of our experiments, we set $\lambda=1.0$. For both in-distribution and out-distribution, we use a batch size of 128. To solve the inner max of the training objective, we also apply PGD with $\epsilon=8/255$, the number of iterations of 5, the step size of $2/255$, and random start to a half of a minibatch while keeping the other half clean to ensure proper performance on both perturbed and clean OOD examples for a  fair comparison. Other training parameters are specified in Section~\ref{sec:more-experiment-settings}.

\textbf{Certified Certain Uncertainty (CCU). } Certified Certain Uncertainty \cite{meinke2019towards} gives guarantees on the confidence of the classifier decision far away from the training data. We use the same training set up as in the paper and code, except we use our training configurations specified in Section~\ref{sec:more-experiment-settings}. 

\textbf{Robust Open-World Deep Learning (ROWL). } ~\cite{sehwag2019analyzing} propose to introduce additional background classes for OOD datasets and perform adversarial training on both the in- and out-of- distribution datasets to achieve robust open-world classification. When an input is classified as the background classes, it is considered as an OOD example. Thus, ROWL gives binary OOD scores (either 0 or 1) to the inputs. In our experiments, we only have one background class and randomly sample data points from the large, auxiliary OOD dataset $\mathcal{D}_{\text{out}}^{\text{auxiliary}}$ to form the OOD dataset. To ensure data balance across classes, we include 7,325 OOD data points for SVHN, 5,000 OOD data points for CIFAR-10 and 500 OOD data points for CIFAR-100. During training, we mix the in-distribution data and OOD data, and use a batch size of 128. To solve the inner max of the training objective, we use PGD with $\epsilon=8/255$, the number of iterations of 5, the step size of $2/255$, and random start. Other training parameters are specified in Section~\ref{sec:more-experiment-settings}.

\subsection{Adversarial Attacks for OOD Detection Methods}
\label{sec:attack-algorithms}

We propose adversarial attack objectives for different OOD detection methods. We consider a family of adversarial perturbations for the OOD inputs: (1) $L_\infty$-norm bounded attack (white-box attack); (2) common image corruptions attack (black-box attack); (3) compositional attack which combines common image corruptions attack and $L_\infty$ norm bounded attack (white-box attack). 

\para{$L_\infty$ norm bounded attack.} For data point $\*x \in \mathbb{R}^{d}$, the $L_\infty$ norm bounded perturbation is defined as 
\begin{align}
\Omega_{\infty, \epsilon}(\*x) = \{\delta \in \mathbb{R}^{d}  \bigm| \| \delta \|_\infty \leq \epsilon \land \*x+\delta \text{ is valid} \},
\end{align} 
where $\epsilon$ is the adversarial budget. $\*x+\delta$ is considered valid if the values of $\*x+\delta$ are in the image pixel value range. 

For MSP, ODIN, OE, ACET, and CCU methods, we propose the following attack objective to generate adversarial OOD example on a clean OOD input $\*x$: 
\begin{align}
    \*x' = \argmax_{\*x' \in \Omega_{\infty, \epsilon}(\*x)} -\frac{1}{K} \sum_{i=1}^{K} \log F(\*x')_i
\end{align}
where $F(\*x)$ is the softmax output of the classifier network. 

For Mahalanobis method, we propose the following attack objective to generate adverasrial OOD example on OOD input $\*x$:
\begin{align}
     \*x' = \argmax_{\*x' \in \Omega_{\infty, \epsilon}(\*x)} -\log \frac{1}{1+e^{-(\sum_\ell \alpha_\ell M_\ell (\*x')+b)}},
\end{align}
where $M_\ell (\*x')$ is the Mahalanobis distance-based confidence score of $\*x'$ from the $\ell$-th feature layer, $\{\alpha_\ell\}$ and $b$ are the parameters of the logistic regression model.

For SOFL method, we propose the following attack objective to generate adversarial OOD example for an input $\*x$:
\begin{align}
    \*x' = \argmax_{\*x' \in \Omega_{\infty, \epsilon}(\*x)} - \log \sum_{i=K+1}^{K+R} \bar{F}(\*x')_i
\end{align}
where $\bar{F}(\*x)$ is the softmax output of the whole neural network (including auxiliary head) and $R$ is the number of reject classes. 

For ROWL and ATOM method, we propose the following attack objective to generate adverasrial OOD example on OOD input $\*x$:
\begin{align}
    \*x' = \argmax_{\*x' \in \Omega_{\infty, \epsilon}(\*x)} - \log \hat{F}(\*x')_{K+1}
\end{align}
where $\hat{F}(\*x)$ is the softmax output of the (K+1)-way neural network.

Due to computational concerns, by default, we will use PGD with $\epsilon=8/255$, the number of iterations of 40, the step size of 1/255 and random start to solve these attack objectives. We also perform ablation study experiments on the attack strength for ACET and ATOM, see Appendix \ref{sec:ablation-attack-strength}.

\para{Common Image Corruptions attack.} We use common image corruptions introduced in \cite{hendrycks2019benchmarking}. We apply 15 types of algorithmically generated corruptions from noise, blur, weather, and digital categories to each OOD image. Each type of corruption has five levels of severity, resulting in 75 distinct corruptions. Thus, for each OOD image, we generate 75 corrupted images and then select the one with the lowest OOD score (or highest confidence score to be in-distribution). Note that we only need the outputs of the OOD detectors to construct such adversarial OOD examples; thus it is a black-box attack. 

\para{Compositional Attack.} For each OOD image, we first apply common image corruptions attack, and then apply the $L_\infty$-norm bounded attack to generate adversarial OOD examples. 

\subsection{Visualizations of Four Types of OOD Samples}
\label{sec:example-of-four-types-of-OOD}
We show visualizations of four types of OOD samples in Figure~\ref{fig:adversarial-images}.

\begin{figure*}[t!]
    \centering
    \begin{subfigure}{0.49\linewidth}
	    \centering
		\includegraphics[width=\linewidth]{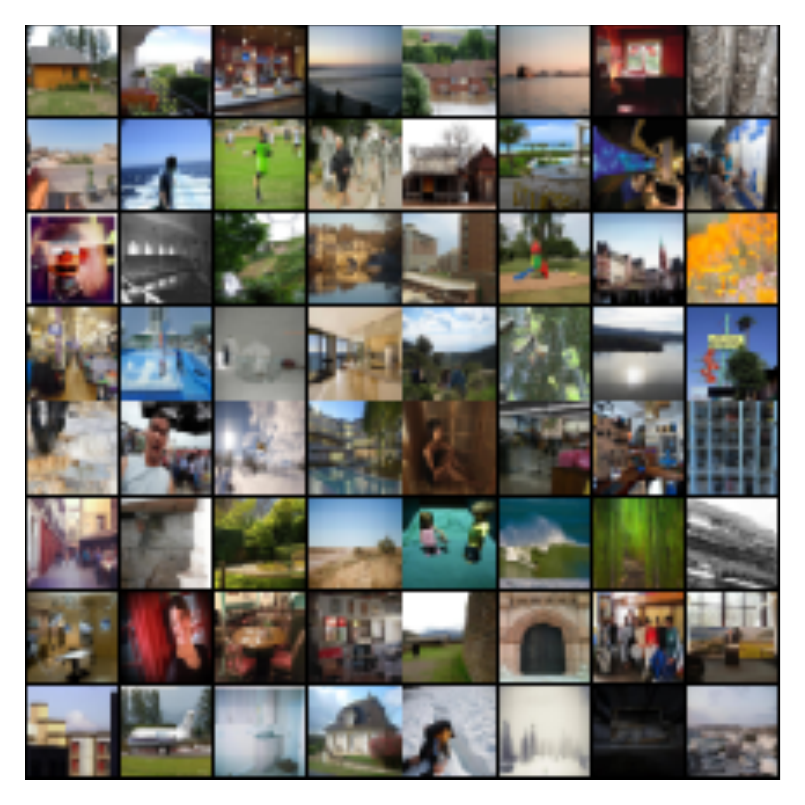}
	\caption{Natural OOD}
	\end{subfigure}
	\begin{subfigure}{0.49\linewidth}
	    \centering
		\includegraphics[width=\linewidth]{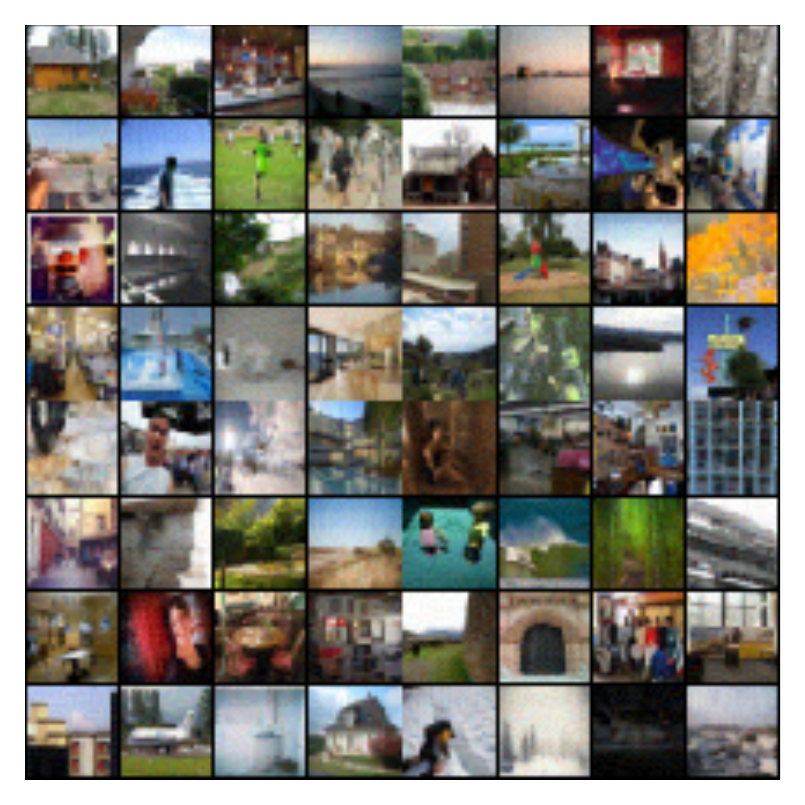}
	\caption{$L_\infty$ OOD}
	\end{subfigure}
    \begin{subfigure}{0.49\linewidth}
	    \centering
		\includegraphics[width=\linewidth]{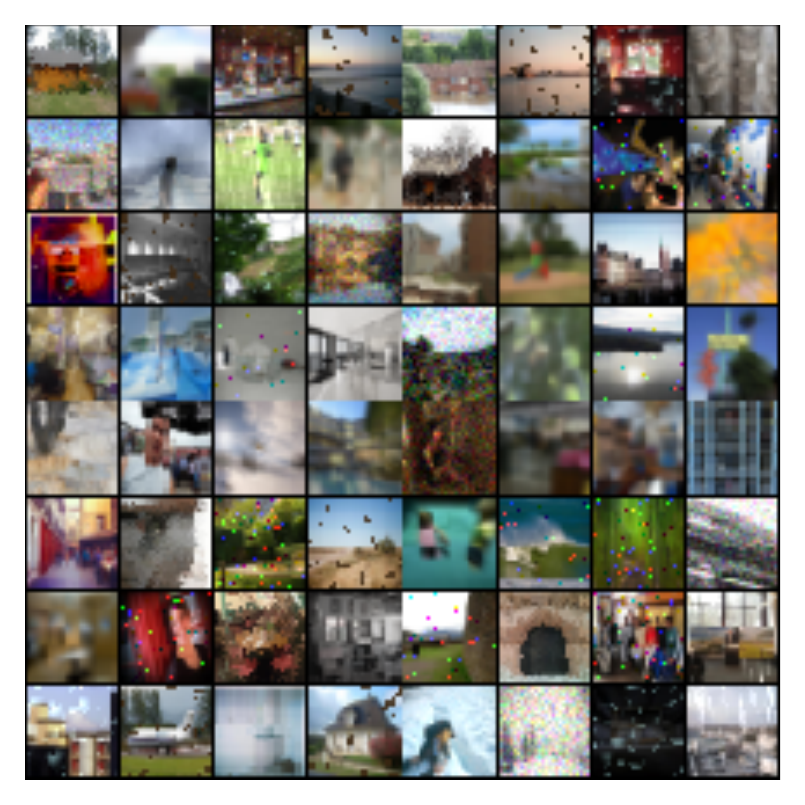}
	\caption{Corruption OOD}
	\end{subfigure}
	\begin{subfigure}{0.49\linewidth}
	    \centering
		\includegraphics[width=\linewidth]{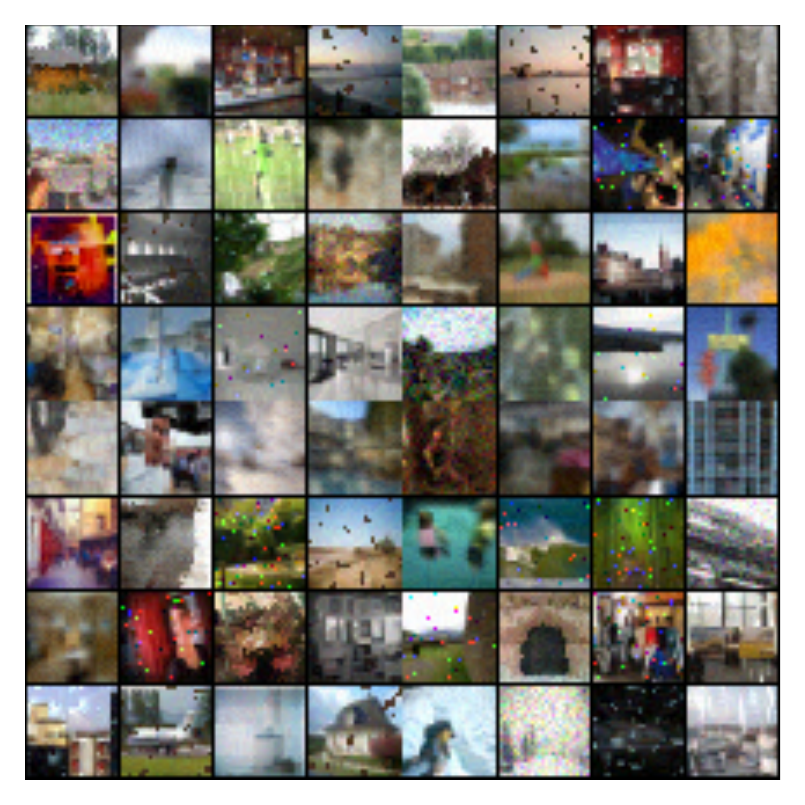}
	\caption{Comp. OOD}
	\end{subfigure}
	\hspace{1cm}
	\caption{\small Examples of four types of OOD samples. 
	}
	\label{fig:adversarial-images}
\end{figure*}

\subsection{Visualization of Easiest and Hardest OOD Examples}
\label{sec:easiest-and-hardest-ood-example}
We show visualizations of OOD examples with highest OOD scores (easiest) and lowest OOD scores (hardest) in Figure~\ref{fig:ood-score-distribution-during-training}. 

\begin{figure*}[t]
    \centering
    \begin{subfigure}{0.3\linewidth}
	    \centering
		\includegraphics[width=\linewidth]{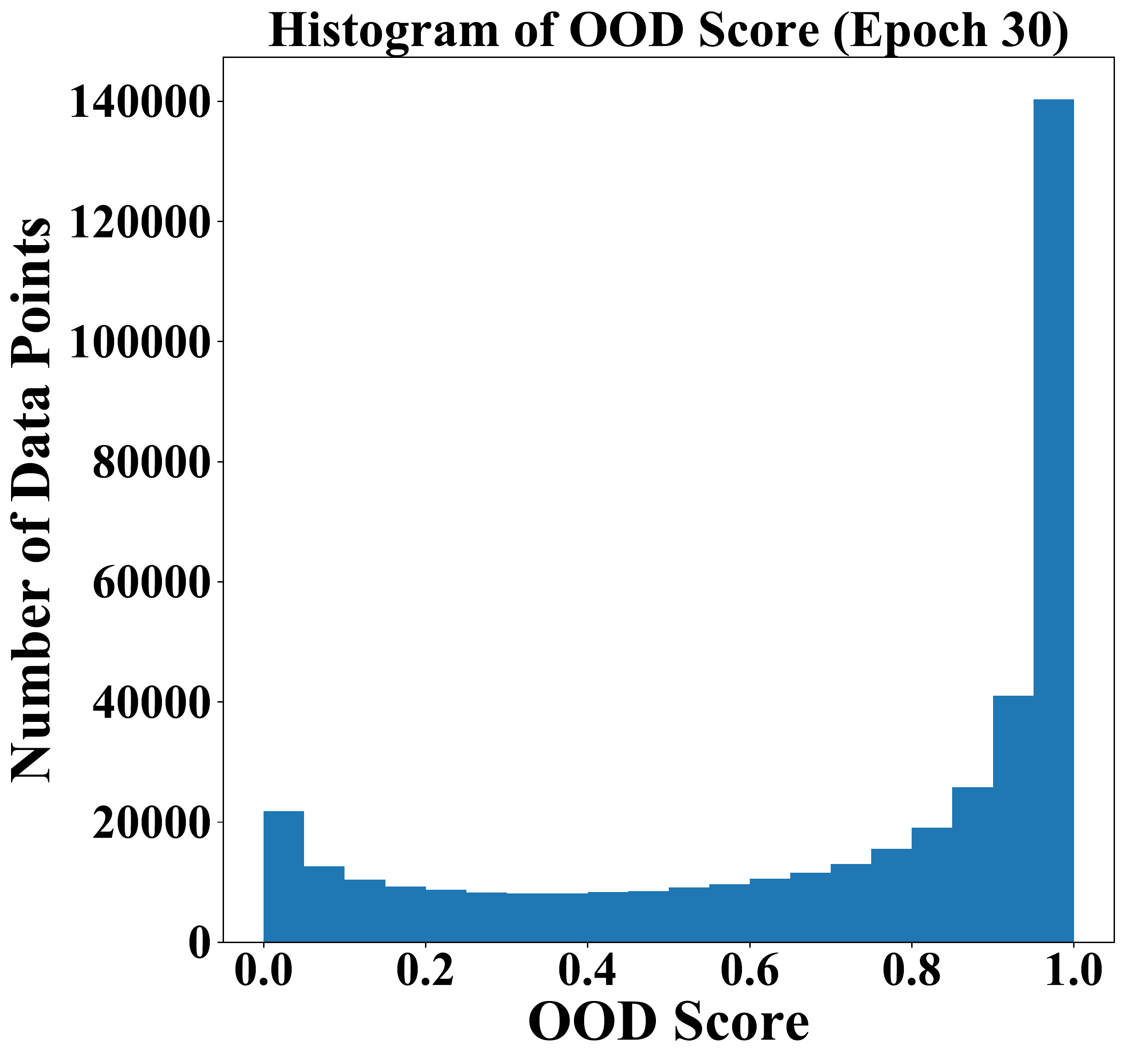}
	\caption{OOD score distribution}
	\end{subfigure}
	\begin{subfigure}{0.3\linewidth}
		\centering
		\includegraphics[width=\linewidth]{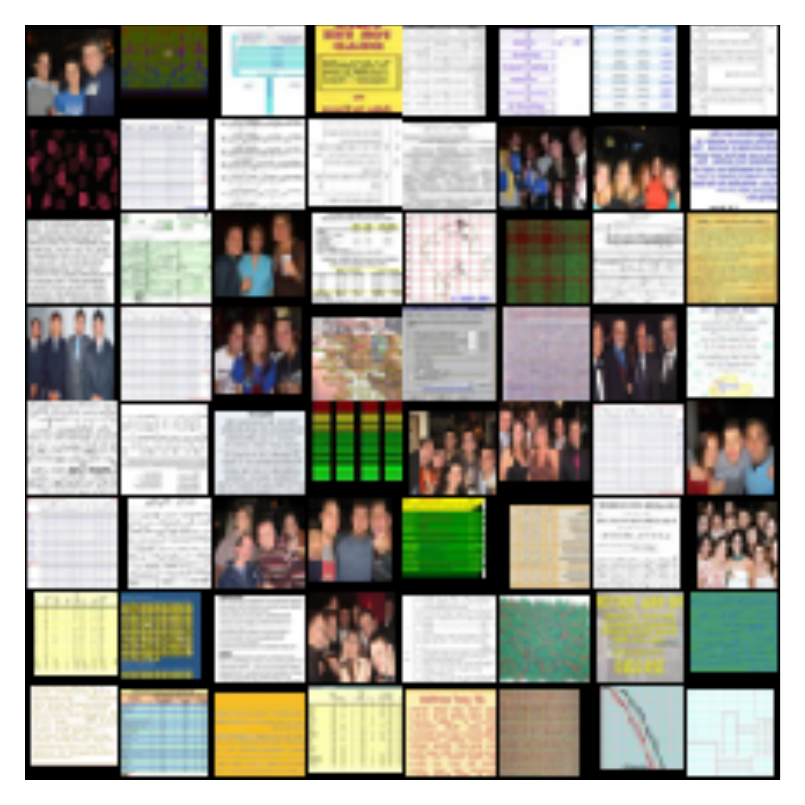}
		\caption{highest OOD scores}
	\end{subfigure} 
	\begin{subfigure}{0.3\linewidth}
		\centering
		\includegraphics[width=\linewidth]{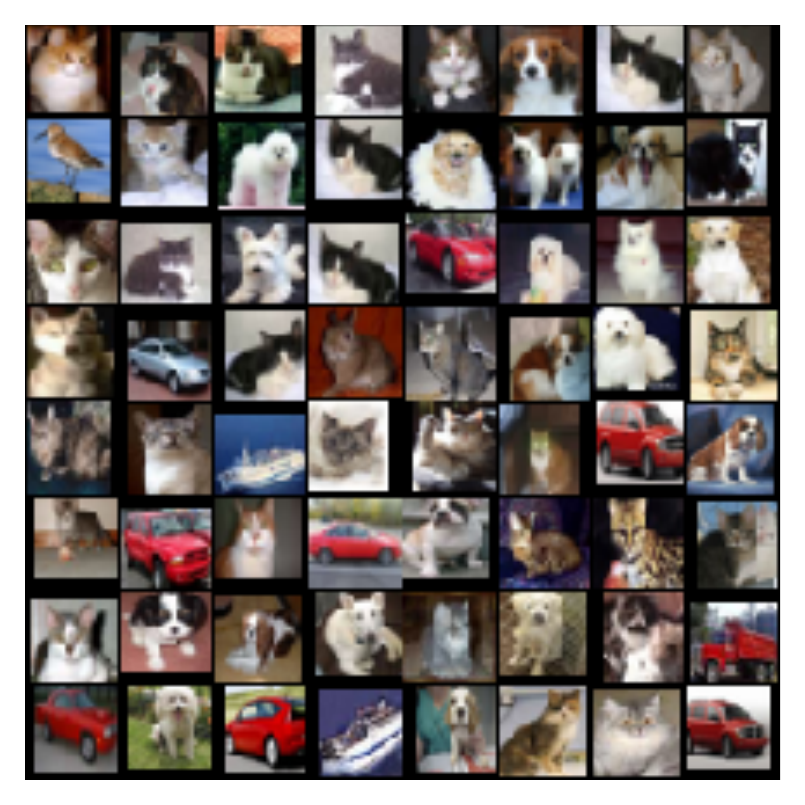}
		\caption{lowest OOD scores}
	\end{subfigure} 
	\vspace{-0.2cm}
	\caption{\small On CIFAR-10, we train a DenseNet with objective (\ref{obj:adv}) for 100 epochs {\bf without} informative outlier mining. At epoch 30, we randomly sample 400,000 data points from $\mathcal{D}_{\text{out}}^{\text{auxiliary}}$, and plot the OOD score frequency distribution (a). We observe that the model quickly converges to solution where OOD score distribution becomes dominated by {\em easy} examples with score closer to 1, as shown in (b). Therefore, training on these easy OOD data points can no longer help improve the decision boundary of OOD detector.  (c) shows the hardest examples mined from TinyImages w.r.t CIFAR-10. 
	}
	\label{fig:ood-score-distribution-during-training}
	\vspace{-0.3cm}
\end{figure*}

\subsection{Histogram of OOD Scores}
In Figure~\ref{fig:ood-score-distribution-during-training-full-plots}, we show histogram of OOD scores for model snapshots trained on CIFAR-10 (in-distribution) using objective (\ref{obj:adv}) \textbf{without} informative outlier mining. We plot every ten epochs for a model trained for a total of 100 epochs. We observe that the model quickly converges to a solution where OOD score distribution becomes dominated by {\em easy} examples with scores closer to 1. This is exacerbated as the model is trained for longer. 

\begin{figure*}[t!]
    \centering

	\includegraphics[width=0.45\linewidth]{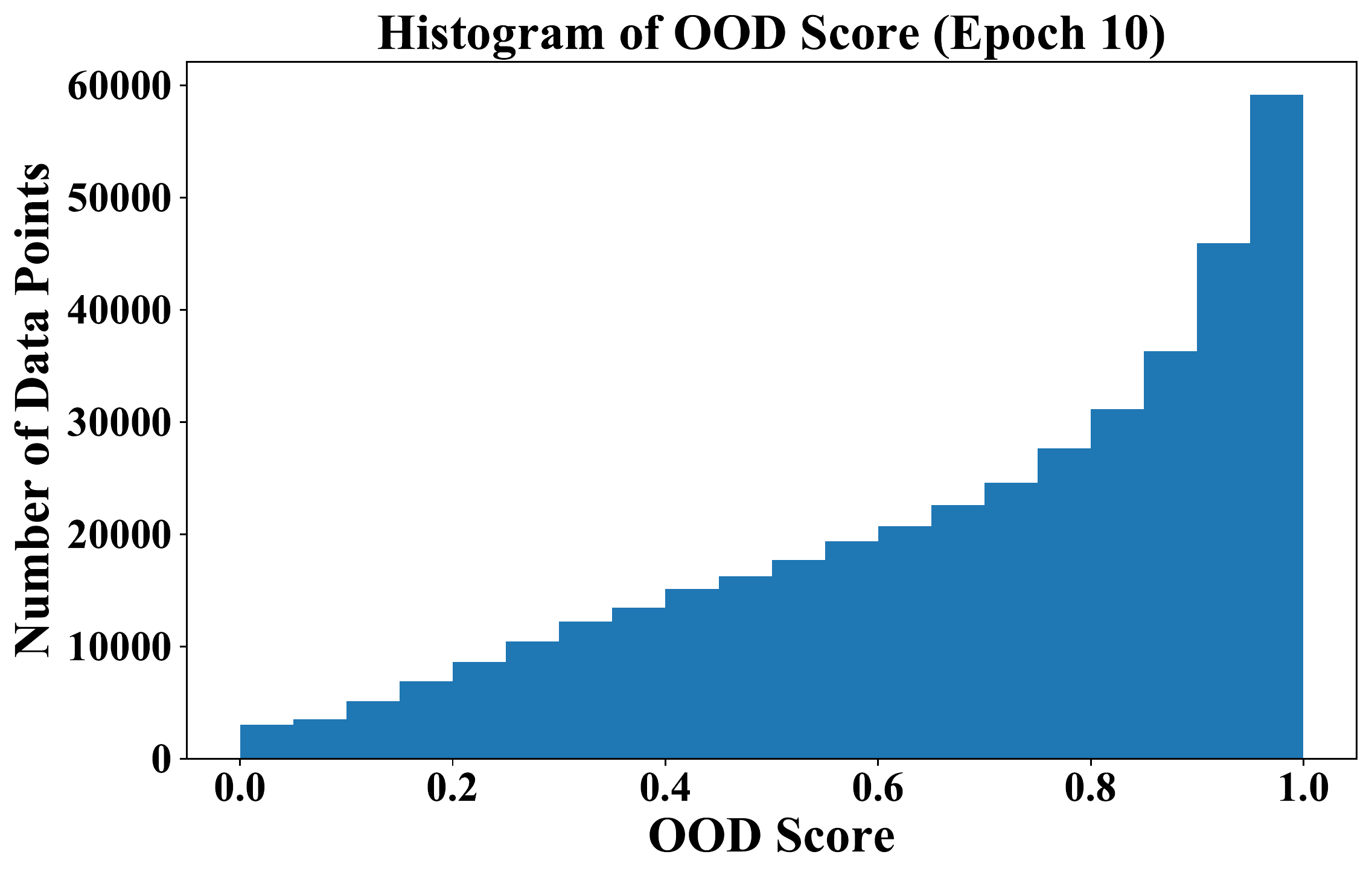}
    \includegraphics[width=0.45\linewidth]{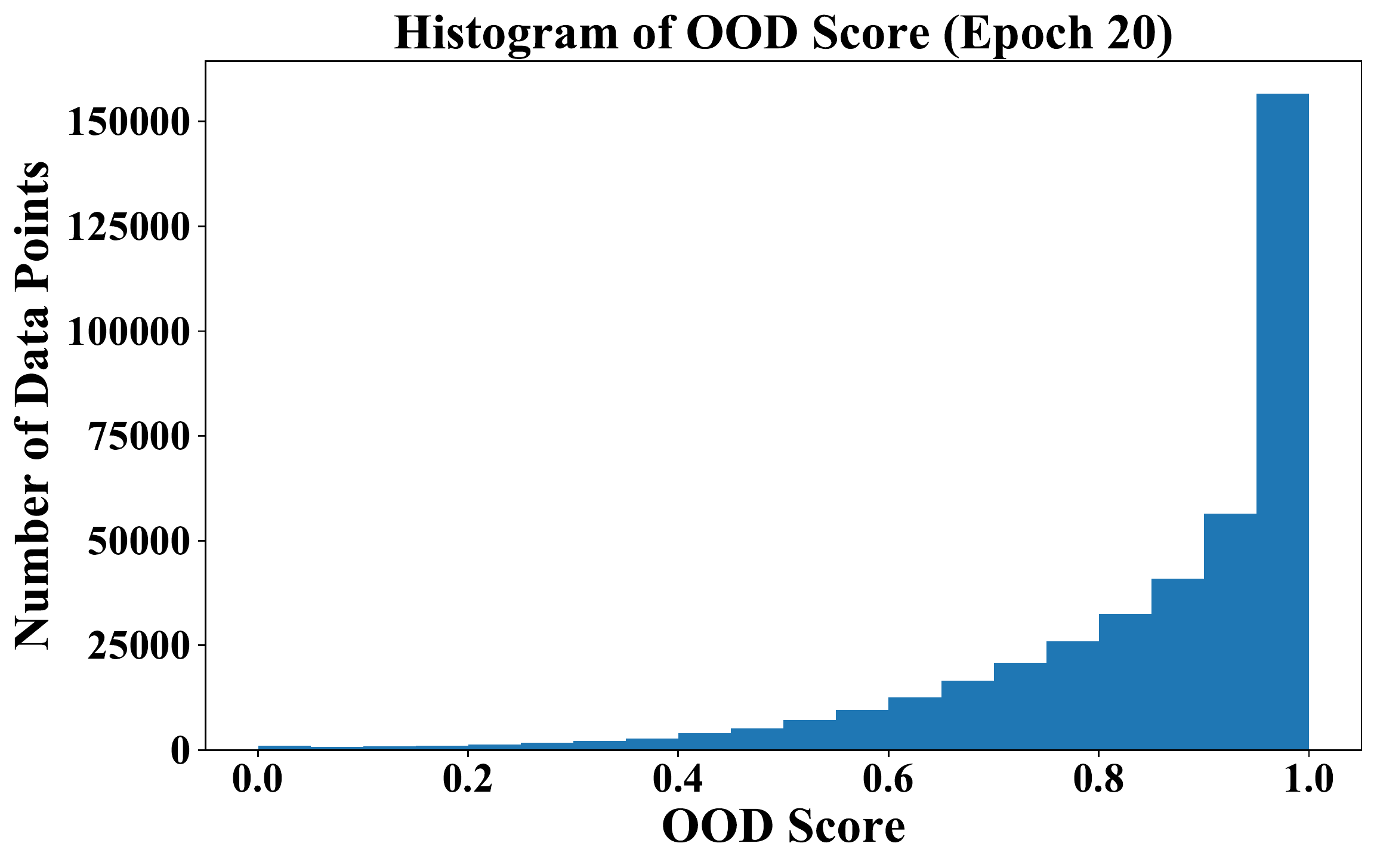}
    \includegraphics[width=0.45\linewidth]{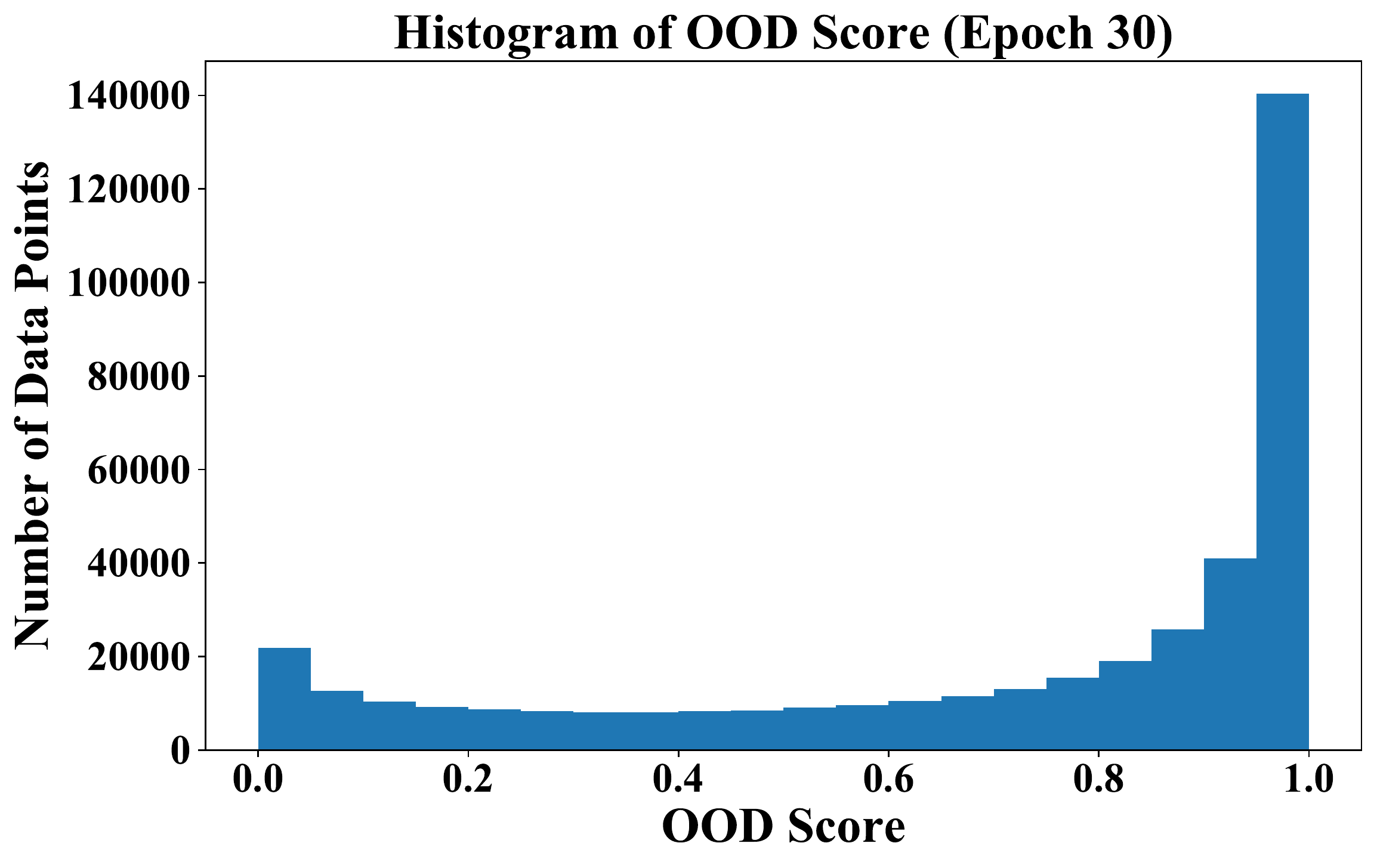}
    \includegraphics[width=0.45\linewidth]{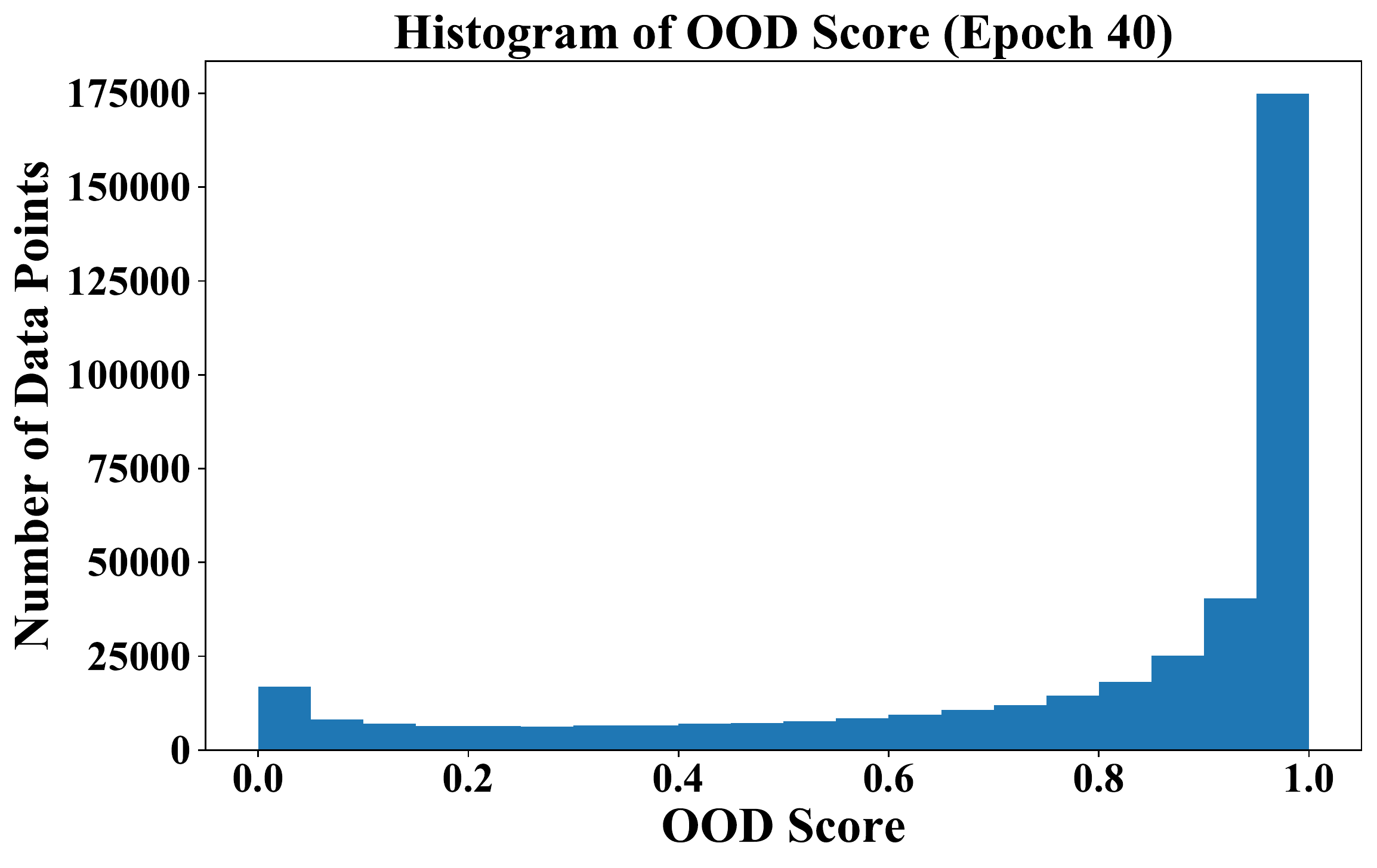}
    \includegraphics[width=0.45\linewidth]{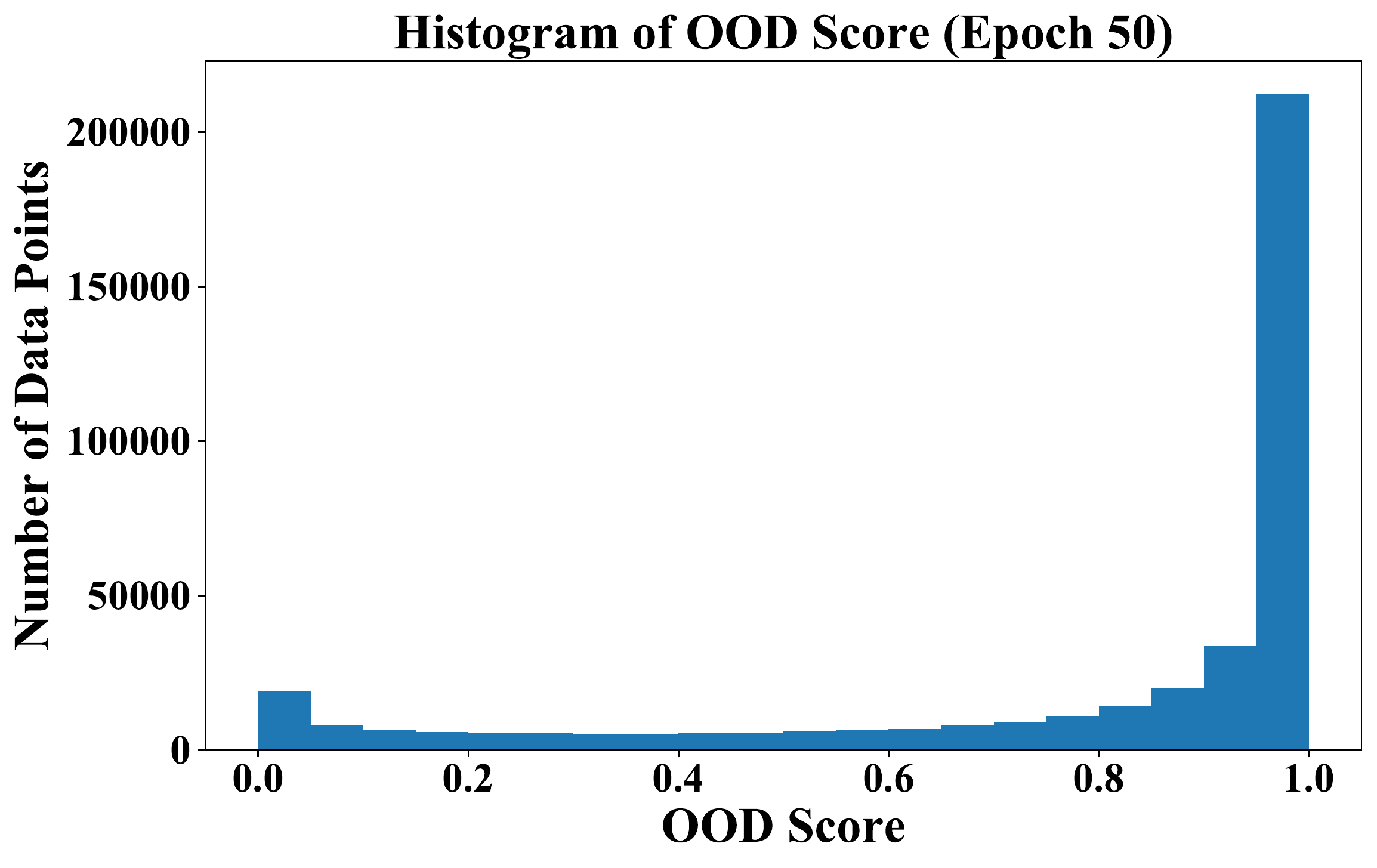}
    \includegraphics[width=0.45\linewidth]{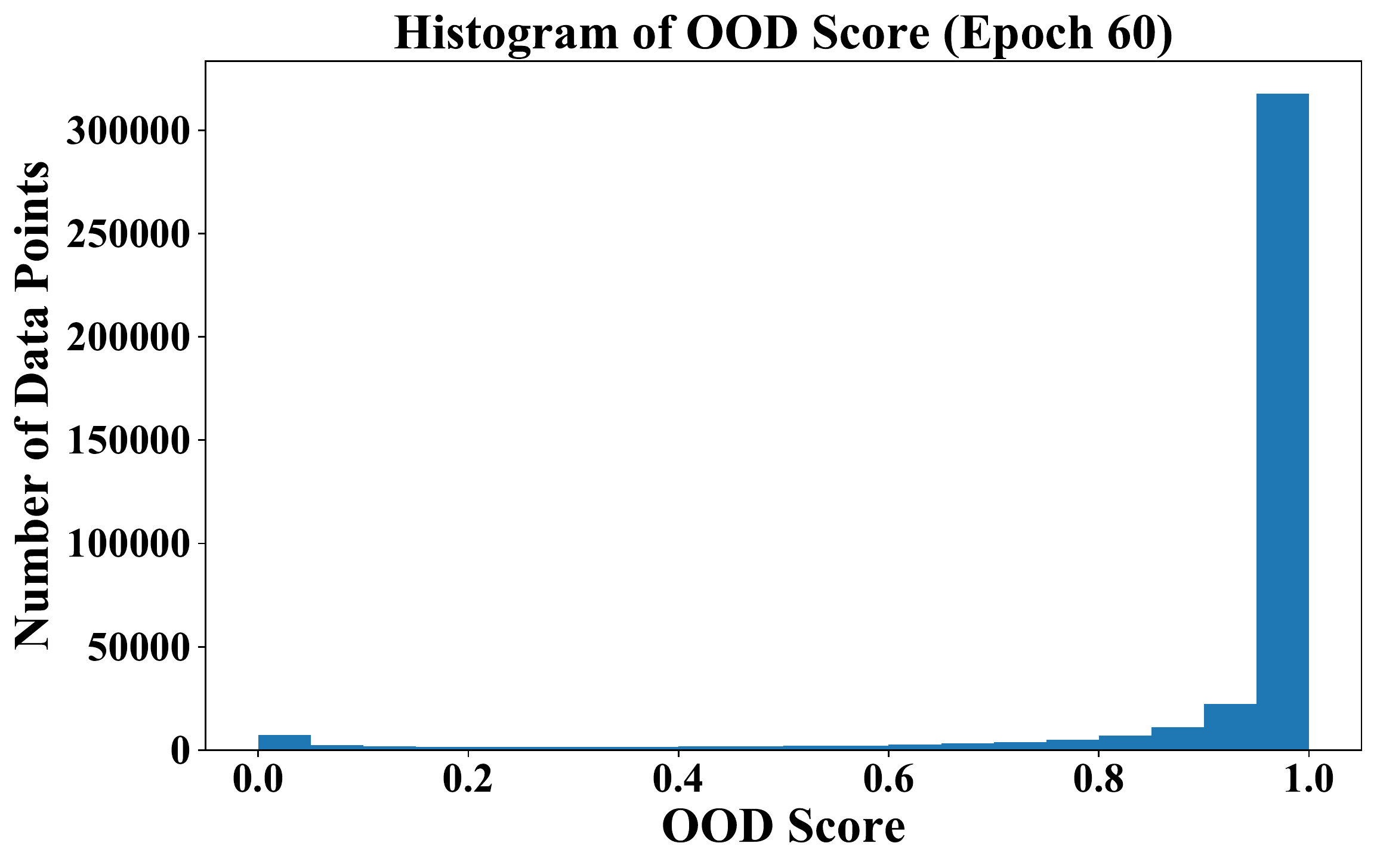}
    \includegraphics[width=0.45\linewidth]{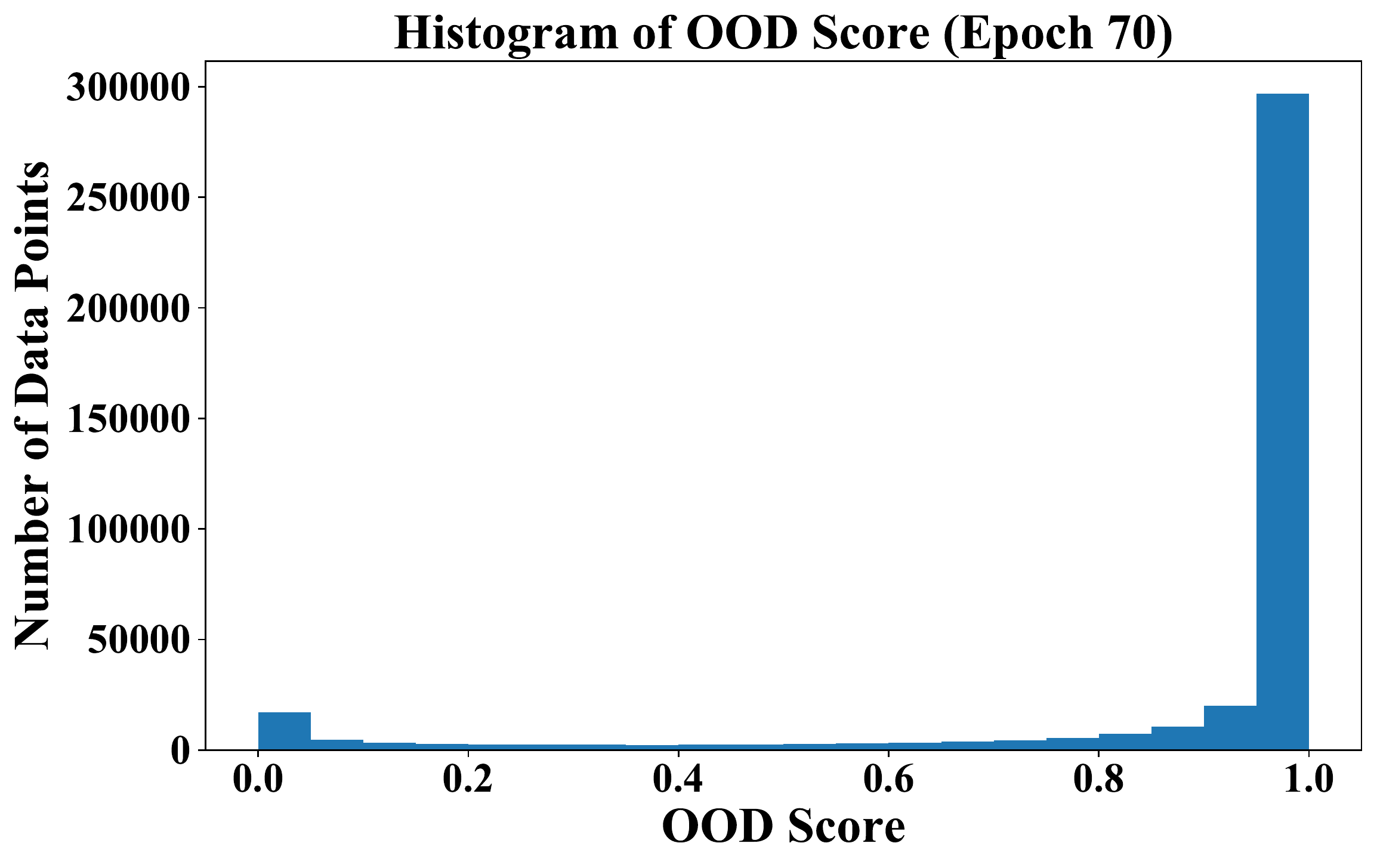}
    \includegraphics[width=0.45\linewidth]{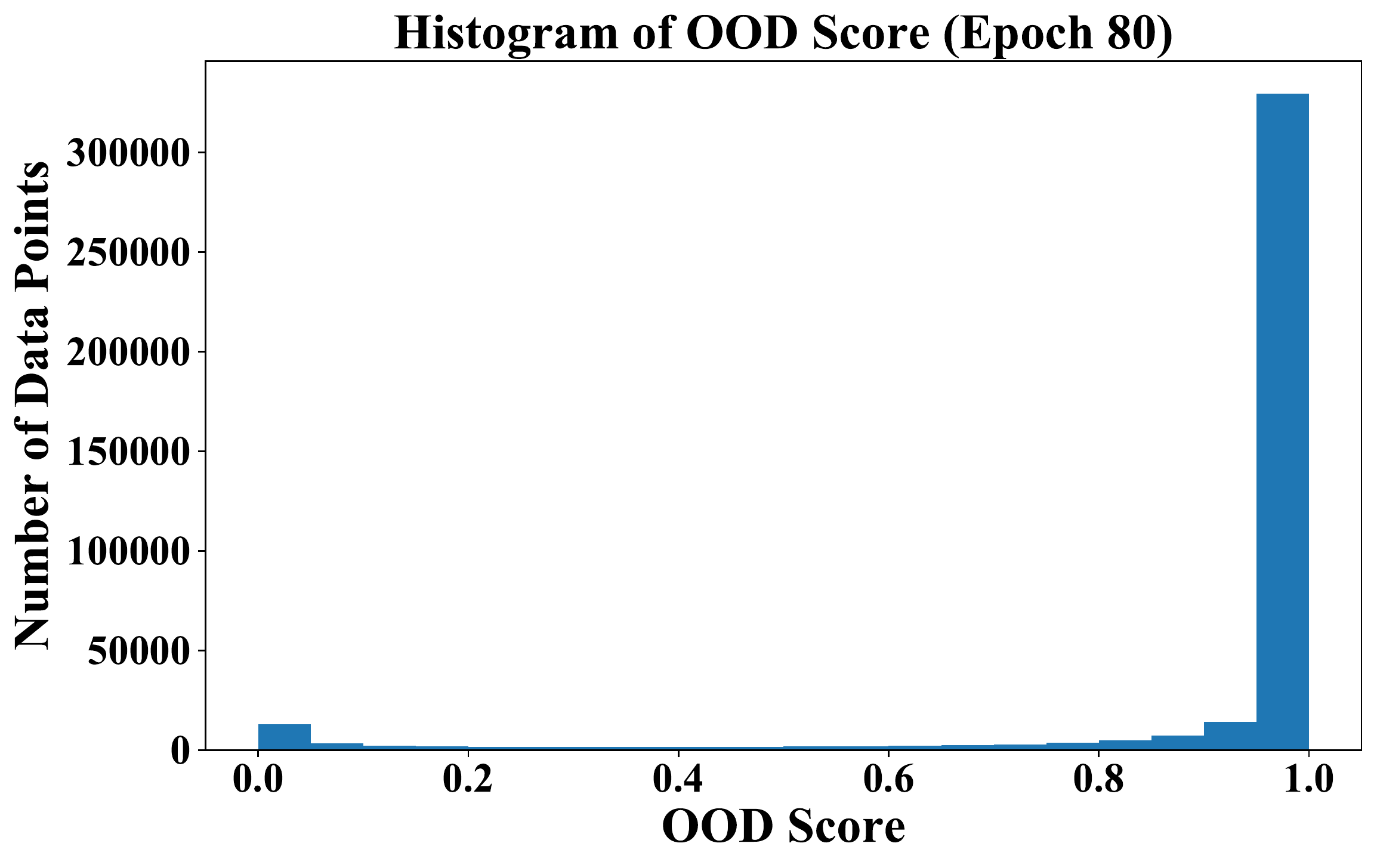}
    \includegraphics[width=0.45\linewidth]{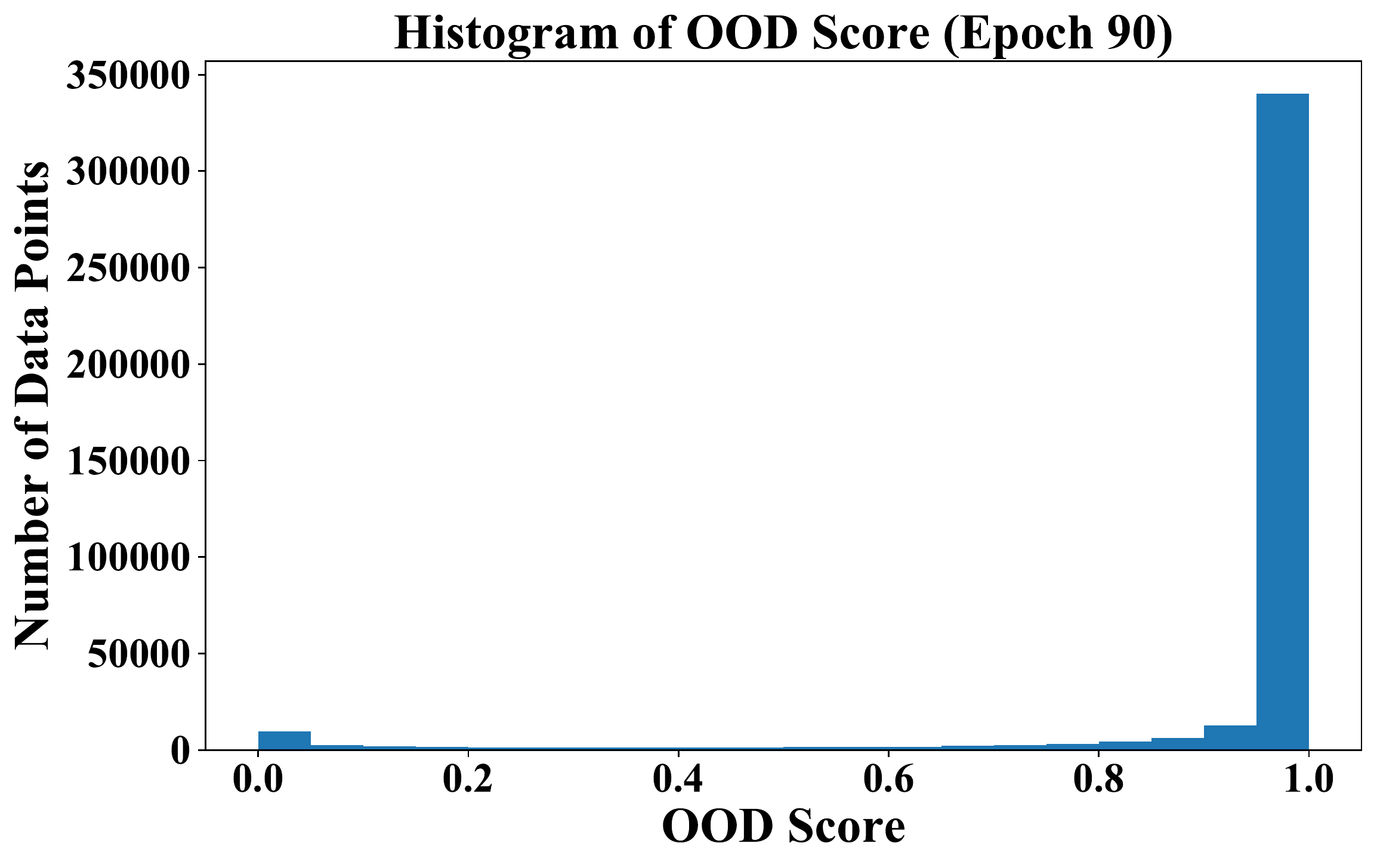}
    \includegraphics[width=0.45\linewidth]{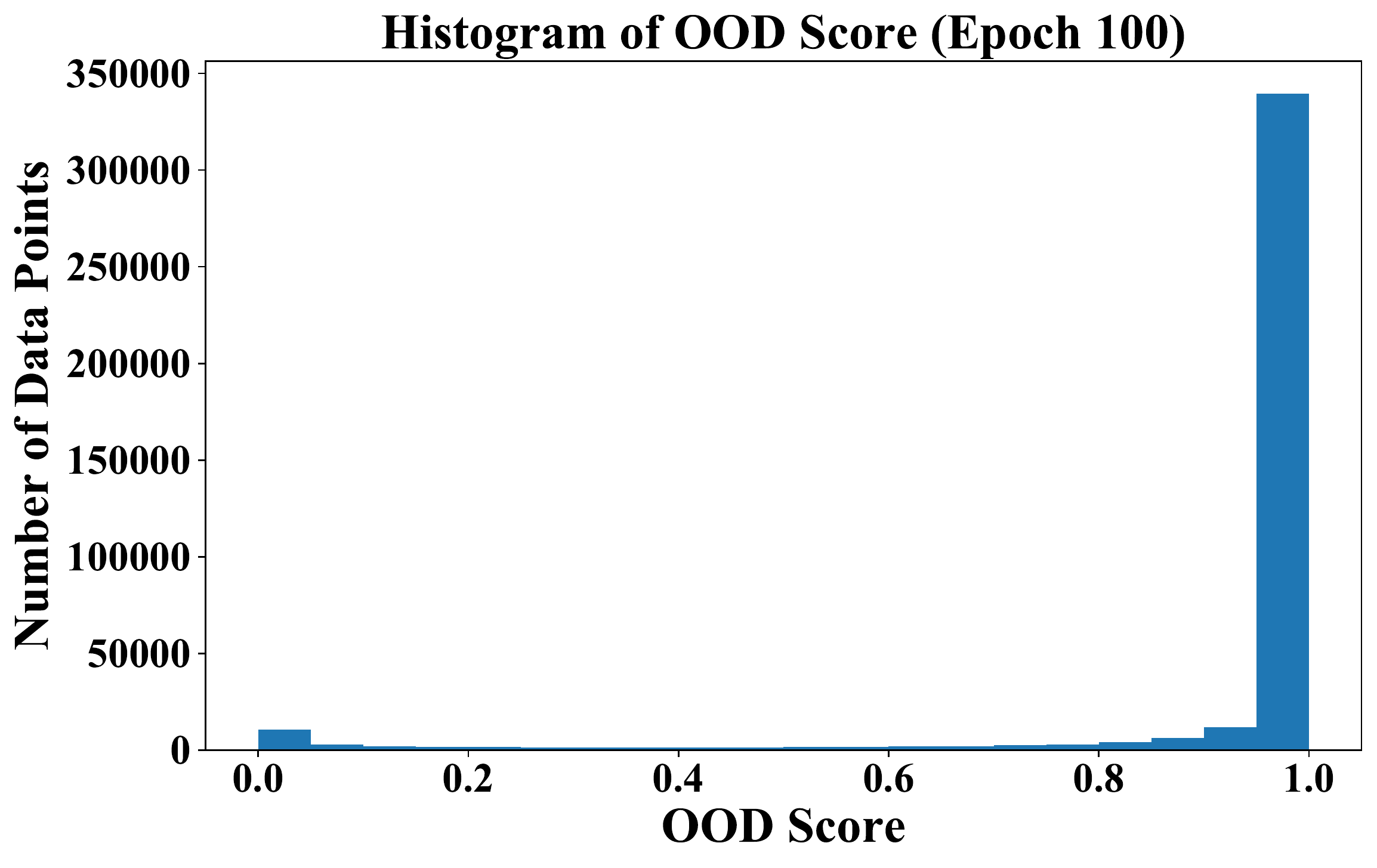}
    
	\hspace{1cm}
	\caption{\small On CIFAR-10, we train the model with objective (\ref{obj:adv}) for 100 epochs \textbf{without} informative outlier mining. For every 10 epochs, we randomly sample 400,000 data points from the large auxiliary OOD dataset and use the current model snapshot to calculate the OOD scores. }
	\label{fig:ood-score-distribution-during-training-full-plots}
\end{figure*}

\subsection{Experimental Results on SVHN as In-distribution Dataset}
\label{sec:svhn-results}
The experimental results on SVHN as in-distribution dataset are shown in Table~\ref{tab:main-results-svhn}.

\begin{table*}[t!]
    \caption[]{\small Comparison with competitive OOD detection methods. We use DenseNet as network architecture for all methods. We evaluate on four types of OOD inputs: (1) natural OOD, (2) corruption attacked OOD, (3) $L_\infty$ attacked OOD, and (4) compositionally attacked OOD inputs. $\uparrow$ indicates larger value is better, and $\downarrow$ indicates lower value is better. All values are percentages and are averaged over six natural OOD test datasets described in Appendix~\ref{sec:more-experiment-settings}. \textbf{Bold} numbers are superior results. }
	\begin{adjustbox}{width=\columnwidth,center}
		\begin{tabular}{l|l|cc|cc|cc|cc}
			\toprule
			\multirow{4}{0.1\linewidth}{$\mathcal{D}_{\text{in}}^{\text{test}}$} &  \multirow{4}{0.06\linewidth}{\textbf{Method}}  &\bf{FPR}  & {\bf AUROC}  & {\bf FPR} & {\bf AUROC}  & {\bf FPR} & {\bf AUROC} & {\bf FPR} & {\bf AUROC}  \\
			& & $\textbf{(5\% FNR)}$ & $\textbf{}$ & {\bf (5\% FNR)} & $\textbf{}$ & $\textbf{(5\% FNR)}$ & {\bf } & $\textbf{(5\% FNR)}$ & {\bf } \\
			& & $\downarrow$ & $\uparrow$ & $\downarrow$ & $\uparrow$ & $\downarrow$ & $\uparrow$ & $\downarrow$ & $\uparrow$ \\ \cline{3-10}
			& & \multicolumn{2}{c|}{\textbf{Natural OOD}} & \multicolumn{2}{c|}{\textbf{Corruption OOD}} & \multicolumn{2}{c|}{\textbf{$L_\infty$ OOD }} & \multicolumn{2}{c}{\textbf{Comp. OOD}} \\  \hline 
			\multirow{10}{0.1\linewidth}{{{\bf SVHN}}}  
			& MSP~\cite{hendrycks2016baseline}  & 38.84 & 93.57 & 99.68 & 68.48 & 99.89 & 1.39 & 100.00 & 0.19 \\
			& ODIN~\cite{liang2018enhancing}  & 31.45 & 93.52 & 97.11 & 63.21 & 99.86 & 0.61 & 100.00 & 0.05 \\
			& Mahalanobis~\cite{lee2018simple} & 22.80 & 95.57 & 93.14 & 60.78 & 97.33 & 8.89 & 99.89 & 0.23 \\
			& SOFL~\cite{mohseni2020self}  & {\bf 0.06} & {\bf 99.98} & {\bf 3.78} & {\bf 99.07} & 75.31 & 46.78 & 99.81 & 2.75 \\
			& OE~\cite{hendrycks2018deep} & 0.60 & 99.88 & 23.44 & 96.23 & 69.36 & 52.19 & 99.65 & 1.27 \\ 
			& ACET~\cite{hein2019relu} & 0.49 & 99.91 & 17.03 & 97.23 & 29.33 & 86.75 & 99.85 & 5.13 \\ 
			& CCU~\cite{meinke2019towards} & 0.50 & 99.90 & 24.17 & 96.11 & 52.17 & 62.24 & 99.42 & 1.60 \\ 
			& ROWL~\cite{sehwag2019analyzing} & 2.04 & 98.87 & 55.03 & 72.37 & 77.24 & 61.27 & 99.79 & {\bf 50.00} \\ 
			& ATOM (ours) & 0.07 & 99.97 & 5.47 & 98.52 & {\bf 7.02} & {\bf 98.00} & {\bf 96.33} & 49.52 \\ 
			\bottomrule
		\end{tabular}
	\end{adjustbox}
	\label{tab:main-results-svhn}
\end{table*}

\subsection{Choose Best q Using Validation Dataset}
\label{appendix:choose-best-q}

We create a validation OOD dataset by sampling 10,000 images from the 80 Million Tiny Images \cite{torralba200880}, which is disjoint from our training data. We choose $q$ from $\{0, 0.125, 0.25, 0.5, 0.75\}$.  The  results on the validation dataset are shown in Table~\ref{tab:validation-results}. We select the best model based on the average FPR at 5\% FNR across four types of OOD inputs. Based on the results, the optimal $q$ is $0$ for SVHN, $0.125$ for CIFAR-10 and $0.5$ for CIFAR-100. 

\begin{table*}[t]
    \caption[]{\small Evaluate models on validation dataset. We use DenseNet as network architecture. $\uparrow$ indicates larger value is better, and $\downarrow$ indicates lower value is better. All values are percentages and are averaged over six different OOD test datasets mentioned in Appendix~\ref{sec:more-experiment-settings}. \textbf{Bold} numbers are superior results.}
	\begin{adjustbox}{width=\columnwidth,center}
		\begin{tabular}{l|l|cc|cc|cc|cc}
			\toprule
			 \multirow{4}{0.1\linewidth}{$\mathcal{D}_{\text{in}}^{\text{test}}$} & \multirow{4}{0.06\linewidth}{\textbf{Method}}  &\bf{FPR}  & {\bf AUROC}  & {\bf FPR} & {\bf AUROC}  & {\bf FPR} & {\bf AUROC} & {\bf FPR} & {\bf AUROC}  \\
			 & & $\textbf{(5\% FNR)}$ & $\textbf{}$ & {\bf (5\% FNR)} & $\textbf{}$ & $\textbf{(5\% FNR)}$ & {\bf } & $\textbf{(5\% FNR)}$ & {\bf } \\
			 & & $\downarrow$ & $\uparrow$ & $\downarrow$ & $\uparrow$ & $\downarrow$ & $\uparrow$ & $\downarrow$ & $\uparrow$ \\ \cline{3-10}
			 & & \multicolumn{2}{c|}{\textbf{Natural OOD}} & \multicolumn{2}{c|}{\textbf{Corruption OOD}} & \multicolumn{2}{c|}{\textbf{$L_\infty$ OOD }} & \multicolumn{2}{c}{\textbf{Comp. OOD}} \\  \hline 
			\multirow{5}{0.1\linewidth}{{{\bf SVHN}}}  
            & ATOM (q=0.0) & {\bf 0.01} & 99.97 & {\bf 4.02} & 98.50 & {\bf 8.08} & 98.14 & {\bf 96.79} & 43.77 \\ 
			& ATOM (q=0.125) & 2.16 & 99.39 & 36.62 & 94.79 & 62.38 & 73.20 & 99.98 & 4.62 \\ 
			& ATOM (q=0.25) & 1.87 & 99.38 & 41.86 & 94.20 & 75.85 & 57.75 & 100.00 & 1.26 \\ 
			& ATOM (q=0.5) & 2.73 & 99.18 & 45.02 & 93.85 & 83.01 & 49.38 & 99.99 & 1.96 \\ 
			& ATOM (q=0.75) & 4.97 & 98.83 & 56.51 & 92.13 & 85.98 & 40.89 & 100.00 & 1.16 \\ \hline 
			\multirow{5}{0.1\linewidth}{{{\bf CIFAR-10}}}  
            & ATOM (q=0.0) & 5.34 & 98.35 & 43.61 & 91.85 & 22.08 & 92.81 & 59.17 & 83.56 \\ 
			& ATOM (q=0.125) & {\bf 4.77} & 98.31 & {\bf 27.49} & 94.24 & {\bf 5.42} & 98.19 & {\bf 29.02} & 93.76 \\ 
			& ATOM (q=0.25) & 5.70 & 98.11 & 28.13 & 93.68 & 19.28 & 95.71 & 40.68 & 91.17 \\ 
			& ATOM (q=0.5) & 8.83 & 97.66 & 39.74 & 91.53 & 9.80 & 97.42 & 44.10 & 90.45 \\ 
			& ATOM (q=0.75) & 12.42 & 97.02 & 45.85 & 90.53 & 13.40 & 96.83 & 46.30 & 90.42 \\ \hline 
			\multirow{5}{0.1\linewidth}{{{\bf CIFAR-100}}}  
            & ATOM (q=0.0) & 44.85 & 91.58 & 98.76 & 64.78 & 53.17 & 85.26 & 98.95 & 58.38 \\ 
			& ATOM (q=0.125) & 36.75 & 92.90 & 96.22 & 73.33 & 38.79 & 91.42 & 96.33 & 71.74 \\ 
			& ATOM (q=0.25) & {\bf 34.66} & 92.62 & 94.13 & 73.86 & {\bf 35.84} & 91.24 & 94.33 & 70.35 \\ 
			& ATOM (q=0.5) & 35.04 & 91.36 & {\bf 89.28} & 71.78 & 36.76 & 90.21 & {\bf 89.57} & 70.62 \\ 
			& ATOM (q=0.75) & 43.49 & 87.95 & 91.80 & 62.47 & 59.34 & 78.00 & 92.94 & 58.53 \\ 
			\bottomrule
		\end{tabular}
	\end{adjustbox}
	\label{tab:validation-results}
\end{table*}

\subsection{Effect of coefficient $\lambda$}
We study the effect of coefficient $\lambda$  in Table~\ref{tab:ablation-study-lambda}. We can see that when $\lambda$ is too small (e.g. $\lambda=0.01$), our algorithm can underperform. ATOM can achieve good performance under moderate value such as $\lambda=1.0$. For simplicity, we set $\lambda=1.0$ for all experiments. 

\begin{table*}[t!]
    \caption[]{\small Ablation study on $\lambda$. We use DenseNet as network architecture and set $q=0.125$. $\uparrow$ indicates larger value is better, and $\downarrow$ indicates lower value is better. All values are percentages and are averaged over six natural OOD test datasets mentioned in Appendix~\ref{sec:more-experiment-settings}. }
	\begin{adjustbox}{width=\columnwidth,center}
		\begin{tabular}{l|l|cc|cc|cc|cc}
			\toprule
			 \multirow{4}{0.08\linewidth}{$\mathcal{D}_{\text{in}}^{\text{test}}$} & \multirow{4}{0.06\linewidth}{\textbf{Model}}  &\bf{FPR}  & {\bf AUROC}  & {\bf FPR} & {\bf AUROC}  & {\bf FPR} & {\bf AUROC} & {\bf FPR} & {\bf AUROC}  \\
			 & & $\textbf{(5\% FNR)}$ & $\textbf{}$ & {\bf (5\% FNR)} & $\textbf{}$ & $\textbf{(5\% FNR)}$ & {\bf } & $\textbf{(5\% FNR)}$ & {\bf } \\
			 & & $\downarrow$ & $\uparrow$ & $\downarrow$ & $\uparrow$ & $\downarrow$ & $\uparrow$ & $\downarrow$ & $\uparrow$ \\ \cline{3-10}
			 & & \multicolumn{2}{c|}{\textbf{Natural OOD}} & \multicolumn{2}{c|}{\textbf{Corruption OOD}} & \multicolumn{2}{c|}{\textbf{$L_\infty$ OOD }} & \multicolumn{2}{c}{\textbf{Comp. OOD}} \\  \hline 
			\multirow{6}{0.1\linewidth}{{{\bf CIFAR-10}}}  
            & ATOM ($\lambda=0.01$) & 4.25 & 98.85 & 50.36 & 90.45 & 51.09 & 64.10 & 73.70 & 68.43 \\ 
			& ATOM ($\lambda=0.1$) & 2.05 & 99.25 & 28.95 & 94.73 & 37.41 & 77.81 & 50.51 & 81.43 \\ 
			& ATOM ($\lambda=1.0$) & 1.69 & 99.20 & 25.26 & 95.29 & 20.55 & 88.94 & 38.89 & 86.71  \\ 
			& ATOM ($\lambda=2.0$) & 1.89 & 99.29 & 26.19 & 94.41 & 14.01 & 94.17 & 32.06 & 92.79  \\ 
			& ATOM ($\lambda=5.0$) & 2.17 & 99.18 & 23.96 & 95.62 & 19.57 & 93.77 & 30.42 & 94.27 \\ 
			\bottomrule
		\end{tabular}
	\end{adjustbox}
	\label{tab:ablation-study-lambda}
\end{table*}

\subsection{Effect of Auxiliary OOD Datasets}
\label{sec:ablation-auxiliary-datasets}

We present results in Table~\ref{tab:main-results-imagenet}, where we use an alternative auxiliary OOD dataset ImageNet-RC. The details of ImageNet-RC is provided in Section~\ref{sec:more-experiment-settings}. We use the same hyperparameters as used in training with TinyImages auxiliary data. For all three in-distribution datasets, we find that using $q=0$ results in the optimal performance. 

\begin{table*}[t!]
    \caption[]{\small Comparison with competitive OOD detection methods. We use ImageNet-RC as the auxiliary OOD dataset (see section~\ref{sec:more-experiment-settings} for the details) for SOFL, OE, ACET, CCU, NTOM and ATOM. We use DenseNet as network architecture for all methods. We evaluate on four types of OOD inputs: (1) natural OOD, (2) corruption attacked OOD, (3) $L_\infty$ attacked OOD, and (4) compositionally attacked OOD inputs. $\uparrow$ indicates larger value is better, and $\downarrow$ indicates lower value is better. All values are percentages and are averaged over six natural OOD test datasets described in Appendix~\ref{sec:more-experiment-settings}. \textbf{Bold} numbers are superior results. }
	\begin{adjustbox}{width=\columnwidth,center}
		\begin{tabular}{l|l|cc|cc|cc|cc}
			\toprule
			\multirow{4}{0.12\linewidth}{$\mathcal{D}_{\text{in}}^{\text{test}}$} &  \multirow{4}{0.06\linewidth}{\textbf{Method}}  &\bf{FPR}  & {\bf AUROC}  & {\bf FPR} & {\bf AUROC}  & {\bf FPR} & {\bf AUROC} & {\bf FPR} & {\bf AUROC}  \\
			& & $\textbf{(5\% FNR)}$ & $\textbf{}$ & {\bf (5\% FNR)} & $\textbf{}$ & $\textbf{(5\% FNR)}$ & {\bf } & $\textbf{(5\% FNR)}$ & {\bf } \\
			& & $\downarrow$ & $\uparrow$ & $\downarrow$ & $\uparrow$ & $\downarrow$ & $\uparrow$ & $\downarrow$ & $\uparrow$ \\ \cline{3-10}
			& & \multicolumn{2}{c|}{\textbf{Natural OOD}} & \multicolumn{2}{c|}{\textbf{Corruption OOD}} & \multicolumn{2}{c|}{\textbf{$L_\infty$ OOD }} & \multicolumn{2}{c}{\textbf{Comp. OOD}} \\  \hline 
			\multirow{10}{0.12\linewidth}{{{\bf SVHN}}}  
			& MSP  & 38.84 & 93.57 & 99.68 & 68.48 & 99.89 & 1.39 & 100.00 & 0.19 \\
			& ODIN  & 31.45 & 93.52 & 97.11 & 63.21 & 99.86 & 0.61 & 100.00 & 0.05 \\
			& Mahalanobis & 22.80 & 95.57 & 93.14 & 60.78 & 97.33 & 8.89 & 99.89 & 0.23 \\
			& SOFL  & {\bf 0.02} & {\bf 99.99} & {\bf 5.93} & {\bf 98.57} & 58.53 & 68.85 & 67.34 & 61.42 \\
			& OE  & 0.13 & 99.96 & 15.76 & 97.51 & 68.76 & 49.57 & 98.80 & 6.21 \\ 
			& ACET  & 0.31 & 99.94 & 29.02 & 95.65 & 2.37 & 99.51 & 30.58 & 95.20 \\ 
			& CCU  & 0.17 & 99.96 & 18.64 & 96.94 & 45.38 & 69.14 & 92.30 & 20.88 \\ 
			& ROWL & 2.04 & 98.87 & 55.03 & 72.37 & 77.24 & 61.27 & 99.79 & 50.00 \\ 
			& ATOM (ours)  & {\bf 0.02} & {\bf 99.99} & 7.03 & 98.38 & {\bf 0.14} & {\bf 99.95} & {\bf 7.30} & {\bf 98.32} \\ \hline 
			\multirow{10}{0.12\linewidth}{{{\bf CIFAR-10}}}  
			& MSP  & 50.54 & 91.79 & 100.00 & 58.35 & 100.00 & 13.82 & 100.00 & 13.67 \\
			& ODIN  & 21.65 & 94.66 & 99.37 & 51.44 & 99.99 & 0.18 & 100.00 & 0.01 \\
			& Mahalanobis & 26.95 & 90.30 & 91.92 & 43.94 & 95.07 & 12.47 & 99.88 & 1.58 \\
			& SOFL   & 6.96 & 98.71 & 22.30 & 95.89 & 97.61 & 12.39 & 99.74 & 7.49 \\
			& OE   & 9.70 & 98.35 & 49.84 & 91.76 & 91.30 & 43.88 & 98.82 & 31.12 \\ 
			& ACET  & 10.72 & 98.01 & 53.85 & 90.19 & 17.10 & 96.01 & 55.21 & 89.78 \\ 
			& CCU  & 10.30 & 98.25 & 44.42 & 92.34 & 93.02 & 20.88 & 99.17 & 9.95 \\ 
			& ROWL  & 25.03 & 86.96 & 94.34 & 52.31 & 99.98 & 49.49 & 100.00 & 49.48 \\ 
			& ATOM (ours)  & {\bf 4.08} & {\bf 99.14} & {\bf 16.17} & {\bf 96.94} & {\bf 7.46} & {\bf 98.50} & {\bf 18.35} & {\bf 96.60} \\ \hline 
			\multirow{10}{0.12\linewidth}{{\bf CIFAR-100}} 
			& MSP & 78.05 & 76.11 & 100.00 & 30.04 & 100.00 & 2.25 & 100.00 & 2.06 \\
			& ODIN  & 56.77 & 83.62 & 100.00 & 36.95 & 100.00 & 0.14 & 100.00 & 0.00 \\
			& Mahalanobis  & 42.63 & 87.86 & 95.92 & 42.96 & 95.44 & 15.87 & 99.86 & 2.08 \\
			& SOFL  & 20.95 & 96.06 & 73.33 & 83.31 & 93.41 & 12.90 & 99.98 & 3.36 \\
			& OE   & 18.52 & 95.27 & 86.83 & 66.95 & 96.27 & 18.79 & 99.97 & 4.88 \\ 
			& ACET   & 19.79 & 94.76 & 81.63 & 70.04 & 26.23 & 91.46 & 81.95 & 69.67 \\ 
			& CCU   & 19.44 & 95.05 & 84.11 & 69.09 & 84.89 & 35.85 & 99.61 & 15.67 \\ 
			& ROWL  & 93.35 & 53.02 & 100.00 & 49.69 & 100.00 & 49.69 & 100.00 & 49.69 \\ 
			& ATOM (ours)   & {\bf 15.49} & {\bf 97.18} & {\bf 57.79} & {\bf 89.49} & {\bf 18.32} & {\bf 96.57} & {\bf 58.49} & {\bf 89.36} \\ 
			\bottomrule
		\end{tabular}
	\end{adjustbox}
	\label{tab:main-results-imagenet}
\end{table*}

\subsection{Effect of Network Architecture}
\label{sec:ablation-network-arch}
We perform experiments to evaluate different OOD detection methods using WideResNet, see Table~\ref{tab:main-results-wideresnet}. For ATOM, we use the same hyperparameters as those selected for DenseNet and find that it also leads to good results. 

\begin{table*}[t!]
    \caption[]{\small Comparison with competitive OOD detection methods. We use WideResNet as network architecture for all methods. We evaluate on four types of OOD inputs: (1) natural OOD, (2) corruption attacked OOD, (3) $L_\infty$ attacked OOD, and (4) compositionally attacked OOD inputs. $\uparrow$ indicates larger value is better, and $\downarrow$ indicates lower value is better. All values are percentages and are averaged over six different OOD test datasets described in Appendix~\ref{sec:more-experiment-settings}. \textbf{Bold} numbers are superior results. }
	\begin{adjustbox}{width=\columnwidth,center}
		\begin{tabular}{l|l|cc|cc|cc|cc}
			\toprule
			\multirow{4}{0.12\linewidth}{$\mathcal{D}_{\text{in}}^{\text{test}}$} &  \multirow{4}{0.06\linewidth}{\textbf{Method}}  &\bf{FPR}  & {\bf AUROC}  & {\bf FPR} & {\bf AUROC}  & {\bf FPR} & {\bf AUROC} & {\bf FPR} & {\bf AUROC}  \\
			& & $\textbf{(5\% FNR)}$ & $\textbf{}$ & {\bf (5\% FNR)} & $\textbf{}$ & $\textbf{(5\% FNR)}$ & {\bf } & $\textbf{(5\% FNR)}$ & {\bf } \\
			& & $\downarrow$ & $\uparrow$ & $\downarrow$ & $\uparrow$ & $\downarrow$ & $\uparrow$ & $\downarrow$ & $\uparrow$ \\ \cline{3-10}
			& & \multicolumn{2}{c|}{\textbf{Natural OOD}} & \multicolumn{2}{c|}{\textbf{Corruption OOD}} & \multicolumn{2}{c|}{\textbf{$L_\infty$ OOD }} & \multicolumn{2}{c}{\textbf{Comp. OOD}} \\  \hline 
			\multirow{10}{0.12\linewidth}{{{\bf SVHN}}}  
			& MSP  & 42.78 & 91.15 & 99.99 & 51.84 & 100.00 & 0.15 & 100.00 & 0.05 \\
			& ODIN  & 43.65 & 89.37 & 99.85 & 44.43 & 100.00 & 0.06 & 100.00 & 0.00 \\
			& Mahalanobis  & 6.94 & 98.47 & 80.91 & 78.92 & 93.44 & 39.42 & 99.70 & 2.69 \\
			& SOFL   & {\bf 0.02} & {\bf 99.99} & {\bf 2.41} & {\bf 99.36} & 85.25 & 39.32 & 99.28 & 2.85 \\
			& OE   & 0.58 & 99.87 & 21.97 & 96.00 & 70.66 & 47.36 & 96.63 & 5.87 \\ 
			& ACET  & 0.43 & 99.92 & 18.12 & 97.11 & 20.75 & 89.55 & 99.66 & 1.80 \\ 
			& CCU  & 0.71 & 99.86 & 30.96 & 94.40 & 73.37 & 41.82 & 99.83 & 0.44 \\ 
			& ROWL  & 1.46 & 99.13 & 45.18 & 77.28 & 86.95 & 56.39 & 99.94 & 49.90 \\ 
			& ATOM (ours) & 0.08 & 99.96 & 4.18 & 98.65 & {\bf 5.93} & {\bf 98.47} & {\bf 95.50} & {\bf 62.00} \\ \hline
			\multirow{10}{0.12\linewidth}{{{\bf CIFAR-10}}}  
			& MSP  & 53.81 & 91.10 & 100.00 & 50.30 & 100.00 & 5.48 & 100.00 & 5.43 \\
			& ODIN  & 36.25 & 91.18 & 99.95 & 31.78 & 100.00 & 0.04 & 100.00 & 0.00 \\
			& Mahalanobis  & 23.93 & 92.48 & 86.78 & 56.90 & 85.75 & 42.63 & 99.22 & 17.51 \\
			& SOFL   & 2.15 & {\bf 99.14} & 38.03 & 93.35 & 99.99 & 0.43 & 100.00 & 0.24 \\
			& OE   & 3.03 & 98.83 & 56.00 & 91.19 & 99.97 & 0.26 & 100.00 & 0.04 \\ 
			& ACET  & 3.19 & 98.74 & 35.78 & 94.25 & 40.22 & 89.52 & 56.08 & 87.76 \\ 
			& CCU  & 2.61 & 98.75 & 36.31 & 93.75 & 99.93 & 0.86 & 99.99 & 0.41 \\ 
			& ROWL  & 17.21 & 90.79 & 86.09 & 56.35 & 99.94 & 49.43 & 99.99 & 49.40 \\ 
			& ATOM (ours) & {\bf 1.50} & 98.87 & {\bf 13.17} & {\bf 96.36} & {\bf 33.66} & {\bf 90.89} & {\bf 40.07} & {\bf 88.99} \\ \hline
			\multirow{10}{0.12\linewidth}{{{\bf CIFAR-100}}}  
			& MSP  & 81.92 & 74.32 & 100.00 & 28.10 & 100.00 & 3.20 & 100.00 & 3.14 \\
			& ODIN  & 68.76 & 79.13 & 100.00 & 26.44 & 100.00 & 0.15 & 100.00 & 0.01 \\
			& Mahalanobis  & 43.48 & 86.02 & 95.78 & 42.25 & 95.83 & 17.41 & 99.92 & 3.55 \\
			& SOFL   & 66.57 & 89.11 & 99.99 & 59.64 & 99.99 & 0.30 & 100.00 & 0.10 \\
			& OE   & 46.14 & 91.16 & 99.96 & 58.55 & 100.00 & 1.80 & 100.00 & 1.69 \\ 
			& ACET  & 46.79 & 91.67 & 99.76 & 62.48 & 95.66 & 41.38 & 99.99 & 36.65 \\ 
			& CCU  & 48.35 & 91.41 & 99.95 & 57.94 & 100.00 & 0.44 & 100.00 & 0.38 \\ 
			& ROWL  & 96.93 & 51.34 & 100.00 & 49.81 & 100.00 & 49.81 & 100.00 & 49.81 \\ 
			& ATOM (ours) & {\bf 30.52} & {\bf 93.90} & {\bf 94.46} & {\bf 74.96} & {\bf 44.30} & {\bf 89.68} & {\bf 95.06} & {\bf 74.25} \\
			\bottomrule
		\end{tabular}
	\end{adjustbox}
	\label{tab:main-results-wideresnet}
\end{table*}

\subsection{Effect of PGD Attack Strength}
\label{sec:ablation-attack-strength}
To see the effect of using stronger PGD attacks, we evaluate ACET (best baseline) and ATOM on $L_\infty$ attacked OOD and compositionally attacked OOD inputs with 100 iterations and 5 random restarts. Results are provided in Table~\ref{tab:ablation-attack-strength}. Under  stronger PGD attack, ATOM outperforms ACET. 

\begin{table*}[t!]
    \centering
    \footnotesize
    \caption[]{\small Evaluation on $L_\infty$ attacked OOD and compositionally attacked OOD inputs with strong PGD attack (100 iterations and 5 random restarts). We use DenseNet as network architecture for all methods. $\uparrow$ indicates larger value is better, and $\downarrow$ indicates lower value is better. All values are percentages and are averaged over six different OOD test datasets described in Appendix~\ref{sec:more-experiment-settings}. \textbf{Bold} numbers are superior results.}
	\begin{tabular}{l|l|cc|cc}
		\toprule
		\multirow{4}{0.12\linewidth}{$\mathcal{D}_{\text{in}}^{\text{test}}$} &  \multirow{4}{0.06\linewidth}{\textbf{Method}} & {\bf FPR} & {\bf AUROC} & {\bf FPR} & {\bf AUROC}  \\
	     & & $\textbf{(5\% FNR)}$ & {\bf } & $\textbf{(5\% FNR)}$ & {\bf } \\
		 & & $\downarrow$ & $\uparrow$ & $\downarrow$ & $\uparrow$ \\ \cline{3-6}
		 & &  \multicolumn{2}{c|}{\textbf{$L_\infty$ OOD }} & \multicolumn{2}{c}{\textbf{Comp. OOD}} \\  \hline 
		\multirow{2}{0.12\linewidth}{{{\bf SVHN}}}  
		& ACET  & 33.64 & 83.89 & 99.90 & 2.84 \\ 
		& ATOM (ours) & {\bf 8.33} & {\bf 97.37} & {\bf 98.56} & {\bf 35.32} \\ \hline
		\multirow{2}{0.12\linewidth}{{{\bf CIFAR-10}}}  
		& ACET  & 76.66 & 76.29 & 97.08 & 51.35 \\ 
		& ATOM (ours)  & {\bf 26.88} & {\bf 83.78} & {\bf 42.86} & {\bf 83.56} \\ \hline
		\multirow{2}{0.12\linewidth}{{{\bf CIFAR-100}}}  
		& ACET  & 87.13 & 48.77 & 99.82 & 33.67 \\ 
		& ATOM (ours) & {\bf 45.11} & {\bf 82.32} & {\bf 93.74} & {\bf 65.88} \\ 
		\bottomrule
	\end{tabular}
	\label{tab:ablation-attack-strength}
\end{table*}

\subsection{Performance of OOD Detector and Classifier on In-distribution Data}
\label{app:in-distribution-acc}

We summarize the performance of OOD detector $G(x)$ and image classifier $f(x)$ on in-distribution test data. The results in Table \ref{tab:main-results-in-distribution} show that ATOM improves the OOD detection performance while achieving in-distribution classification accuracy that is on par with a pre-trained network. 

\begin{table*}[t]
        \center
        \small
        \caption[]{\small The performance of OOD detector and classifier on in-distribution test data. We use DenseNet for all methods. We use three metrics: FNR, Prediction Accuracy and End-to-end Prediction Accuracy. We pick the threshold for the OOD detectors such that  95\% of in-distribution test data points are classified as in-distribution. Prediction Accuracy measures the accuracy of the classifier on in-distribution test data. End-to-end Prediction Accuracy measures the accuracy of the open world classification system (detector+classifier), where an example is classified correctly if and only if the detector treats it as in-distribution and the classifier predicts its label correctly. }
		\begin{tabular}{l|l|ccc}
			\toprule
	    \multirow{2}{0.14\linewidth}{$\mathcal{D}_{\text{in}}^{\text{test}}$} & \multirow{2}{0.08\linewidth}{\textbf{Method}} &\bf{FNR}   & {\bf Pred.} & {\bf End-to-end.}  \\ 
			&  &  & \bf{Acc.} & \bf{Pred. Acc.} \\ \hline
			\multirow{9}{0.14\linewidth}{{{\bf SVHN}}}  
			& MSP & 5.01 & 95.83 & 93.00 \\
			& ODIN & 5.01 & 95.83 & 92.44 \\
			& Mahalanobis & 5.01 & 95.83 & 91.50 \\
			& SOFL  & 5.01 & 96.45 & 92.81 \\
			& OE & 5.01 & 95.93 & 93.11  \\ 
			& ACET & 5.01 & 95.58 & 92.79 \\ 
			& CCU & 5.01 & 95.85 & 92.95 \\ 
			& ROWL & 0.22 & 95.23 & 95.23 \\ 	
			& ATOM (ours) & 5.01 & 96.09 & 91.95 \\ \hline
			\multirow{9}{0.14\linewidth}{{{\bf CIFAR-10}}}  
			& MSP & 5.01 & 94.39 & 91.76 \\
			& ODIN & 5.01 & 94.39 & 91.00 \\
			& Mahalanobis & 5.01 & 94.39 & 89.72 \\
			& SOFL  & 5.01 & 95.11 & 91.60 \\
			& OE & 5.01 & 94.79 & 91.86  \\ 
			& ACET & 5.01 & 91.48 & 88.61 \\ 
			& CCU & 5.01 & 94.89 & 91.88 \\ 
			& ROWL & 1.04 & 93.18 & 93.18 \\ 	
			& ATOM (ours) & 5.01 & 95.20 & 91.33 \\ \hline
			\multirow{9}{0.14\linewidth}{{\bf CIFAR-100}} 
			& MSP & 5.01 & 75.05 & 73.87 \\
			& ODIN & 5.01 & 75.05 & 73.72 \\
			& Mahalanobis & 5.01 & 75.05 & 71.12 \\
			& SOFL  & 5.01 & 74.37 & 72.62 \\
			& OE & 5.01 & 75.28 & 73.74  \\ 
			& ACET & 5.01 & 74.43 & 72.72 \\ 
			& CCU & 5.01 & 76.04 & 74.60 \\ 
			& ROWL & 0.62 & 72.51 & 72.51 \\ 	
			& ATOM (ours) & 5.01 & 75.06 & 72.72 \\ 
			\bottomrule
		\end{tabular}
	\label{tab:main-results-in-distribution}
\end{table*}

\subsection{Complete Experimental Results}
\label{sec:complete-experimental-results}

We report the performance of OOD detectors on each of the six natural OOD test datasets in Table \ref{tab:svhn-detail-results} (SVHN), Table \ref{tab:cifar10-detail-results} (CIFAR-10) and Table \ref{tab:cifar100-detail-results} (CIFAR-100).  

\begin{table*}[t]
    \caption[]{\small Comparison with competitive OOD detection methods. We use SVHN as in-distribution dataset and use DenseNet as network architecture for all methods. We evaluate the performance on all four types of OOD inputs: (1) natural OOD, (2) corruption attacked OOD, (3) $L_\infty$ attacked OOD, and (4) compositionally attacked OOD inputs. $\uparrow$ indicates larger value is better, and $\downarrow$ indicates lower value is better. All values are percentages. \textbf{Bold} numbers are superior results.}
	\begin{adjustbox}{width=\columnwidth,center}
		\begin{tabular}{l|l|cc|cc|cc|cc}
			\toprule
			\multirow{4}{0.14\linewidth}{$\mathcal{D}_{\text{out}}^{\text{test}}$} &  \multirow{4}{0.06\linewidth}{\textbf{Method}}  &\bf{FPR}  & {\bf AUROC}  & {\bf FPR} & {\bf AUROC}  & {\bf FPR} & {\bf AUROC} & {\bf FPR} & {\bf AUROC}  \\
			& & $\textbf{(5\% FNR)}$ & $\textbf{}$ & {\bf (5\% FNR)} & $\textbf{}$ & $\textbf{(5\% FNR)}$ & {\bf } & $\textbf{(5\% FNR)}$ & {\bf } \\
			& & $\downarrow$ & $\uparrow$ & $\downarrow$ & $\uparrow$ & $\downarrow$ & $\uparrow$ & $\downarrow$ & $\uparrow$ \\ \cline{3-10}
			& & \multicolumn{2}{c|}{\textbf{Natural OOD}} & \multicolumn{2}{c|}{\textbf{Corruption OOD}} & \multicolumn{2}{c|}{\textbf{$L_\infty$ OOD }} & \multicolumn{2}{c}{\textbf{Comp. OOD}} \\  \hline 
			\multirow{10}{0.14\linewidth}{{{\bf LSUN-C}}}  
			& MSP  &  33.90 &  94.01 &  99.59 &  68.76 &  99.94 &   0.82 & 100.00 &   0.11  \\
			& ODIN  & 28.40 &  93.77 &  97.37 &  59.61 &  99.90 &   0.38 & 100.00 &   0.01  \\
			& Mahalanobis & 0.20 &  99.64 &  79.20 &  80.24 &  98.01 &  10.95 &  99.99 &   0.07 \\
			& SOFL  & {\bf 0.00} & {\bf 100.00} &   {\bf 2.10} &  {\bf 99.38} &  68.89 &  54.15 &  99.98 &   0.31 \\
			& OE  & {\bf 0.00} & {\bf 100.00} &  23.06 &  96.38 &  54.03 &  77.79 &  99.91 &   0.52 \\ 
			& ACET & {\bf 0.00} & {\bf 100.00} &  20.48 &  96.75 &   0.09 &  99.97 &  99.85 &   4.76 \\ 
			& CCU & {\bf 0.00} & {\bf 100.00} &  22.26 &  96.46 &  35.10 &  80.37 &  99.92 &   0.55 \\ 
			& ROWL & 2.50 &  98.64 &  75.48 &  62.15 &  96.22 &  51.78 &  99.99 &  {\bf 49.90} \\ 
			& ATOM (ours) & {\bf 0.00} & {\bf 100.00} &  16.71 &  96.91 &   {\bf 0.00} & {\bf 100.00} &  {\bf 97.85} &  39.64 \\ \hline
			\multirow{10}{0.14\linewidth}{{{\bf LSUN-R}}}  
			& MSP  &  44.84 &  92.75 &  99.93 &  66.82 & 100.00 &   1.22 & 100.00 &   0.11  \\
			& ODIN  & 36.84 &  92.91 &  98.53 &  62.91 & 100.00 &   0.25 & 100.00 &   0.01 \\
			& Mahalanobis & 18.35 &  96.64 &  98.92 &  54.27 &  99.04 &  10.01 & 100.00 &   0.00  \\
			& SOFL  & {\bf 0.00} & {\bf 100.00} &   0.55 &  {\bf 99.75} &  49.23 &  77.50 &  99.88 &   4.77 \\
			& OE  & 0.02 & {\bf 100.00} &   9.44 &  98.34 &  41.86 &  79.82 &  99.91 &   1.54 \\ 
			& ACET & {\bf 0.00} & {\bf 100.00} &   5.36 &  98.87 &   2.60 &  99.35 &  99.96 &   7.00 \\ 
			& CCU & {\bf 0.00} & {\bf 100.00} &   9.44 &  98.46 &  12.25 &  94.36 &  98.90 &   2.78 \\ 
			& ROWL & 0.01 &  99.89 &  22.57 &  88.61 &  39.54 &  80.12 &  99.68 &  50.05 \\ 
			& ATOM (ours) & {\bf 0.00} & {\bf 100.00} &   {\bf 0.13} &  99.67 &   {\bf 0.09} &  {\bf 99.97} &  {\bf 96.08} &  {\bf 63.20} \\ \hline
			\multirow{10}{0.14\linewidth}{{{\bf iSUN}}}  
			& MSP  &  40.78 &  93.57 &  99.85 &  68.88 &  99.99 &   1.37 & 100.00 &   0.14  \\
			& ODIN  & 31.98 &  93.97 &  97.70 &  64.77 &  99.99 &   0.29 & 100.00 &   0.01  \\
			& Mahalanobis & 19.73 &  96.25 &  98.35 &  56.95 &  98.73 &  10.21 & 100.00 &   0.01 \\
			& SOFL  & {\bf 0.00} & {\bf 100.00} &   0.75 &  {\bf 99.70} &  56.78 &  69.82 &  99.87 &   4.39 \\
			& OE  & 0.01 & {\bf 100.00} &   9.50 &  98.36 &  45.45 &  76.08 &  99.76 &   1.72 \\ 
			& ACET & 0.01 & {\bf 100.00} &   6.17 &  98.84 &   4.90 &  98.52 &  99.92 &   6.96 \\ 
			& CCU & 0.02 & {\bf 100.00} &   9.94 &  98.32 &  17.36 &  90.54 &  99.13 &   2.56 \\ 
			& ROWL & 0.04 &  99.87 &  24.43 &  87.68 &  44.62 &  77.58 &  99.29 &  50.24 \\ 
			& ATOM (ours) & {\bf 0.00} & {\bf 100.00} &   {\bf 0.24} &  99.59 &   {\bf 0.25} &  {\bf 99.94} &  {\bf 95.20} &  {\bf 61.70} \\ \hline
			\multirow{10}{0.14\linewidth}{{{\bf Textures}}}  
			& MSP  &  41.91 &  92.08 &  99.33 &  61.77 &  99.45 &   2.84 & 100.00 &   0.44  \\
			& ODIN  &  40.48 &  89.91 &  97.36 &  52.69 &  99.29 &   2.14 & 100.00 &   0.27 \\
			& Mahalanobis & 24.34 &  94.15 &  85.94 &  62.00 &  88.83 &  17.73 &  99.38 &   1.27 \\
			& SOFL  & {\bf 0.34} &  {\bf 99.91} &  11.83 &  97.47 &  89.36 &  23.54 &  99.20 &   2.42 \\
			& OE  & 2.73 &  99.50 &  43.05 &  92.51 &  84.36 &  26.88 &  98.32 &   2.72 \\ 
			& ACET & 2.55 &  99.54 &  32.55 &  94.67 &  64.57 &  58.46 &  99.40 &   5.73 \\ 
			& CCU & 2.22 &  99.57 &  41.86 &  92.46 &  78.09 &  31.94 &  98.58 &   2.31 \\ 
			& ROWL & 7.85 &  95.96 &  71.17 &  64.30 &  89.73 &  55.02 &  99.77 &  {\bf 50.01} \\ 
			& ATOM (ours) & 0.41 &  99.83 &  {\bf 11.31} &  {\bf 97.66} &  {\bf 28.81} &  {\bf 90.97} &  {\bf 91.61} &  49.35 \\ \hline
			\multirow{10}{0.14\linewidth}{{{\bf Places365}}}  
			& MSP  &  36.14 &  94.37 &  99.60 &  71.71 &  99.99 &   1.03 & 100.00 &   0.14  \\
			& ODIN  & 26.84 &  95.01 &  96.25 &  68.40 &  99.99 &   0.32 & 100.00 &   0.01  \\
			& Mahalanobis & 34.83 &  93.95 &  98.14 &  55.70 &  99.64 &   1.97 & 100.00 &   0.00 \\
			& SOFL  & {\bf 0.00} &  {\bf 99.99} &   3.46 &  {\bf 99.09} &  94.58 &  26.96 &  99.96 &   1.98 \\
			& OE  & 0.37 &  99.91 &  27.78 &  95.86 &  94.97 &  26.54 &  99.99 &   0.47 \\ 
			& ACET & 0.15 &  99.96 &  19.28 &  97.06 &  49.09 &  83.63 &  99.99 &   2.65 \\ 
			& CCU & 0.27 &  99.95 &  29.12 &  95.65 &  83.26 &  39.88 & 100.00 &   0.59 \\ 
			& ROWL & 0.97 &  99.41 &  67.26 &  66.26 &  96.67 &  51.56 & 100.00 &  {\bf 49.89} \\ 
			& ATOM (ours) & {\bf 0.00} &  {\bf 99.99} &   {\bf 2.31} &  98.59 &   {\bf 5.60} &  {\bf 98.71} &  {\bf 98.74} &  38.95 \\ \hline
			\multirow{10}{0.14\linewidth}{{{\bf CIFAR-10}}}  
			& MSP  &  35.45 &  94.66 &  99.79 &  72.97 & 100.00 &   1.07 & 100.00 &   0.19  \\
			& ODIN  &  24.15 &  95.54 &  95.45 &  70.89 & 100.00 &   0.28 & 100.00 &   0.01 \\
			& Mahalanobis & 39.35 &  92.80 &  98.28 &  55.52 &  99.70 &   2.47 & 100.00 &   0.01 \\
			& SOFL  & 0.01 &  {\bf 99.99} &   3.99 &  {\bf 99.03} &  93.00 &  28.70 &  99.95 &   2.60 \\
			& OE  & 0.48 &  99.89 &  27.79 &  95.92 &  95.49 &  26.02 &  99.98 &   0.63 \\ 
			& ACET & 0.21 &  99.94 &  18.32 &  97.18 &  54.71 &  80.56 &  99.98 &   3.69 \\ 
			& CCU & 0.47 &  99.91 &  32.40 &  95.34 &  86.98 &  36.35 &  99.98 &   0.79 \\ 
			& ROWL & 0.84 &  99.47 &  69.28 &  65.25 &  96.68 &  51.55 & 100.00 &  {\bf 49.89} \\ 
			& ATOM (ours) & {\bf 0.00} &  99.98 &   {\bf 2.14} &  98.72 &   {\bf 7.39} &  {\bf 98.42} &  {\bf 98.51} &  44.28 \\ 
			\bottomrule
		\end{tabular}
	\end{adjustbox}
	\label{tab:svhn-detail-results}
\end{table*}

\begin{table*}[t]
    \caption[]{\small Comparison with competitive OOD detection methods. We use CIFAR-10 as in-distribution dataset and use DenseNet as network architecture for all methods. We evaluate the performance on all four types of OOD inputs: (1) natural OOD, (2) corruption attacked OOD, (3) $L_\infty$ attacked OOD, and (4) compositionally attacked OOD inputs. $\uparrow$ indicates larger value is better, and $\downarrow$ indicates lower value is better. All values are percentages. \textbf{Bold} numbers are superior results.}
	\begin{adjustbox}{width=\columnwidth,center}
		\begin{tabular}{l|l|cc|cc|cc|cc}
			\toprule
			\multirow{4}{0.14\linewidth}{$\mathcal{D}_{\text{out}}^{\text{test}}$} &  \multirow{4}{0.06\linewidth}{\textbf{Method}}  &\bf{FPR}  & {\bf AUROC}  & {\bf FPR} & {\bf AUROC}  & {\bf FPR} & {\bf AUROC} & {\bf FPR} & {\bf AUROC}  \\
			& & $\textbf{(5\% FNR)}$ & $\textbf{}$ & {\bf (5\% FNR)} & $\textbf{}$ & $\textbf{(5\% FNR)}$ & {\bf } & $\textbf{(5\% FNR)}$ & {\bf } \\
			& & $\downarrow$ & $\uparrow$ & $\downarrow$ & $\uparrow$ & $\downarrow$ & $\uparrow$ & $\downarrow$ & $\uparrow$ \\ \cline{3-10}
			& & \multicolumn{2}{c|}{\textbf{Natural OOD}} & \multicolumn{2}{c|}{\textbf{Corruption OOD}} & \multicolumn{2}{c|}{\textbf{$L_\infty$ OOD }} & \multicolumn{2}{c}{\textbf{Comp. OOD}} \\  \hline 
			\multirow{10}{0.14\linewidth}{{{\bf LSUN-C}}}  
			& MSP  &  27.34 &  96.30 & 100.00 &  71.57 & 100.00 &  13.75 & 100.00 &  13.68  \\
			& ODIN  &  1.89 &  99.50 &  98.94 &  71.86 & 100.00 &   0.06 & 100.00 &   0.00 \\
			& Mahalanobis & 14.82 &  94.63 &  93.31 &  46.23 &  98.07 &   8.04 &  99.99 &   1.36 \\
			& SOFL  & 0.39 &  99.40 &  56.07 &  93.02 & 100.00 &   2.36 & 100.00 &   1.96 \\
			& OE  & 0.97 &  99.52 &  40.85 &  93.99 &  99.99 &   0.36 & 100.00 &   0.19 \\ 
			& ACET & 1.76 &  99.42 &  31.46 &  95.40 &  45.54 &  90.57 &  88.09 &  70.84 \\ 
			& CCU & 0.62 &  {\bf 99.65} &  33.44 &  94.56 &  99.97 &   0.31 & 100.00 &   0.05 \\ 
			& ROWL & 10.92 &  94.02 &  91.67 &  53.64 & 100.00 &  49.48 & 100.00 &  49.48\\ 
			& ATOM (ours) & {\bf 0.25} &  99.53 &  {\bf 16.38} &  {\bf 96.75} &   {\bf 0.38} &  {\bf 99.49} &  {\bf 17.27} &  {\bf 96.49} \\ \hline
			\multirow{10}{0.14\linewidth}{{{\bf LSUN-R}}}  
			& MSP  &  43.89 &  93.93 & 100.00 &  64.26 & 100.00 &  13.74 & 100.00 &  13.66  \\
			& ODIN  &  3.29 &  99.20 &  98.96 &  63.83 & 100.00 &   0.14 & 100.00 &   0.00 \\
			& Mahalanobis & 7.43 &  97.88 &  98.71 &  39.24 &  94.59 &  16.62 & 100.00 &   0.45 \\
			& SOFL  & 1.67 &  99.29 &  56.41 &  90.62 & 100.00 &   0.55 & 100.00 &   0.43 \\
			& OE  & 0.99 &  {\bf 99.43} &  51.55 &  92.25 &  99.99 &   0.15 & 100.00 &   0.02 \\ 
			& ACET & 3.87 &  99.10 &  77.90 &  87.37 &  70.50 &  {\bf 84.01} &  99.87 &  48.53 \\ 
			& CCU & 1.53 &  99.28 &  57.03 &  90.53 & 100.00 &   0.04 & 100.00 &   0.10 \\ 
			& ROWL & 36.03 &  81.47 &  98.72 &  50.12 & 100.00 &  49.48 & 100.00 &  49.48 \\ 
			& ATOM (ours) & {\bf 0.41} &  99.33 &  {\bf 16.77} &  {\bf 96.63} &  {\bf 52.76} &  71.84 &  {\bf 50.33} &  {\bf 75.18} \\ \hline
			\multirow{10}{0.14\linewidth}{{{\bf iSUN}}}  
			& MSP  &  46.18 &  93.58 & 100.00 &  62.76 &  99.99 &  13.95 & 100.00 &  13.67 \\
			& ODIN  & 4.45 &  99.00 &  98.90 &  62.14 & 100.00 &   0.33 & 100.00 &   0.01 \\
			& Mahalanobis & 8.58 &  98.00 &  98.10 &  42.96 &  89.64 &  22.78 & 100.00 &   0.74 \\
			& SOFL  & 2.24 &  99.22 &  53.99 &  90.97 & 100.00 &   0.51 & 100.00 &   0.50 \\
			& OE  & 1.14 &  {\bf 99.40} &  48.25 &  92.46 &  99.97 &   0.13 & 100.00 &   0.02 \\ 
			& ACET & 6.16 &  98.59 &  75.36 &  87.00 &  78.59 &  {\bf 79.84} &  99.63 &  47.11 \\ 
			& CCU & 1.74 &  99.27 &  52.44 &  91.10 & 100.00 &   0.05 & 100.00 &   0.08 \\ 
			& ROWL & 35.44 &  81.76 &  97.02 &  50.97 & 100.00 &  49.48 & 100.00 &  49.48 \\ 
			& ATOM (ours) & {\bf 0.66} &  99.34 &  {\bf 15.10} &  {\bf 96.79} &  {\bf 52.48} &  68.44 &  {\bf 51.94} &  {\bf 72.08} \\ \hline
			\multirow{10}{0.14\linewidth}{{{\bf Textures}}}  
			& MSP  &  64.66 &  87.64 & 100.00 &  51.85 & 100.00 &  14.14 & 100.00 &  13.72 \\
			& ODIN  & 52.45 &  84.81 &  99.56 &  38.20 &  99.95 &   0.56 & 100.00 &   0.07 \\
			& Mahalanobis & 25.39 &  92.20 &  71.42 &  61.60 &  88.85 &  17.24 &  99.27 &   3.60 \\
			& SOFL  & 3.78 &  99.04 &  56.81 &  89.38 &  99.88 &   2.06 &  99.98 &   1.36 \\
			& OE  & 6.24 &  98.43 &  53.32 &  88.83 &  99.68 &   1.40 &  99.95 &   0.68 \\ 
			& ACET & 11.74 &  97.96 &  54.41 &  90.52 &  64.49 &  77.66 &  94.26 &  55.24 \\ 
			& CCU & 5.83 &  98.45 &  54.61 &  86.33 &  99.49 &   1.68 &  99.84 &   0.99 \\ 
			& ROWL & 19.33 &  89.82 &  82.87 &  58.04 &  99.89 &  49.53 &  99.98 &  49.49 \\ 
			& ATOM (ours) & {\bf 1.81} &  {\bf 99.47} &  {\bf 20.05} &  {\bf 95.92} &  {\bf 10.05} &  {\bf 96.47} &  {\bf 29.01} &  {\bf 91.40} \\ \hline
			\multirow{10}{0.14\linewidth}{{{\bf Places365}}}  
			& MSP  &  62.03 &  88.29 & 100.00 &  57.68 & 100.00 &  13.66 & 100.00 &  13.66 \\
			& ODIN  & 43.84 &  90.45 &  99.89 &  52.81 & 100.00 &   0.01 & 100.00 &   0.00 \\
			& Mahalanobis & 85.77 &  65.76 &  99.47 &  22.75 &  99.79 &   1.93 & 100.00 &   0.66 \\
			& SOFL  & 7.73 &  97.81 &  61.66 &  88.07 & 100.00 &   0.47 & 100.00 &   0.27 \\
			& OE  & 11.08 &  97.00 &  67.91 &  87.39 & 100.00 &   0.03 & 100.00 &   0.01 \\ 
			& ACET & 18.63 &  95.97 &  79.42 &  85.00 &  93.09 &  66.83 &  99.61 &  46.19 \\ 
			& CCU & 8.49 &  97.63 &  67.68 &  85.75 & 100.00 &   0.03 & 100.00 &   0.03 \\ 
			& ROWL & 43.76 &  77.60 &  97.67 &  50.64 & 100.00 &  49.48 & 100.00 &  49.48 \\ 
			& ATOM (ours) & {\bf 6.30} &  {\bf 97.92} &  {\bf 31.44} &  {\bf 93.41} &  {\bf 6.96} & {\bf 97.78} & {\bf 32.99} & {\bf 92.93} \\ \hline
			\multirow{10}{0.14\linewidth}{{{\bf SVHN}}}  
			& MSP  & 59.15 &  90.99 & 100.00 &  41.97 & 100.00 &  13.67 & 100.00 &  13.66 \\
			& ODIN  & 23.96 &  95.00 & 100.00 &  19.78 & 100.00 &   0.00 & 100.00 &   0.00 \\
			& Mahalanobis & 19.73 &  93.33 &  90.52 &  50.89 &  99.48 &   8.18 & 100.00 &   2.67 \\
			& SOFL  & 0.85 &  99.47 &  87.50 &  79.85 & 100.00 &   0.13 & 100.00 &   0.07 \\
			& OE  & 1.55 &  99.16 &  75.62 &  89.02 & 100.00 &   0.01 & 100.00 &   0.02 \\ 
			& ACET &  31.50 &  94.99 &  83.04 &  85.25 &  94.49 &  69.42 &  99.81 &  54.36 \\ 
			& CCU & 2.14 &  99.25 &  75.34 &  88.01 & 100.00 &   0.00 & 100.00 &   0.00 \\ 
			& ROWL & 4.73 &  97.12 &  98.09 &  50.43 & 100.00 &  49.48 & 100.00 &  49.48 \\ 
			& ATOM (ours) & {\bf 0.69} & {\bf 99.63} &  {\bf 51.84} &  {\bf 92.22} & {\bf 0.69} & {\bf 99.63} & {\bf 51.81} & {\bf 92.18} \\  
			\bottomrule
		\end{tabular}
	\end{adjustbox}
	\label{tab:cifar10-detail-results}
\end{table*}

\begin{table*}[t]
    \caption[]{\small Comparison with competitive OOD detection methods. We use CIFAR-100 as in-distribution dataset and use DenseNet as network architecture for all methods. We evaluate the performance on all four types of OOD inputs: (1) natural OOD, (2) corruption attacked OOD, (3) $L_\infty$ attacked OOD, and (4) compositionally attacked OOD inputs. $\uparrow$ indicates larger value is better, and $\downarrow$ indicates lower value is better. All values are percentages. \textbf{Bold} numbers are superior results.}
	\begin{adjustbox}{width=\columnwidth,center}
		\begin{tabular}{l|l|cc|cc|cc|cc}
			\toprule
			\multirow{4}{0.14\linewidth}{$\mathcal{D}_{\text{out}}^{\text{test}}$} &  \multirow{4}{0.12\linewidth}{\textbf{Method}}  &\bf{FPR}  & {\bf AUROC}  & {\bf FPR} & {\bf AUROC}  & {\bf FPR} & {\bf AUROC} & {\bf FPR} & {\bf AUROC}  \\
			& & $\textbf{(5\% FNR)}$ & $\textbf{}$ & {\bf (5\% FNR)} & $\textbf{}$ & $\textbf{(5\% FNR)}$ & {\bf } & $\textbf{(5\% FNR)}$ & {\bf } \\
			& & $\downarrow$ & $\uparrow$ & $\downarrow$ & $\uparrow$ & $\downarrow$ & $\uparrow$ & $\downarrow$ & $\uparrow$ \\ \cline{3-10}
			& & \multicolumn{2}{c|}{\textbf{Natural OOD}} & \multicolumn{2}{c|}{\textbf{Corruption OOD}} & \multicolumn{2}{c|}{\textbf{$L_\infty$ OOD }} & \multicolumn{2}{c}{\textbf{Comp. OOD}} \\  \hline 
			\multirow{10}{0.14\linewidth}{{{\bf LSUN-C}}}  
			& MSP  &  62.03 &  84.78 & 100.00 &  32.47 & 100.00 &   2.52 & 100.00 &   2.31  \\
			& ODIN  & 15.47 &  97.34 & 100.00 &  42.25 & 100.00 &   0.20 & 100.00 &   0.01 \\
			& Mahalanobis & 47.44 &  93.47 &  98.80 &  58.23 &  98.94 &  15.07 &  99.97 &   3.01 \\
			& SOFL  & 17.38 &  96.66 & 100.00 &  51.59 & 100.00 &   1.10 & 100.00 &   0.58 \\
			& OE  & 14.75 &  97.33 &  99.91 &  54.29 & 100.00 &   1.61 & 100.00 &   0.64 \\ 
			& ACET & 14.60 &  97.41 &  98.65 &  67.81 &  23.07 &  95.35 &  98.88 &  56.02 \\ 
			& CCU & {\bf 12.03} &  {\bf 97.84} &  99.60 &  61.24 & 100.00 &   0.86 & 100.00 &   0.50 \\ 
			& ROWL & 88.67 &  55.35 & 100.00 &  49.69 & 100.00 &  49.69 & 100.00 &  49.69 \\ 
			& ATOM (ours) &  21.40 &  96.31 &  {\bf 79.72} &  {\bf 87.67} &  {\bf 21.50} &  {\bf 96.26} &  {\bf 79.98} &  {\bf 87.11} \\ \hline
			\multirow{10}{0.14\linewidth}{{{\bf LSUN-R}}}  
			& MSP  &  77.48 &  76.40 & 100.00 &  32.21 & 100.00 &   1.94 & 100.00 &   1.78  \\
			& ODIN  & 34.81 &  93.37 & 100.00 &  45.41 & 100.00 &   0.22 & 100.00 &   0.00  \\
			& Mahalanobis & {\bf 14.87} &  97.06 &  99.89 &  31.14 &  94.81 &  24.18 & 100.00 &   0.37 \\
			& SOFL  &  50.27 &  90.28 &  99.88 &  50.12 & 100.00 &   0.11 & 100.00 &   0.20 \\
			& OE  & 56.25 &  84.35 &  99.96 &  41.17 & 100.00 &   0.70 & 100.00 &   0.52 \\ 
			& ACET & 56.35 &  88.17 &  99.50 &  51.85 &  98.67 &  18.64 &  99.74 &  22.45 \\ 
			& CCU & 38.44 &  91.83 &  99.94 &  50.62 & 100.00 &   0.59 & 100.00 &   0.47 \\ 
			& ROWL & 88.25 &  55.57 & 100.00 &  49.69 & 100.00 &  49.69 & 100.00 &  49.69 \\ 
			& ATOM (ours) & 17.93 &  96.94 &  {\bf 95.72} & {\bf 71.51} & {\bf 31.39} & {\bf 87.38} &  {\bf 95.87} &  {\bf 66.15} \\ \hline
			\multirow{10}{0.14\linewidth}{{{\bf iSUN}}}  
			& MSP  &  78.87 &  75.69 & 100.00 &  31.83 & 100.00 &   2.08 & 100.00 &   1.82 \\
			& ODIN  & 38.92 &  92.15 & 100.00 &  43.43 & 100.00 &   0.31 & 100.00 &   0.00 \\
			& Mahalanobis & {\bf 16.46} &  {\bf 96.75} &  99.76 &  33.38 &  89.71 &  28.16 & 100.00 &   0.38 \\
			& SOFL  & 53.51 &  89.27 &  99.96 &  48.75 & 100.00 &   0.19 & 100.00 &   0.21 \\
			& OE  & 61.59 &  81.51 &  99.99 &  39.94 & 100.00 &   0.81 & 100.00 &   0.54 \\ 
			& ACET & 60.49 &  86.80 &  99.70 &  49.61 &  98.54 &  19.92 &  99.94 &  23.51 \\ 
			& CCU & 40.97 &  90.89 &  99.98 &  49.21 & 100.00 &   0.79 & 100.00 &   0.44 \\ 
			& ROWL & 90.42 &  54.48 & 100.00 &  49.69 & 100.00 &  49.69 & 100.00 &  49.69 \\ 
			& ATOM (ours) & 20.09 &  96.62 &  {\bf 97.05} &  {\bf 68.31} &  {\bf 35.70} &  {\bf 85.87} &  {\bf 97.13} & {\bf 63.16} \\ \hline
			\multirow{10}{0.14\linewidth}{{{\bf Textures}}}  
			& MSP  &  85.57 &  70.08 & 100.00 &  25.93 & 100.00 &   2.63 & 100.00 &   2.28 \\
			& ODIN  & 83.58 &  70.71 & 100.00 &  27.30 & 100.00 &   0.11 & 100.00 &   0.01 \\
			& Mahalanobis & {\bf 34.59} & 89.82 &  {\bf 78.35} &  62.36 &  89.27 &  20.33 &  99.18 &   5.43 \\
			& SOFL  & 57.00 &  87.35 &  99.73 &  44.11 &  99.98 &   0.59 & 100.00 &   0.34 \\
			& OE  & 59.86 &  86.17 &  99.88 &  42.95 &  99.98 &   1.45 & 100.00 &   0.73 \\ 
			& ACET & 62.02 &  86.26 &  99.70 &  49.02 &  83.17 &  63.61 &  99.86 &  40.58 \\ 
			& CCU & 60.80 &  86.34 &  99.88 &  44.90 & 100.00 &   1.39 & 100.00 &   0.62 \\ 
			& ROWL & 97.00 &  51.19 & 100.00 &  49.69 & 100.00 &  49.69 & 100.00 &  49.69 \\ 
			& ATOM (ours) & 40.11 &  {\bf 90.28} &  90.21 &  {\bf 69.99} &  {\bf 47.27} &  {\bf 82.52} &  {\bf 91.45} &  {\bf 65.91} \\ \hline
			\multirow{10}{0.14\linewidth}{{{\bf Places365}}}  
			& MSP  &  83.65 &  73.71 & 100.00 &  32.13 & 100.00 &   1.83 & 100.00 &   1.91 \\
			& ODIN  & 79.19 &  76.48 & 100.00 &  39.16 & 100.00 &   0.00 & 100.00 &   0.00 \\
			& Mahalanobis & 94.64 &  59.52 &  99.80 &  19.43 &  99.96 &   1.13 & 100.00 &   0.50 \\
			& SOFL  & 60.49 &  {\bf 87.57} & 100.00 &  40.25 & 100.00 &   0.06 & 100.00 &   0.15 \\
			& OE  & 58.37 &  86.39 & 100.00 &  50.83 & 100.00 &   0.57 & 100.00 &   0.54 \\ 
			& ACET & 56.26 &  86.75 &  99.66 &  57.36 &  83.60 &  72.35 &  99.83 &  44.27 \\ 
			& CCU & {\bf 55.23} &  87.21 & 100.00 &  44.03 & 100.00 &   0.46 & 100.00 &   0.43 \\ 
			& ROWL & 96.86 &  51.26 & 100.00 &  49.69 & 100.00 &  49.69 & 100.00 &  49.69 \\ 
			& ATOM (ours) & 56.52 &  84.53 &  {\bf 96.65} &  {\bf 63.43} &  {\bf 58.68} &  {\bf 82.50} &  {\bf 96.68} &  {\bf 62.36} \\ \hline
			\multirow{10}{0.14\linewidth}{{{\bf SVHN}}}  
			& MSP  & 80.71 &  76.00 & 100.00 &  25.69 & 100.00 &   2.47 & 100.00 &   2.29 \\
			& ODIN  & 88.66 &  71.65 & 100.00 &  24.17 & 100.00 &   0.00 & 100.00 &   0.00 \\
			& Mahalanobis & 47.78 &  90.54 &  {\bf 98.93} &  53.25 &  99.96 &   6.34 & 100.00 &   2.78 \\
			& SOFL  & {\bf 21.50} &  {\bf 96.15} & 100.00 &  36.59 & 100.00 &   0.07 & 100.00 &   0.16 \\
			& OE  & 44.47 &  92.58 & 100.00 &  40.86 & 100.00 &   0.53 & 100.00 &   0.57 \\ 
			& ACET & 55.86 &  90.36 & 100.00 &  49.51 &  70.55 &  86.86 &  99.99 &  44.93 \\ 
			& CCU & 50.79 &  91.59 & 100.00 &  40.01 & 100.00 &   0.41 & 100.00 &   0.40 \\ 
			& ROWL & 98.88 &  50.25 & 100.00 &  49.69 & 100.00 &  49.69 & 100.00 &  49.69 \\ 
			& ATOM (ours) & 37.78 &  93.68 &  99.54 &  {\bf 70.86} &  {\bf 37.78} &  {\bf 93.68} &  {\bf 99.55} & {\bf 70.23} \\ 
			\bottomrule
		\end{tabular}
	\end{adjustbox}
	\label{tab:cifar100-detail-results}
\end{table*}